\documentclass{article}

\PassOptionsToPackage{numbers, compress}{natbib}

\usepackage[final]{neurips_2025}

\usepackage[utf8]{inputenc} 
\usepackage[T1]{fontenc}    
\usepackage{xcolor}         
\definecolor{refdarkblue}{rgb}{0,0.08,0.45}
\usepackage[colorlinks=true, citecolor=refdarkblue, linkcolor=refdarkblue, urlcolor=refdarkblue]{hyperref}
\usepackage{url}            

\usepackage{amsfonts}       
\usepackage{amsmath}
\usepackage{amssymb}
\usepackage{amsthm}
\usepackage{xfrac}

\usepackage{nicefrac}       
\usepackage[capitalize,noabbrev]{cleveref}

\usepackage{microtype}
\usepackage{graphicx}
\usepackage{subfigure}
\usepackage{booktabs} 
\usepackage{pifont}
\usepackage{enumitem}

\theoremstyle{plain}
\newtheorem{theorem}{Theorem}[section]

\newtheorem{lemma}[theorem]{Lemma}
\newtheorem{claim}[theorem]{Claim}
\Crefname{claim}{Claim}{Claims}
\newtheorem{corollary}[theorem]{Corollary}
\theoremstyle{definition}
\newtheorem{definition}[theorem]{Definition}

\Crefname{assumption}{Assumption}{Assumptions}
\newtheorem{example}[theorem]{Example}
\theoremstyle{remark}
\newtheorem{remark}[theorem]{Remark}

\newcommand{\rb}[1]{\left(#1\right)} 
\newcommand{\bb}[1]{\left[#1\right]} 
\newcommand{\cb}[1]{\left\{#1\right\}} 
\newcommand{\norm}[1]{\left\Vert#1\right\Vert}
\newcommand{\abs}[1]{{\left\lvert{#1}\right\rvert}}
\newcommand{\vect}[1]{\mathbf{#1}} 

\newcommand{\calB}{\mathcal{B}}
\newcommand{\calC}{\mathcal{C}}

\newcommand{\calH}{\mathcal{H}}
\newcommand{\calI}{\mathcal{I}}

\newcommand{\calN}{\mathcal{N}}

\newcommand{\calS}{\mathcal{S}}

\newcommand{\calU}{\mathcal{U}}
\newcommand{\calV}{\mathcal{V}}
\newcommand{\calW}{\mathcal{W}}
\newcommand{\calX}{\mathcal{X}}
\newcommand{\calY}{\mathcal{Y}}
\newcommand{\calZ}{\mathcal{Z}}

\newcommand{\bbE}{\mathbb{E}}

\newcommand{\bbI}{\mathbb{I}}

\newcommand{\bbN}{\mathbb{N}}

\newcommand{\bbP}{\mathbb{P}}

\newcommand{\bbR}{\mathbb{R}}

\newcommand{\bbZ}{\mathbb{Z}}

\newcommand{\bfa}{\mathbf{a}}
\newcommand{\bfb}{\mathbf{b}}

\newcommand{\bfe}{\mathbf{e}}

\newcommand{\bfn}{\mathbf{n}}

\newcommand{\bfq}{\mathbf{q}}
\newcommand{\bfr}{\mathbf{r}}

\newcommand{\bfv}{\mathbf{v}}
\newcommand{\bfw}{\mathbf{w}}
\newcommand{\bfx}{\mathbf{x}}
\newcommand{\bfy}{\mathbf{y}}
\newcommand{\bfz}{\mathbf{z}}

\newcommand{\bfA}{\mathbf{A}}

\newcommand{\bfI}{\mathbf{I}}

\newcommand{\bfM}{\mathbf{M}}

\newcommand{\bfQ}{\mathbf{Q}}

\newcommand{\bfX}{\mathbf{X}}

\newcommand{\bfzero}{\mathbf{0}}

\newcommand{\bfeta}{{\boldsymbol{\eta}}}
\newcommand{\btheta}{{\boldsymbol{\theta}}}

\newcommand{\blambda}{{\boldsymbol{\lambda}}}
\newcommand{\bvareps}{{\boldsymbol{\varepsilon}}}

\newcommand{\bSigma}{{\boldsymbol{\Sigma}}}
\newcommand{\bLambda}{{\boldsymbol{\Lambda}}}

\DeclareMathOperator{\diag}{diag}
\DeclareMathOperator{\essinf}{ess\,inf}
\DeclareMathOperator{\esssup}{ess\,sup}
\DeclareMathOperator{\tr}{Tr}

\newcommand{\ind}[2]{\bbI_{#2} \cb{#1}}
\newcommand{\indnosub}[1]{\bbI \cb{#1}}

\newcommand{\innerprod}[2]{\left\langle #1, #2 \right\rangle}
\newcommand{\D}[2]{D\left(#1\,\|\,#2\right)}
\newcommand{\KL}[2]{\mathrm{KL}\left(#1\,\middle\|\,#2\right)}
\newcommand{\KLrel}[3]{\mathrm{KL}_{#1} \left( #2 \,\middle\|\,#3 \right)}
\newcommand{\kl}[2]{\mathrm{kl}\left(#1\,\middle\|\,#2\right)}
\newcommand{\Dinf}[2]{D_{\infty}\left(#1\,\|\,#2\right)}
\newcommand{\Dinfrel}[3]{D_{\infty}^{#1} \left( #2 \,\|\, #3 \right)}
\newcommand{\diff}{\mathop{}\!\mathrm{d}}
\newcommand{\setst}{\,\mid\,}

\newcommand{\eg}{e.g.~}
\newcommand{\ie}{i.e.~}

\makeatletter
\newenvironment{recall}[1][\proofname]{\par
\normalfont \topsep6\p@\@plus6\p@\relax
\trivlist
\item\relax
{\bfseries
Recall #1}%
{\bfseries\@addpunct{.}}\hspace\labelsep\ignorespaces
}
\makeatother

\newcommand{\comment}[1]{{#1}}

\newcommand{\todo}[1]{\comment{\color{red}[TODO: {#1}]}}

\newcommand{\removed}[1]{}
\newcommand{\movetoappendix}[1]{}

\newcommand{\DOM}{\Omega}
\newcommand{\DOMCLOS}{\overline{\DOM}}
\newcommand{\DOMINT}{\DOM^{\circ}}
\newcommand{\PARAMDIM}{d}

\newcommand{\BOUNDFIELD}{\eta}
\newcommand{\BOUNDVEC}{\bfeta}
\newcommand{\NORMFIELD}{\hat{n}}
\newcommand{\NORMAL}{\bfn}
\newcommand{\neighb}[1]{\calN \rb{#1}}
\newcommand{\SUBDOM}{B}
\newcommand{\PARAMDOM}{\Theta}

\newcommand{\param}{\btheta}
\newcommand{\paramls}{\param_{\mathrm{LS}}}
\newcommand{\paramteach}{\param^\star}

\newcommand{\drift}{\bfb}
\newcommand{\drifti}{b_i}
\newcommand{\dispcoef}{\bSigma}
\newcommand{\diffcoef}{\bfA}
\newcommand{\diffcol}{\bfa}

\newcommand{\diffij}{a_{ij}}
\newcommand{\diffkj}{a_{kj}}

\newcommand{\reflectionprocess}{\bfr_t}

\newcommand{\trainloss}{L_S}
\newcommand{\poploss}{L_{\datadist}}

\newcommand{\trainfail}{E_S}
\newcommand{\expectedfail}{E_D}

\newcommand{\datadist}{D}
\newcommand{\generalpotential}{\Psi}
\newcommand{\potential}[1]{{\generalpotential_{\mathrm{#1}}}}

\newcommand{\stationarydist}{p_{\infty}}
\newcommand{\priordist}{\rho}
\newcommand{\posteriordist}{\hat{\rho}}
\newcommand{\posteriorstationary}{p_{\infty}}

\definecolor{customblue}{HTML}{1D6CB1}
\definecolor{toneddownred}{HTML}{E00000}
\definecolor{badregionred}{HTML}{FBC0AA}

\newcommand{\cmark}{\ding{51}} 
\newcommand{\xmark}{\ding{55}} 

\usepackage{xspace}
\newcommand{\RenyiInf}{R\'enyi infinity\xspace}

\title{Temperature is All You Need for Generalization \\ in Langevin Dynamics and other Markov Processes}

\author{%
  Itamar Harel\thanks{Corresponding author: \texttt{itamarharel01@gmail.com}} \\
  Technion
  \\
  \!\!
  \And
  Yonathan Wolanowsky \\
  Technion
  \And
  Gal Vardi \\
  Weizmann Institute of Science\!
  \And
  Nathan Srebro \\
  Toyota Technological Institute at Chicago
  \And
  Daniel Soudry \\
  Technion 
}

\begin{document}

\maketitle

\begin{abstract}
   We analyze the generalization gap (gap between the training and test errors) when training a potentially over-parametrized model using a Markovian stochastic training algorithm, initialized from some distribution $\param_0 \sim p_0$. We focus on Langevin dynamics with a positive temperature $\beta^{-1}$, i.e.~gradient descent on a training loss $L$ with infinitesimal step size, perturbed with $\beta^{-1}$-variances Gaussian noise, and lightly regularized or bounded.
   There, we bound the generalization gap, \emph{at any time during training}, by $\sqrt{(\beta\mathbb{E} L (\param_0) + \ln(1/\delta))/N}$ with probability $1-\delta$ over the dataset, where $N$ is the sample size, and $\mathbb{E} L (\param_0) = O(1)$ with standard initialization scaling.  In contrast to previous guarantees, we have no dependence on either training time or reliance on mixing, nor a dependence on dimensionality, gradient norms, or any other properties of the loss or model.  This guarantee follows from a general analysis of any Markov process-based training that has a Gibbs-style stationary distribution. The proof is surprisingly simple, once we observe that the marginal distribution divergence from initialization remains bounded, as implied by a generalized second law of thermodynamics.
\end{abstract}

\section{Introduction}

One main goal of contemporary machine learning theory is to predict a model's behavior before training occurs.
A commonly desired metric is the generalization of overparameterized models, such as neural networks (NN). For these models, such a predictive theory of generalization is still lacking, despite great empirical success \citep{Zhang16,Gunasekar2017}.
In particular, a significant line of work aimed to explain the role of optimization in generalization (\eg \cite{Gunasekar2017,soudry2018implicit,Lyu2020Gradient,vardi2023implicit}), and specifically the effect of stochasticity (\eg \cite{raginsky2017non,mou2018generalization,chiang2022loss,pmlr-v235-buzaglo24a}).

Data-dependent Markov processes are a common optimization approach.
These include stochastic gradient descent (SGD), as well as other stochastic gradient methods either studied theoretically \citep{jin2017escape,raginsky2017non}, or used in practice such as SGD with momentum \citep{Nesterov1983AMF}, ADAM \citep{kingma2015adam}, and many more.
Of particular interest are continuous Langevin dynamics (CLD) and discrete analogues of it, which have been studied extensively as models for SGD (see \cref{sec: comparison for cld}).

In \cref{sec:general Markov Processes} we develop, for the first time, a generalization bound applicable to \emph{any data-dependent Markov process} with a Gibbs-type stationary distribution (\ie whose finite density exists and is nonzero w.r.t. some data-independent base measure).
An important feature of our analysis is that it is \emph{entirely independent of the training time} $t$, both in that we do not rely on training for only a small number of steps, nor that we rely on mixing --- the guarantees are valid at any time, with no dependence at all on $t$.
Furthermore, it is also completely trajectory independent.

In \cref{sec: examples} we apply these general results to the particular case where training is done with CLD with loss $L$ and inverse temperature $\beta$, deriving a particularly simple generalization bound for CLD, which we compare to previous generalization bounds for CLD in \cref{sec: related work}, as well as discussing other related work.
Finally, we address limitations and future work in \cref{sec: discussion and future work}.

To prove these results, we first show in \cref{sec:general Markov Processes} how,  for the marginal distribution at time $t$, $p_t$,  its divergence (either KL or the \RenyiInf divergence) from initialization is bounded due to its monotonicity, \ie a generalized second law of thermodynamics \citep{Cover1994,Merhav2011}. 
This surprisingly simple derivation\footnote{\eg to bound the KL divergence of a Markov process having a stationary distribution with potential $\Psi \in \left[0, \infty\right)$, i.e. $\diff p_\infty / \diff p_0 \propto e^{-\Psi}$ (e.g., $\Psi=\beta L$ for CLD), the second law implies the first inequality below
\scriptsize
\begin{align*}
    \mathrm{KL} \! \left(p_{t}||p_{0}\right) &  =  \! \int \! p_{t}\ln\frac{p_{t}}{p_{0}} \!= \! \int \! p_{t}\ln\frac{p_{t}}{p_{\infty}} \! +\! \int \! p_{t} \ln \frac{p_{\infty}}{p_{0}} \! \leq \! \int \! p_{0}\ln\frac{p_{0}}{p_{\infty}} \! + \! \int \! p_{t} \ln \frac{p_{\infty}}{p_{0}} \! =   E_{p_{0}} \Psi -  E_{p_{t}} \Psi \leq  E_{p_{0}} \Psi .
\end{align*}
} 
leads to our key technical result (\cref{cor: second law + chain rule for d inf and kl}).
Standard PAC-Bayes generalization bounds \citep{maurer2004note} then yield our generalization bounds (\cref{thm: main result sup and mean,cor: generalization for ngf box and regularization}).

\section{Generalization Bounds for General Markov Process}\label{sec:general Markov Processes}

In this Section, we consider general data-dependent Markov processes over predictors and obtain a bound on their generalization gap.  
Importantly, although the bound only depends on the initialization distribution and a stationary distribution, it will apply to predictors at any time $t \ge 0$ along the Markov process.
Our main goal is to apply these bounds to stochastic training methods, such as Langevin dynamics,  where the iterates form a data-dependent Markov process.
But to emphasize the broad generality of the results, in this section we consider a generic stochastic optimization framework and general data-dependent Markov processes. 

We obtain generalization guarantees by bounding the KL-divergence (or, for high probability bounds, the \RenyiInf divergence, see \cref{def: divergences}) between the data-dependent marginal distribution $p_t$ of the predictors at time $t$, and some data-independent base measure $\nu$ (the PAC-Bayes ``prior'').
The crux of the analysis is therefore bounding the divergence between $p_t$ and $\nu$, based only on assumptions on the initial distribution $p_0$ (specifically, the divergence between $p_0$ and $\nu$) and a stationary distribution $p_\infty$ (specifically, requiring that $p_\infty$ can be expressed as a Gibbs distribution with bounded potential or expected potential, see \cref{def: Gibbs}) --- we do this in \cref{subsec:mcdiv}.
Then, in \cref{subsec:gencor} we plug these bounds on the divergence between $p_t$ and $\nu$ into standard PAC-Bayes bounds to obtain the desired generalization guarantees. 

Detailed proofs of all the results in this section can be found in \cref{app-sec: proof of new main}. 

\subsection{Bounding the Divergence of a Markov Process}\label{subsec:mcdiv}

In this subsection, we consider a general time-invariant Markov process\footnote{Formally stated: we require that for any $0\leq t_1<t_2<t_3$ we have that $h_{t_3}$ is independent of $h_{t_1}$ conditioned on $h_{t_2}$ (Markov property) and that for any $0 \le t_1,t_2,\Delta$ we have that $h_{t_1+\Delta}|h_{t_1}$ has the same conditional distribution as $h_{t_2+\Delta}|h_{t_2}$ (time-invariance).} $h_t\in\calH$ over a state space $\calH$.
The Markov process can be either in discrete or continuous time, \ie we can think of $t$ as either an integer or a real index. 
We denote by $p_t$ the marginal distribution at time $t$, \ie $h_t\sim p_t$.
We do \emph{not} assume that the Markov process is ergodic, and all our results will rely on the existence of \emph{some} stationary distribution $p_\infty$. 
The main goal of this subsection is to bound the divergence $\D{p_t}{\nu}$ between the marginal distribution at time $t$ and some reference distribution $\nu$.  We can think of a bound on the divergence as ensuring high entropy relative to $\nu$, or in other words that $p_t$ does not concentrate too much relative to $\nu$, \ie does not have too much probability mass in a small $\nu$-region.
We present all bounds for both the KL-divergence $\KL{p}{q}$ and the \RenyiInf divergence $\Dinf{p}{q}$, defined below.

\textbf{Divergences and Gibbs distributions.} 
We recall the definitions of our two divergences, and also relate them to the Gibbs distribution.
It will also be convenient for us to introduce ``relative'' versions of divergences.

\begin{definition}[Divergences \footnote{The term ``divergence'' is a slight abuse of notation, as the following definitions are not strictly non-negative, without specifying $\mu$.}]\label{def: divergences}For probability distributions $p,q$ and $\mu$:
    \begin{enumerate}[ leftmargin=15pt]
        \item The $\mu$-\textbf{weighted Kullback-Leibler (KL) divergence} (a.k.a. relative cross-entropy) is\footnote{For two measures $p$ and $q$,  $\diff p / \diff q$ is the Radon-Nikodym derivative (\ie the density of $p$ w.r.t. $q$) when it exists (i.e.~when $p \ll q$, \ie $p$ is absolutely continuous w.r.t. $q$), or $\infty$ otherwise.} $\KLrel{\mu}{p}{q} = \int \diff \mu \ln \frac{\diff p}{\diff q}$, and the \textbf{KL-divergence} is then $\KL{p}{q} = \KLrel{p}{p}{q}$.
        \item The \textbf{\RenyiInf divergence} is\footnote{The essential supremum of a function $f$ w.r.t.~a measure $\mu$ is $\esssup_\mu f = \inf \cb{b \in \bbR \setst \mu \rb{f > b} = 0}$,  \ie the smallest (infimum) number that bounds $f$ from above almost everywhere.
The essential infimum  is defined similarly.} $\Dinfrel{\mu}{p}{q} = \esssup_{\mu} \ln \frac{\diff p}{\diff q}$, with $\Dinf{p}{q} = \Dinfrel{p}{p}{q}$.
    \end{enumerate}
\end{definition}

\begin{definition}[Gibbs distribution] \label{def: Gibbs}
A distribution $p$ is \textbf{Gibbs} w.r.t.~a \textbf{base distribution} $q$ with \textbf{potential} $\generalpotential\!:\!\calH\rightarrow \bbR$ if $Z = \int e^{- \generalpotential} \diff q < \infty$ and
\[
\diff p = Z^{-1} e^{-\generalpotential} \diff q \,.
\]
\end{definition}
\begin{claim} \label{claim: sum of dists is difference of expectations or extremas}
If $p,q,\mu,\nu$ are probability measures, and $p$ is Gibbs w.r.t. $q$ with potential $\generalpotential < \infty$, then
\begin{enumerate}[ leftmargin=15pt]
    \item $\KLrel{\mu}{p}{q} + \KLrel{\nu}{q}{p} = \bbE_{\nu} \generalpotential - \bbE_{\mu} \generalpotential$, 
    \item $\Dinfrel{\mu}{p}{q}+\Dinfrel{\nu}{q}{p} = \esssup_\nu \generalpotential - \essinf_{\mu} \generalpotential$.
\end{enumerate}
So, $\KL{p}{q}+\KL{q}{p}= \bbE_{q} \generalpotential - \bbE_{p} \generalpotential$, and $\Dinf{p}{q}+\Dinf{q}{p}= \esssup_q \generalpotential - \essinf_p \generalpotential$.
\end{claim}
That is, the potential of a Gibbs distribution $p$ allows us to bound the divergence \emph{in both directions} between $p$ and the base measure $q$.
A generalized converse of \cref{claim: sum of dists is difference of expectations or extremas} also holds, and we have that bounding on the {\em symmetrized} divergences (but not just on one direction!) is also sufficient for $p$ being Gibbs with a bounded potential.%
\footnote{\label{foot:converse}More formally: $\KL{p}{q}+\KL{q}{p}\leq \beta$ iff there exists a potential $\Psi$ such that $p$ is Gibbs w.r.t. $q$ with potential $\Psi$ and $\bbE_{q} \generalpotential - \bbE_{p} \generalpotential\leq \beta$, and similarly $\Dinf{p}{q}+\Dinf{q}{p}\leq \beta$ iff there exists a potential $0\leq \Psi\leq\beta$ such that $p$ is Gibbs w.r.t. $q$ with potential $\Psi$. See \cref{app-claim: footnote converse of symmeterized KL} for a proof.}

\textbf{Second Law of Thermodynamics.} Central to our analysis is the following monotonicity property on the divergence between the marginal distribution of a Markov process and {\em any} stationary distribution.
\begin{claim}[Cover's Second Law of Thermodynamics]\label{claim: second law}
Let $p_t$ be the marginal distribution of a time-invariant Markov process, and $p_\infty$ a stationary distribution for the transitions of the Markov process (the process need not be ergodic, and $p_t$ need not converge to $p_\infty$).
Then for any $t \geq 0$
\begin{align*}
    \KL{p_t}{p_\infty} \leq \KL{p_{0}}{p_\infty}  \quad\quad \textrm{and}\quad\quad \Dinf{p_{t}}{p_\infty} \leq \Dinf{p_{0}}{p_\infty} \,.
\end{align*}
\end{claim}
When the stationary distribution is uniform (thus having maximal entropy), the KL-form of \cref{claim: second law} recovers the familiar second law of thermodynamics, \ie that the entropy is monotonically non-decreasing.
The more general form, as in \cref{claim: second law}, is a direct consequence of the data processing inequality, as pointed out by Theorem 4 of \citet{Cover1994} (see also \citep{cover2001elements,Merhav2011} and the generalization to R\'enyi divergences in \cite[Theorem 9 and Example 2]{van2014renyi} ---for completeness we provide a proof in \cref{app-sec: auxiliary general}). 

In our case, the stationary distribution $p_\infty$ will not be uniform, but rather will be very data-dependent (we are interested mostly in processes that aim to optimize some data-dependent quantity, such as Langevin dynamics).
Nevertheless, we do want to use \cref{claim: second law} to control the entropy of $p_t$ relative to some benign data-independent base distribution $\nu$ (which we can informally think of as ``uniform'').
To do so, we can use the chain rule and plug in \cref{claim: second law} to obtain that for any distribution $\nu$ and at any time $t$ we have (see \Cref{app-lem: triangle like inequalities for kl and d infty with second law}  in \cref{app-sec: proof of new main} for the full derivation):
\begin{multline} \label{eq:klchain}
    \KL{p_t}{\nu} =
    \KL{p_t}{p_\infty} + \KLrel{p_t}{p_\infty}{\nu} 
    \leq \KL{p_0}{p_\infty} + \KLrel{p_t}{p_\infty}{\nu} \\
    = \KL{p_0}{\nu} + \KLrel{p_0}{\nu}{p_\infty} + \KLrel{p_t}{p_\infty}{\nu} \,,
\end{multline}
and similarly,
\begin{equation}\label{eq:dinfchain}
    \Dinf{p_t}{\nu} \leq \Dinf{p_0}{\nu} + \Dinfrel{p_0}{\nu}{p_\infty} + \Dinfrel{p_t}{p_\infty}{\nu} \,.
\end{equation}
Bounding the last two terms in \eqref{eq:klchain} and \eqref{eq:dinfchain} using \cref{claim: sum of dists is difference of expectations or extremas} we obtain the main result of this subsection:

\fbox{\parbox{\textwidth}{
\begin{corollary} \label{cor: second law + chain rule for d inf and kl}
    For any distribution $\nu$ and any time-invariant Markov process, and any stationary distribution $p_\infty$ that is Gibbs w.r.t. $\nu$ with potential $\generalpotential \geq 0$ (the Markov chain need not be ergodic, and need not converge to $p_\infty$), at any time $t\geq 0$:
    \begin{align}
        \KL{p_t}{\nu} &\le \KL{p_0}{\nu} + \bbE_{p_0} \generalpotential - \bbE_{p_t} \generalpotential \le \KL{p_0}{\nu} + \bbE_{p_0} \generalpotential \\
        \Dinf{p_t}{\nu} &\leq \Dinf{p_0}{\nu} + \esssup_{p_0} \generalpotential  
    \end{align}
\end{corollary}
}}

The important feature of \cref{cor: second law + chain rule for d inf and kl} is that it bounds the divergence \emph{at any time $t$}, in terms of a right-hand side that depends only on the initial distribution $p_0$ and a stationary distribution $p_\infty$.
Interpreting the divergence $\D{p_t}{\nu}$ as a measure of concentration, the Corollary ensures that at no point during its run, and regardless of mixing, does the Markov process concentrate too much, and it always maintains high entropy (relative to the base measure $\nu$). 

\begin{remark}
    In order to bound the divergence $\D{p_t}{\nu}$ at finite time $t$, it is not enough to rely only on the divergences $\D{p_0}{\nu}$ and $\D{p_\infty}{\nu}$ from the initial and stationary distributions, and it is necessary to rely also on the reverse divergence $\D{\nu}{p_\infty}$ --- see \cref{app-sec: tightness and necessity}.
\end{remark}

\subsection{From Divergences to Generalization}\label{subsec:gencor}

\Cref{cor: second law + chain rule for d inf and kl} can be directly used to obtain PAC-Bayes type generalization guarantees.
Specifically, we consider a \textbf{generic stochastic optimization setting} specified by a bounded instantaneous objective $f:\calH\times\calZ\rightarrow [0,1]$ over a class $\calH$, which we will refer to as the ``predictor'' class, and instance domain $\calZ$.
For example, in supervised learning $\calZ=\calX\times\calY$, $\calH \subseteq \calY^\calX$ and $f(h,(x,y))=\indnosub{h(x) \neq y}$ measures the error of predicting $h(x)$ when the correct label is $y$.
For a source distribution $\datadist$ over $\calZ$ and data $S\sim\datadist^N$ of size $N$ we would like to relate the population and empirical objectives
\begin{equation} \label{eq: definition of ED and ES}
    \expectedfail \rb{h}=\bbE_{z\sim\datadist}[f(h,z)] \quad\quad
    \trainfail \rb{h} = \frac{1}{N}\sum_{z\in S} f(h,z).
\end{equation}
In our case, we are interested in predictors generated by a \textbf{data-dependent Markov process} $h_t$.
That is, conditioned on the data $S$, $\{h_t\}_{t\geq 0}$ is a time-invariant Markov process, specified by some (possibly data-dependent) initial distribution $p_0(h_0;S)$, and a transition distribution that would also depend on the data $S$, and specifies a (randomized) rule for generating the next iterate $h_{t+1}$ (if in discrete time) from the current iterate $h_t$ and the data $S$ (as in, e.g., stochastic gradient descent or stochastic gradient Langevin dynamics; SGLD).

We present two types of generalization guarantees: guarantees that hold \emph{in expectation} over a draw from the Markov process (\eqref{eq: main thm general potential expectation} below) and guarantees that hold \emph{with high probability} over a single draw from the Markov process (as in \eqref{eq: main thm general potential probability}, \eg a single run of CLD).
In both cases, the guarantees hold with high probability over the training set.

\fbox{\parbox{\textwidth}{
\begin{theorem} \label{thm: main result sup and mean}
Consider any distribution $\datadist$ over $\calZ$, function $f:\calH \times \calZ \to [0,1]$, sample size $N \ge 8$, and any distribution $\nu$ over $\calH$.
Let $\{ h_t \in \calH \}_{t\geq 0}$ be a discrete or continuous time process (i.e.~$t \in \mathbb{Z}_+$ or $t \in \mathbb{R}_+$) that is time-invariant Markov conditioned on $S$, that starts from an initial distribution $p_0(\cdot;S)$ (that may depend on $S$), and admits a stationary distribution conditioned on $S$, $p_\infty(\cdot;S)$.
Let $\Psi_S(h)\geq 0$ be a non-negative potential function and assume that $p_\infty(\cdot;S)$ is Gibbs w.r.t.~$\nu$ with potential $\Psi_S$.
Then:
\begin{enumerate}[ leftmargin=15pt]
    \item with probability $1-\delta$ over $S\sim \datadist^N$, 
    \begin{align}
        \bbE\left[\expectedfail(h_t) - \trainfail(h_t) \middle| S \right] &\leq 
        \sqrt{
        \frac{\KL{p_0(\cdot;S)}{\nu}+\bbE\left[\Psi_S(h_0)\middle|S\right] + \ln \sfrac{N}{\delta}}{2 N} } \,, \label{eq: main thm general potential expectation}
        \end{align}
    \item with probability $1-\delta$ over $S\sim \datadist^N$ and over $h_t$:
    \begin{align} 
    \expectedfail(h_t) - \trainfail(h_t) &\leq 
    \sqrt{
    \frac{\Dinf{p_0(\cdot;S)}{\nu}+\esssup_{h\sim p_0(\cdot;S)} \generalpotential_S (h) 
    + \ln \sfrac{N}{\delta}}{2N}} \,. \label{eq: main thm general potential probability}
    \end{align}
\end{enumerate}
\end{theorem}
}}
\begin{proof}
The Theorem follows immediately by plugging the divergence bounds of \Cref{cor: second law + chain rule for d inf and kl} into standard PAC-Bayes guarantees, which we do in \cref{app-sec: proof of new main}.
\end{proof}

\begin{remark}
    A simplified variant of \Cref{thm: main result sup and mean} can be stated when the initial distribution $p_0$ is data-independent and always equal to $\nu$.  In this case the divergence between $p_0$ and $\nu$ vanishes, and \eqref{eq: main thm general potential expectation} and \eqref{eq: main thm general potential probability} become 
    \begin{equation}
        \bbE\left[\expectedfail(h_t) - \trainfail(h_t) \middle| S \right] \leq \!
        \sqrt{
        \frac{\bbE_{p_0} \bb{ \Psi_S \mid S} + \ln \sfrac{N}{\delta}}{2N} }
        ,\;
        \expectedfail(h_t) - \trainfail(h_t) 
        \leq \!
        \sqrt{
        \frac{\esssup_{p_0} \Psi_S + \ln \sfrac{N}{\delta}}{2N} } \,.
    \end{equation}
     But allowing $p_0 \neq \nu$ is more general, as it both allows using a data-dependent initialization (recall that $\nu$ must be data independent) and it allows initializing to a distribution where $\D{p_\infty}{p_0}$ is infinite --- e.g., we can allow initializing to a degenerate initial distribution $p_0$ whose support is a strict subset of the support of $p_\infty$ (in which case $p_\infty$ will definitely {\em not} be Gibbs w.r.t.~$p_0$), as long as the $\nu$-mass of the support of $p_0$ is not too small. 
\end{remark}

\begin{remark}
    In \Cref{thm: main result sup and mean}, the Markov process need not be ergodic, and need not converge to $p_\infty$, or converge at all.
    If there are multiple stationary distributions, the theorem holds for all of them, and so we can take $p_\infty$ to be any stationary distribution we want.
    And in any case, there is no mixing requirement, and the theorem holds at any time $t$. 
\end{remark}

\begin{remark}
    Our data-dependent Markov process of interest, and in particular CLD and SGD, might aim to minimize $\trainfail(h_t)$, and the potential $\Psi$ might also be related to it (as in, e.g., CLD).
    This is allowed, but is in no way required in \Cref{thm: main result sup and mean}.
    Even for CLD, these might be related but not the same, as we might be minimizing a surrogate loss, such as a logistic loss, but are interested in bounding the generalization gap for a zero-one error.
    In stating \Cref{thm: main result sup and mean} we intentionally refer to an arbitrary stochastic optimization problem and an arbitrary data-dependent Markov process, that are allowed to be related or dependent in arbitrary ways.
\end{remark}

\begin{remark}
In \cref{app-sec: tightness and necessity} we show that in order to ensure generalization at every intermediate $t$, it is not sufficient to only bound $\KL{p_\infty}{\nu}$ or $\Dinf{p_\infty}{\nu}$, and we do need the stronger symmetric bound ensured by the Gibbs potential and \cref{claim: sum of dists is difference of expectations or extremas}; and that it is also necessary to relate both $p_0$ and $p_\infty$ to the \emph{same} data independent distribution $\nu$, as relating them to different data-independent distributions ensures generalization at the beginning and at the end, but not the middle of training.
\end{remark}

\begin{remark}
    In \Cref{thm: main result sup and mean} we plugged \Cref{cor: second law + chain rule for d inf and kl} into a simplified PAC-Bayes bound  that allows for easy interpretation and comparison with other results.
    But once we have the divergence bounds of \Cref{cor: second law + chain rule for d inf and kl}, we can just as easily plug them into tighter PAC-Bayes bounds --- see \Cref{app-sec: proof of new main}. 
    For example, when $\trainfail \rb{h_t}  \approx 0$, these yield a rate of $O\rb{1/N}$. 
\end{remark}

\section{Special Case: Continuous Langevin Dynamics} \label{sec: examples}

Clearly, given \cref{thm: main result sup and mean} all we need to do in order to derive explicit generalization bounds for \emph{any} Markovian training procedure, is to find a stationary distribution, and bound its potential (or its expectation at $p_0$).
In this section, we will exemplify our results in a few special cases of continuous-time Langevin dynamics (CLD), a commonly studied approximation for NN training with ``infinitesimal learning rate'' (\eg \citep{mandt2017stochastic}, see \cref{sec: comparison for cld} for additional references), which have a normalized stationary distribution that we can write analytically. 

\textbf{Additional notation.} In the following, it will be convenient to consider a parametric model. Specifically, we assume that there exists some parameter space $\PARAMDOM \subseteq \bbR^\PARAMDIM$ that parameterizes a hypothesis class $\calH \subseteq \calY^{\calX}$ via a mapping $\PARAMDOM \ni \param \mapsto h_\param \in \calH$, and assume Markovian dynamics \emph{in parameter space}, instead of in the hypothesis space (note that Markov processes in parameter space may not be Markovian in hypothesis space, but the same generalization results apply ).
We shall also use, with some abuse of notation, $\varphi \rb{\param} = \varphi \rb{h_\param}$ for any data-dependent or data-independent function $\varphi$ over hypotheses, \eg a training loss/objective $\trainloss$ w.r.t a training set $S$.
Finally, we use $\calC^2$ to denote the space of twice continuously differentiable functions on $\PARAMDOM$.

\textbf{CLD in a bounded domain.}
Let $\PARAMDOM$ be a box in $\bbR^\PARAMDIM$, and suppose that training is modeled with CLD in a bounded domain, \ie that the parameters evolve according to the stochastic differential equation with reflection at the boundary (SDER) 
\begin{align} \label{eq: train loss sder bounded}
    d\param_t = - \nabla \trainloss \rb{\param_t} dt + \sqrt{2 \beta^{-1}} d\bfw_t + d \reflectionprocess \,,
\end{align}
where $\trainloss \ge 0$  is twice continuously differentiable, $\bfw_t$ is a standard Brownian motion, and $\reflectionprocess$ is a reflection process that constrains $\param_t$ within $\PARAMDOM$.
Such weight clipping is quite common in practical scenarios such as NN training.
For simplicity, we assume that $\reflectionprocess$ has normal reflection, meaning that the reflection is perpendicular to the boundary.
An established result in the analysis of SDERs states that under these assumptions \eqref{eq: train loss sder bounded} has a stationary distribution $\posteriorstationary \rb{\param} \propto e^{-\beta \trainloss \rb{\param}} \ind{\param}{\PARAMDOM}$ (see \cref{app-sec: SDER in a box}).
Thus, when $p_0 = \mathrm{Uniform} \rb{\PARAMDOM}$, we have $p_0 = \nu$.

\textbf{Regularized CLD in $\bbR^\PARAMDIM$.}
Suppose that the parameters evolve according to the stochastic differential equation (SDE) with weight decay (\ie $
\ell^2$ regularization)
\begin{align} \label{eq: train loss sde regularized}
    d\param_t = - \nabla \trainloss \rb{\param_t} dt - \lambda \beta^{-1} \param_t dt + \sqrt{2 \beta^{-1}} d\bfw_t \,,
\end{align}
where $\trainloss \ge 0$ is twice continuously differentiable, $\bfw_t$ is a standard Brownian motion. 
Such weight decay is also quite common in practical scenarios such as NN training.
Similar to the previous case, with the regularization and twice continuous differentiability of $\trainloss$ this process has a unique stationary distribution $\posteriorstationary \rb{\param} \propto e^{-\beta \trainloss \rb{\param}} \phi_{\lambda} \rb{\param}$, where $\phi_\lambda$ is the density of the multivariate Gaussian $\calN \rb{\bfzero, \lambda^{-1} \bfI_{\PARAMDIM}}$.
Thus, when $p_0 = \calN \rb{\bfzero, \lambda^{-1} \bfI_{\PARAMDIM}}$, we also have $p_0 = \nu$.

We can now formulate a generalization bound for both cases.

\begin{corollary} \label{cor: generalization for ngf box and regularization}
    Assume that the parameters evolve according to either \eqref{eq: train loss sder bounded} with $p_0 = \mathrm{Uniform} \rb{\PARAMDOM}$, or \eqref{eq: train loss sde regularized} with $p_0 = \calN \rb{\bfzero, \lambda^{-1} \bfI_{\PARAMDIM}}$.
    Then for any time $t\ge 0$, and $\delta \in \rb{0, 1}$,
    \begin{enumerate}[ leftmargin=15pt]
        \item w.p. $1-\delta$ over $S\sim \datadist^N$, 
        \begin{align} \label{eq: cld bound in expectation}
        \bbE_{\param_t \sim p_t} \bb{\expectedfail(\param_t) - \trainfail(\param_t) \mid S} \le \sqrt{\frac{\beta \bbE_{\param \sim p_0} \bb{ \trainloss (\param) \mid S} + \ln \rb{N/\delta}}{2N}} \,.
        \end{align}
        \item w.p. $1-\delta$ over $S\sim \datadist^N$ and $\param_t \sim p_t$
        \begin{align} \label{eq: cld bound in probability}
        \expectedfail(\param_t) - \trainfail(\param_t) \le \sqrt{\frac{\beta \esssup_{p_0} \trainloss (\param) + \ln \rb{N/\delta}}{2N}} \,.
        \end{align}
    \end{enumerate}
\end{corollary}
The proof is simple --- by assumption, in both cases $p_0 = \nu$ so $\Dinf{p_0}{\nu} = 0$. 
The rest is a direct substitution into \cref{thm: main result sup and mean}, and in particular, using $\beta \trainloss$ as potential $\generalpotential_S$.

\subsection{Interpreting \texorpdfstring{\cref{cor: generalization for ngf box and regularization}}{}}

\cref{cor: generalization for ngf box and regularization} raises questions on the relevance of this setting, which we address below: (1) How large is $\bbE_{p_0} \trainloss \rb{\param}$ in practically relevant cases? (2) Can we attribute the generalization to the regularization (either with the $\ell_2$ regularization term, or the bounded domain)?  (3) Can models successfully train in the presence of noise with a variance large enough to make the bounds non-vacuous?

\textbf{Magnitude of the initial loss.} Commonly, the dependence on $\bbE_{p_0} \trainloss \rb{\param}$ with realistic $p_0$ and $\trainloss$ is relatively mild.
For example, using standard initialization schemes, Gaussian process approximations \citep{Neal1996,matthews2018gaussian,lee2018deep,hanin2023random} imply that the output of an infinitely wide fully connected neural network converges to a Gaussian with mean 0 and $O(1)$ variance at initialization. 
So in many cases $\bbE_{p_0} \trainloss \rb{\param} = O(1)$, such as for the scalar square and logistic losses. 
In the multi-output case, $\bbE_{p_0} \trainloss \rb{\param}$ may also depend on the number of outputs (e.g., logarithmically so in softmax-cross-entropy). 
A more difficult question is concerned with the case that $\esssup_{p_0} \trainloss = \infty$, which is common when $p_0$ has infinite support.
This can be mitigated by clipping the loss, which is standard in practice (e.g. in reinforcement learning \citep{mnih2015human,schulman2017proximal}) and in the theory of optimization \citep{levy2021storm,kavis2022high}. 
Moreover, this clipping can be done in a differentiable way (e.g. using either softmin, tanh (e.g. $c \cdot \tanh (L / c)$), etc) and at values only slightly higher than the typical loss at the initialization (since the loss is roughly monotonically decreasing in CLD with small noise, the optimization process would typically operate below the clipping and will not be affected by it).

\textbf{Magnitude of regularization.} In the above result we must use regularization (or a bounded domain) that matches the initialization $p_0$ (this can be somewhat relaxed, see \cref{sec: extensions and modifications}).
The same assumption, that the regularization matches the initialization, was also made in other theoretical works on CLD \citep{mou2018generalization,Li2020On,futami2023time}.
Note that, NN models regularized this way remain highly expressive, both empirically (\cref{app-sec: experiments}) and theoretically (\cref{app-sec: no_uc}), and therefore we cannot use this regularization alone, together with classical uniform convergence approaches to show generalization.
Intuitively, this is because the regularization term can be tiny, for example, in \eqref{eq: train loss sde regularized} the regularization term is divided by $\beta$. 
Therefore, when $\beta = O \rb{N}$ (which is sufficient for a non-vacuous result), $p_0=\nu$, and we use a standard deep nets initialization distribution $p_0$ (e.g., \citep{pmlr-v9-glorot10a,he2015delving}, where $\lambda \propto {\mathrm{layer\,width}}$), the regularization coefficient is $O \rb{\frac{\mathrm{layer\,width}}{N}}$ that is rather small in realistic cases. Therefore, we found (empirically) that it does not seem to have a large effect at practical timescales. 
In addition, one can always increase the regularization by modifying the loss $\trainloss \leftarrow \trainloss + c \norm{\param}^2$ in \eqref{eq: train loss sde regularized}. 
Under standard initializations, this changes the loss in the bound by an $O(c \Tilde{d})$ factor, where $\Tilde{d}$ is the depth of the neural network and so $c \Tilde{d}$ is small, for common values of $c$ and $\Tilde{d}$. 
Therefore, combining these observations, we do not see the magnitude of the regularization as a significant practical issue.

\textbf{Magnitude of noise: theoretical perspective.} In the above result we must use  $\beta  = O \rb{N}$ to obtain a non-vacuous bound. This requirement is standard in many theoretical works. For example, as we will discuss below in \cref{sec: comparison for cld}, all previous generalization bounds for CLD and SGLD also required, to generalize well, $\beta  = O \rb{N}$ and potentially much worse (lower $\beta$). In addition, other theoretical works on noisy training also typically had $\beta  = O \rb{N}$ or worse. For example, when considering the ability of noisy gradient descent to escape saddle points, \citet{jin2017escape} uses noise sampled uniformly from a ball with a radius that depends on the dimensionality and smoothness of the problem, and thus cannot decay with $N$.
Moreover, it is known that the Gibbs posterior\footnote{Generalization bounds for the Gibbs posterior typically assume that it is ``trained'' and ``tested'' on the same function, while here the distribution is defined by the loss and ``tested'' on the error.} generalizes well with $\beta = O \rb{\sqrt{N}}$ (\eg see Theorem 2.8 in \citealp{alquier2024user}), which is significantly smaller than $\beta = O \rb{N}$.
Lastly, in \cref{app-sec: quadratic objective} we examine the impact of $\beta$ in the simple model of linear regression with i.i.d. standard Gaussian input, labels produced by a constant-magnitude teacher label noise, trained using regularized CLD as in \eqref{eq: train loss sde regularized}, with $\lambda \propto d$ to match standard initialization. We find there that whenever $d\ll \beta \ll N$, the added noise does not significantly affect the training or population losses, and our bound is useful, i.e., it implies a vanishing generalization gap (since $\beta \ll N$ and $\mathbb{E}_{p_0}L=O(1)$). Note that $d\ll N$ is not a major constraint, since  $d\ll N$ is required to obtain low population loss in this setting, even if we did not add noise to the training process (\ie $\beta=\infty$).

\textbf{Magnitude of noise: empirical perspective.} An inverse temperature of $\beta  = O \rb{N}$ is also relevant in many practical settings. For example, in Bayesian settings, when we wish to (approximately) sample from the posterior, it is quite common to use variants of SGLD; then inverse temperatures of order $\beta = O \rb{N}$ are commonly used to achieve good generalization \citep{wenzel2020good}, which matches our results.
In the standard practical training settings, the inverse temperature is a hyperparameter tuned to best fit a given problem. Empirically, in \cref{app-sec: experiments} we find that $\beta = O \rb{N}$ can be tuned to obtain non-vacuous generalization bounds for overparameterized NNs in a few small binary classification datasets (binary MNIST, Fashion MNIST, SVHN, and a parity problem), \ie the sum of the generalization gap bound and the training error is smaller than $0.5$.
Importantly, these non-vacuous bounds do not use any trajectory-dependent quantities as other non-vacuous bounds (\eg \citep{dziugaite2017,lotfi2022pac}), which can make them arguably more useful as they can be calculated before training. The bounds are still not very tight (at noise levels that allow for non-vacuous bounds), but we believe there is still much room for improvement in future work. 

\subsection{Extensions and Modifications} \label{sec: extensions and modifications}

\textbf{State dependent diffusion coefficient.}
Consider a state-dependent diffusion coefficient
\begin{align*}
    d \param_t = - \nabla \trainloss \rb{\param_t} dt + \sqrt{2 \beta^{-1} \sigma^2 \rb{\param_t}} d \bfw_t + d \reflectionprocess \,,
\end{align*}
where $\sigma^2 \in \calC^2$.
For example, in \cref{app-sec: stationary distirbutions of cld} we derive the explicit form of stationary distributions when $\sigma^2 \rb{\param} = \rb{\trainloss \rb{\param} + \alpha}^k$ or $\sigma^2 \rb{\param} = e^{\alpha \trainloss \rb{\param}}$, for some  $k \in \bbN$ and $\alpha > 0$. In both cases, the analytic form of the stationary potential $\generalpotential$ can be used directly with \cref{thm: main result sup and mean} to derive generalization bounds.

\textbf{Restricted initialization.}
In \cref{app-sec: cld generalization subsec} we present generalizations of \cref{cor: generalization for ngf box and regularization} to cases where $p_0$ and $\nu$ are different.
Specifically, for the bounded case we consider $p_0$ that is uniform in a subset $\PARAMDOM_0 \subset \PARAMDOM$ of the domain, and for the regularized case we consider general diagonal Gaussian initialization and regularization.
In particular, this means that some of the parameters can be more loosely regularized/bounded at a cost proportional to their number.
For example, in a deep NN, if only a single layer is loosely regularized/bounded, the KL-divergence cost will be proportional only to the number of parameters in that layer, not the entire $\PARAMDIM$.

\section{Related Work} \label{sec: related work}
\paragraph{Information theoretic guarantees and PAC-Bayes theory.}
A common type of generalization bounds consists of a measure of the dependence between the learned model and the dataset used to train it, such as the mutual information between the data and algorithm \citep{raginsky2016information,xu2017information,russo2020how} or the KL-divergence between the predictor's distribution and any data-independent distribution \citep{mcallester1998some,catoni2007pac,alquier2024user}.
In particular, recent works were able to estimate such dependence measures from trained models to derive non-vacuous generalization bounds, even for deep overparameterized models.
For example, \citet{dziugaite2021role} used held-out data to bound the KL-divergence in a PAC-Bayes bound with a data-dependent prior.
Other works used some property of the trained model to estimate the information content, adding valuable insight to the mechanisms facilitating the successful generalization,  such as the size of the compressed model after training, due to noise stability \citep{arora2018stronger}, and data structure \citep{lotfi2022pac}.

\textbf{Generalization of the Gibbs posterior.}
One classic result in the PAC-Bayesian theory of generalization is that the Gibbs posterior with properly tuned temperature minimizes the PAC-Bayes bound of \citet{mcallester1998some}, \ie the KL-regularized expected loss.
\citet{raginsky2017non} used uniform stability \citep{bousquet2002stability} to derive a different generalization bound for sampling from the Gibbs distribution.
Due to these known generalization capabilities, some works relied on it to derive bounds for related algorithms.

\subsection{Explicit Comparison for CLD} \label{sec: comparison for cld}

\begin{table}[t]
\small
\centering
\caption{
\textbf{Comparison of generalization bounds for CLD.}
We compare the main bounds in settings similar to the CLD setting considered here.
All the bounds here consider different functions for training and evaluation, as was done in this paper with $\trainloss$ and $\trainfail,\expectedfail$, respectively. For simplicity, we assume that $\trainfail,\expectedfail$ are bounded in $[0,1]$, and are therefore $1/2$-subGaussian via Hoeffding's inequality.
We use $g_t$ to denote \emph{trajectory-dependent} statistics of the gradients, $K$ for the Lipschitz constant, and $C$ for a bound on the loss, or the expected loss at initialization, when they are required.
For compactness, low-order terms are omitted, time-dependent quantities are simplified to an approximate asymptotic value, and trajectory dependent integrals are solved by considering the statistics $g_t$ constant w.r.t. the variable of integration. 
Finally, all bounds assume a Gaussian initialization $\calN \rb{ \bfzero, \lambda^{-1} \beta^{-1} \bfI_\PARAMDIM}$ and regularization term $\frac{\lambda}{2} \norm{\param_t}^2$, both with the same $\lambda$. 
}
\label{tab: continuous bound comparison}
\begin{tabular}{cccccc}
\toprule
Paper & Trajectory dependent & dimension dependence & Bound (big $O$) \\
\midrule
\citet{mou2018generalization} & \cmark & through gradients & $\sqrt{\frac{\beta}{N}} \cdot \sqrt{\frac{1}{\lambda} g_t^2}$ \\
\citet{Li2020On} & \xmark & through $K$ & $\frac{e^{4\beta C} \sqrt{\beta}}{N} \cdot \frac{2K}{\sqrt{\lambda}}$ \\
\citet{futami2023time} & \cmark & through gradients & $\sqrt{\frac{\beta}{N} e^{8 \beta C}} \cdot \sqrt{\frac{1}{\lambda} g_t^2}$ \\
Ours \eqref{eq: cld bound in expectation} & \xmark & \xmark & $\sqrt{\frac{\beta}{N}} \cdot \sqrt{C}$ \\
\bottomrule
\end{tabular}
\end{table}

Many previous works \citep{raginsky2017non,mou2018generalization,Li2020On,farghly2021time,futami2023time,dupuis2024uniform} derived generalization bounds specifically for CLD, under different assumptions.
Our bound offers some improvements over previous ones:
\begin{itemize}[leftmargin=15pt]
    \item It is trajectory independent, and does not require gradient statistics \citep{mou2018generalization,futami2023time}.
    \item It does not require very large time scales to make sure we have already converged near Gibbs \citep{raginsky2017non}, nor does it deteriorate with time, as is common for stability-based bounds \citep{mou2018generalization,dupuis2024uniform}. 
    \item It does not depend on the dimension of the parameters, neither explicitly through constants \citep{farghly2021time}, nor implicitly, \eg through the Lipschitz constant or the norms of the gradients \citep{mou2018generalization,Li2020On,futami2023time}.
    In particular, as previously discussed, using standard initialization, our in-expectation bound in \eqref{eq: cld bound in expectation} is dimension independent.
    However, our high-probability bound \eqref{eq: cld bound in probability} relies on the effective supremum at $t = 0$, and may also depend on the dimension if the loss is not bounded. 
    \item The dependence on the inverse temperature $\beta$ and loss' (or expected loss) bound $C$ is polynomial ($\sqrt{\beta C}$) instead of exponential \citep{Li2020On,farghly2021time,futami2023time}.
    \item The bounded expectation assumption in \eqref{eq: cld bound in expectation} is weaker than a uniform bound on the loss \citep{Li2020On,futami2023time}.
    \item \cref{thm: main result sup and mean} and \cref{cor: generalization for ngf box and regularization} demonstrate that our results hold for general initialization-regularization pairs, beyond Gaussian initialization with matching $\ell^2$ regularization.
\end{itemize} 
In \cref{tab: continuous bound comparison} we compare in more detail \cref{cor: generalization for ngf box and regularization} to other bounds that remain bounded as $t \to \infty$.

Finally, \citet{dupuis2024uniform} recently derived bounds on the generalization gap for \emph{all} intermediate times $0 \le s \le t$ \emph{simultaneously}.
Naturally, as avoiding parameters with large generalization gap is increasingly less likely as the process mixes, their bounds grow with time.
Therefore, \citet{dupuis2024uniform}'s bounds are qualitatively different, and higher than most other bounds, including ours.

\subsection{Technical Novelty} 
As a representative example, we first focus on \citet{raginsky2017non}, which provided a bound for CLD (as an intermediate step for deriving a generalization bound for SGLD, a discretized version of CLD).
Using spectral methods \citep[e.g.][]{bakry1985diffusions}, they bound the distance between the process' distribution to the Gibbs posterior, which, when combined with the generalization bound for the Gibbs distribution, results in generalization bounds for intermediate times.
Our \cref{cor: second law + chain rule for d inf and kl} and the preceding arguments are similar to the proof of Lemma 3.4 of \citet{raginsky2017non} that bounds the divergence between the initialization and the Gibbs distribution, where their dissipativity coefficient $m$ corresponds to our explicit $\ell^2$ regularization coefficient $\lambda$.
We use some significant observations that make the bound simpler, and time/dimension/Lipschitz/smoothness independent.
\begin{itemize}[leftmargin=15pt]
    \item Instead of a bound on the convergence of intermediate time distributions to Gibbs, which restricts the result to very large times and introduces exponential dependence on dimensionality through the spectral gap, we only require the monotonic convergence to it.
    As a result, we do not use a spectral gap, but a complexity term for the initial distribution.
    This also enables us to generalize the result to \emph{any} Markov process, relying on $\bbE_{p_0} \generalpotential$ as a complexity term for the Gibbs posterior, which is also included in Lemma 3.4 of \citet{raginsky2017non} along other quantities.
    \item By using a symmetric version of the divergence (\eg by summing $\KL{p}{q}$ and $\KL{q}{p}$) we were able to \emph{completely remove the partition function} from the analysis, avoiding the complications arising from it. 
    \item By separating the regularization from the loss we were able to disentangle their effects.
\end{itemize}

This approach also sidesteps the main difficulties encountered by other works, e.g., using stability-based bounds \citep{mou2018generalization,Li2020On,futami2023time} which either diverge with training time or have dimension dependence.

\subsection{Generalization Guarantees Applicable for Neural Networks}
Many additional lines of work established generalization guarantees applicable for NNs, but are less directly related to our work. 
These results have some limitations that do not exist in ours. 
For example, NTK analysis \citep{jacot2018neural} can imply generalization guarantees in certain settings, but they do not allow for feature learning; Mean-field results \citep{mei2018mean} require non-standard initialization and specific architectures; Algorithmic stability analysis \citet{bousquet2002stability,hardt2016train,richards2021learning,lei2022stability,wang2025generalization} only apply when the number of iterations is sufficiently small; 
Norm-based generalization bounds \citep{bartlett2017spectrally,golowich2018size} ignore optimization aspects and depend exponentially on the network's depth;
And bounds for random interpolators \citep{pmlr-v235-buzaglo24a} involve impractical training procedures.

A closely related setting to the one studied here is SGLD, \ie a discretized version of CLD.
There is an extensive line of work bounding the generalization gap of such models (see \citep{raginsky2017non,mou2018generalization,pensia2018generalization,negrea2019information,farghly2021time,futami2023time,dadi2025generalization} for a partial list).
These results typically have a significant dependence on hyperparameter stemming from the discretization such as the learning rate and batch size, or suffer from constraints similar to the ones discussed in \cref{sec: comparison for cld}, such as dependence on trajectory or dimensionality (\eg via smoothness, parameter norms, log-Sobolev or spectral gap constants).

\section{Discussion, Limitations, and Future Work} \label{sec: discussion and future work}

\textbf{Summary.} We derived a simple generalization bound for general parametric models trained using a Markov-process-based algorithm, where the dynamics have a stationary distribution with bounded potential or expected potential. For CLD with regularization/boundedness constraint matching the initial distribution, we proved that the model generalizes well when the inverse temperature is of order $\beta = O \rb{N}$. There are several interesting directions to extend this result.

\textbf{Non-isotropic noise.}
We can consider a more general model for training, such as
\begin{align*}
    d\param_t = -\nabla L \rb{\param_t} dt + \dispcoef \rb{\param_t} d\bfw_t \,,
\end{align*}
where $\dispcoef$ is a matrix-valued dispersion coefficient. In contrast, in this paper, to derive concrete generalization bounds, we focused on CLD with isotropic noise, \ie such that $\dispcoef$ is a scalar multiple of the identity matrix. The reason for this was that our bound (\cref{cor: generalization for ngf box and regularization}) relies on explicit analytical expressions or bounds on stationary distributions, which are difficult to find in the general case. In addition, in typical overparameterized settings, the noise induced by the randomness of SGD may not only be non-isotropic, but also low-rank.
The analysis of such processes poses various challenges beyond the ability to derive an analytic form for their stationary distribution.
For example, they may concentrate on low-dimensional manifolds, possibly making the KL-divergence term infinite, or making some of the assumptions unrealistic (\eg the choice of initial distribution).

\textbf{No regularization.}
In this work, we only considered processes that have stationary \emph{probability} measures.
For this reason, in the examples in \cref{sec: examples} we used either a bounded domain or regularization.
This seems essential for generalization at $t \to \infty$, unless there are other architectural constrains.
For example, consider training a model for classification of randomly labeled data.
Without regularization, sufficiently expressive models are likely to arrive (at some point) at high training accuracy, yet it cannot generalize in this setting.
Nonetheless, it might be possible to ensure generalization as a function of time, but here we focus on time-independent bounds.

\textbf{Discrete time steps.}
The behavior of SGD with a large step size may be qualitatively different than that of the continuous process considered here.
Specifically, \citet{azizian2024what} showed that while the asymptotic distribution of SGD resembles the Gibbs posterior, it is influenced by the step size, and geometry of the loss surface.
While an extension of our analysis to this setting is straightforward \emph{given} a stationary distribution, such stationary distributions are typically hard to find explicitly (except in simple cases, such as quadratic potentials), and the error terms coming from their approximations are typically detrimental to finding non-vacuous generalization bounds, as they may depend on the dimension of the parameters through the model's Lipschitz or smoothness coefficients, etc. (\citealp{mou2018generalization,Li2020On,futami2023time,dupuis2024uniform}). 
Hence, a direct application of our approach to such algorithms requires additional considerations.
An alternative approach is to incorporate a Metropolis-Hastings type rejection \citep{metropolis1953,hastings1970}, ensuring that the stationary distribution is indeed the Gibbs posterior.

\textbf{Can noise be useful for generalization?} There is a long line of work in the literature (\eg see \citep{geiping2022stochastic} and references therein), debating the effect of noise on generalization.
Our work does not imply that higher noise improves the test error, only that it decreases the gap between training and testing.
Since higher noise {\em could hurt} the training error, the overall effect depends on the specific situation.
Even if introducing noise does not improve test performance, there could still be an advantage to introducing noise, based on our results, in that it reduces the {\em gap} and thus could {\em increase} the training error to match the test error in cases we cannot hope to learn (i.e.~to get small test error).
This is a good thing since it prevents being mislead by overfitting, hopefully without hurting the test error when we can generalize well (i.e.~in learnable regimes, both training and test errors are low, perhaps also without noise, but in non-learnable regimes, where the test error is necessarily high, noise forces the training error to be high as well, so that the gap is small).
Indeed, in our small-scale experiments in \cref{app-sec: experiments}, we noticed that a small amount of noise can decrease the generalization gap, without significantly harming the test error (\eg see the bottom half of \cref{app-tab:MNIST,app-tab:FMNIST,app-tab:SVHN}).
Further analysis is necessary in order to establish general conditions under which test performance is not significantly hurt by noise, while ensuring a small gap.
This, in particular, requires studying the effect of noise on the training loss, and what noise level still ensures obtaining a small training loss in learnable regimes.

\begin{ack}
The research of DS was Funded by the European Union (ERC, A-B-C-Deep, 101039436). 
Views and opinions expressed are however those of the author only and do not necessarily reflect those of the European Union or the European Research Council Executive Agency (ERCEA). 
Neither the European Union nor the granting authority can be held responsible for them. 
DS also acknowledges the support of the Schmidt Career Advancement Chair in AI. 
GV is supported by the Israel Science Foundation (grant No. 2574/25), by a research grant from Mortimer Zuckerman (the Zuckerman STEM Leadership Program), and by research grants from the Center for New Scientists at the Weizmann Institute of Science, and the Shimon and Golde Picker -- Weizmann Annual Grant. 
Part of this work was done as part of the NSF-Simons funded Collaboration on the Mathematics of Deep Learning. 
NS was partially supported by the NSF TRIPOD Institute on Data Economics Algorithms and Learning (IDEAL) and an NSF-IIS award.
\end{ack}


\bibliographystyle{plainnat}
\bibliography{999_biblio}

\begin{thebibliography}{71}
\providecommand{\natexlab}[1]{#1}
\providecommand{\url}[1]{\texttt{#1}}
\expandafter\ifx\csname urlstyle\endcsname\relax
  \providecommand{\doi}[1]{doi: #1}\else
  \providecommand{\doi}{doi: \begingroup \urlstyle{rm}\Url}\fi

\bibitem[Alquier et~al.(2024)]{alquier2024user}
Pierre Alquier et~al.
\newblock User-friendly introduction to pac-bayes bounds.
\newblock \emph{Foundations and Trends{\textregistered} in Machine Learning}, 17\penalty0 (2):\penalty0 174--303, 2024.

\bibitem[Anthony and Bartlett(2009)]{anthony2009neural}
Martin Anthony and Peter~L Bartlett.
\newblock \emph{Neural network learning: Theoretical foundations}.
\newblock cambridge university press, 2009.

\bibitem[Arora et~al.(2018)Arora, Ge, Neyshabur, and Zhang]{arora2018stronger}
Sanjeev Arora, Rong Ge, Behnam Neyshabur, and Yi~Zhang.
\newblock Stronger generalization bounds for deep nets via a compression approach.
\newblock In \emph{International conference on machine learning}, pages 254--263. PMLR, 2018.

\bibitem[Azizian et~al.(2024)Azizian, Iutzeler, Malick, and Mertikopoulos]{azizian2024what}
Wa{\"\i}ss Azizian, Franck Iutzeler, Jerome Malick, and Panayotis Mertikopoulos.
\newblock What is the long-run distribution of stochastic gradient descent? a large deviations analysis.
\newblock In \emph{Forty-first International Conference on Machine Learning}, 2024.
\newblock URL \url{https://openreview.net/forum?id=vsOF7qDNhl}.

\bibitem[Bakry and {\'E}mery(1985)]{bakry1985diffusions}
D.~Bakry and M.~{\'E}mery.
\newblock Diffusions hypercontractives.
\newblock In Jacques Az{\'e}ma and Marc Yor, editors, \emph{S{\'e}minaire de Probabilit{\'e}s XIX 1983/84}, pages 177--206, Berlin, Heidelberg, 1985. Springer Berlin Heidelberg.
\newblock ISBN 978-3-540-39397-9.

\bibitem[Bartlett et~al.(2017)Bartlett, Foster, and Telgarsky]{bartlett2017spectrally}
Peter~L Bartlett, Dylan~J Foster, and Matus~J Telgarsky.
\newblock Spectrally-normalized margin bounds for neural networks.
\newblock \emph{Advances in neural information processing systems}, 30, 2017.

\bibitem[Bousquet and Elisseeff(2002)]{bousquet2002stability}
Olivier Bousquet and Andr\'{e} Elisseeff.
\newblock Stability and generalization.
\newblock \emph{J. Mach. Learn. Res.}, 2:\penalty0 499–526, March 2002.
\newblock ISSN 1532-4435.
\newblock \doi{10.1162/153244302760200704}.
\newblock URL \url{https://doi.org/10.1162/153244302760200704}.

\bibitem[Buzaglo et~al.(2024)Buzaglo, Harel, Nacson, Brutzkus, Srebro, and Soudry]{pmlr-v235-buzaglo24a}
Gon Buzaglo, Itamar Harel, Mor~Shpigel Nacson, Alon Brutzkus, Nathan Srebro, and Daniel Soudry.
\newblock How uniform random weights induce non-uniform bias: Typical interpolating neural networks generalize with narrow teachers.
\newblock In Ruslan Salakhutdinov, Zico Kolter, Katherine Heller, Adrian Weller, Nuria Oliver, Jonathan Scarlett, and Felix Berkenkamp, editors, \emph{Proceedings of the 41st International Conference on Machine Learning}, volume 235 of \emph{Proceedings of Machine Learning Research}, pages 5035--5081. PMLR, 21--27 Jul 2024.
\newblock URL \url{https://proceedings.mlr.press/v235/buzaglo24a.html}.

\bibitem[Catoni(2007)]{catoni2007pac}
Olivier Catoni.
\newblock Pac-bayesian supervised classification: the thermodynamics of statistical learning.
\newblock \emph{arXiv preprint arXiv:0712.0248}, 2007.

\bibitem[Chiang et~al.(2022)Chiang, Ni, Miller, Bansal, Geiping, Goldblum, and Goldstein]{chiang2022loss}
Ping-yeh Chiang, Renkun Ni, David~Yu Miller, Arpit Bansal, Jonas Geiping, Micah Goldblum, and Tom Goldstein.
\newblock Loss landscapes are all you need: Neural network generalization can be explained without the implicit bias of gradient descent.
\newblock In \emph{The Eleventh International Conference on Learning Representations}, 2022.

\bibitem[Cover(1994)]{Cover1994}
Thomas~M. Cover.
\newblock Which processes satisfy the second law?
\newblock In J.~J. Halliwell, J.~Perez-Mercader, and W.~H. Zurek, editors, \emph{Physical Origins of Time Asymmetry}, pages 98--107. Cambridge University Press, New York, 1994.

\bibitem[Cover and Thomas(2001)]{cover2001elements}
Thomas~M. Cover and Joy~A. Thomas.
\newblock \emph{Entropy, Relative Entropy and Mutual Information}, chapter~2, pages 12--49.
\newblock John Wiley \& Sons, Ltd, 2001.
\newblock ISBN 9780471200611.
\newblock \doi{https://doi.org/10.1002/0471200611.ch2}.
\newblock URL \url{https://onlinelibrary.wiley.com/doi/abs/10.1002/0471200611.ch2}.

\bibitem[Dadi and Cevher(2025)]{dadi2025generalization}
Leello~Tadesse Dadi and Volkan Cevher.
\newblock Generalization of noisy {SGD} in unbounded non-convex settings.
\newblock In \emph{Forty-second International Conference on Machine Learning}, 2025.
\newblock URL \url{https://openreview.net/forum?id=Au9rfI6Fjd}.

\bibitem[Dupuis et~al.(2024)Dupuis, Viallard, Deligiannidis, and Simsekli]{dupuis2024uniform}
Benjamin Dupuis, Paul Viallard, George Deligiannidis, and Umut Simsekli.
\newblock Uniform generalization bounds on data-dependent hypothesis sets via pac-bayesian theory on random sets.
\newblock \emph{Journal of Machine Learning Research}, 25\penalty0 (409):\penalty0 1--55, 2024.

\bibitem[Dziugaite and Roy(2017)]{dziugaite2017}
Gintare~Karolina Dziugaite and Daniel~M. Roy.
\newblock Computing nonvacuous generalization bounds for deep (stochastic) neural networks with many more parameters than training data.
\newblock In \emph{Proceedings of the Conference on Uncertainty in Artificial Intelligence}, 2017.

\bibitem[Dziugaite and Roy(2025)]{dziugaite2025size}
Gintare~Karolina Dziugaite and Daniel~M. Roy.
\newblock The size of teachers as a measure of data complexity: Pac-bayes excess risk bounds and scaling laws.
\newblock In Yingzhen Li, Stephan Mandt, Shipra Agrawal, and Emtiyaz Khan, editors, \emph{Proceedings of The 28th International Conference on Artificial Intelligence and Statistics}, volume 258 of \emph{Proceedings of Machine Learning Research}, pages 3979--3987. PMLR, 03--05 May 2025.
\newblock URL \url{https://proceedings.mlr.press/v258/dziugaite25a.html}.

\bibitem[Dziugaite et~al.(2021)Dziugaite, Hsu, Gharbieh, Arpino, and Roy]{dziugaite2021role}
Gintare~Karolina Dziugaite, Kyle Hsu, Waseem Gharbieh, Gabriel Arpino, and Daniel Roy.
\newblock On the role of data in pac-bayes bounds.
\newblock In \emph{International Conference on Artificial Intelligence and Statistics}, pages 604--612. PMLR, 2021.

\bibitem[Farghly and Rebeschini(2021)]{farghly2021time}
Tyler Farghly and Patrick Rebeschini.
\newblock Time-independent generalization bounds for sgld in non-convex settings.
\newblock \emph{Advances in Neural Information Processing Systems}, 34:\penalty0 19836--19846, 2021.

\bibitem[Futami and Fujisawa(2023)]{futami2023time}
Futoshi Futami and Masahiro Fujisawa.
\newblock Time-independent information-theoretic generalization bounds for sgld.
\newblock \emph{Advances in Neural Information Processing Systems}, 36:\penalty0 8173--8185, 2023.

\bibitem[Geiping et~al.(2022)Geiping, Goldblum, Pope, Moeller, and Goldstein]{geiping2022stochastic}
Jonas Geiping, Micah Goldblum, Phil Pope, Michael Moeller, and Tom Goldstein.
\newblock Stochastic training is not necessary for generalization.
\newblock In \emph{International Conference on Learning Representations}, 2022.
\newblock URL \url{https://openreview.net/forum?id=ZBESeIUB5k}.

\bibitem[Glorot and Bengio(2010)]{pmlr-v9-glorot10a}
Xavier Glorot and Yoshua Bengio.
\newblock Understanding the difficulty of training deep feedforward neural networks.
\newblock In Yee~Whye Teh and Mike Titterington, editors, \emph{Proceedings of the Thirteenth International Conference on Artificial Intelligence and Statistics}, volume~9 of \emph{Proceedings of Machine Learning Research}, pages 249--256, Chia Laguna Resort, Sardinia, Italy, 13--15 May 2010. PMLR.
\newblock URL \url{https://proceedings.mlr.press/v9/glorot10a.html}.

\bibitem[Golowich et~al.(2018)Golowich, Rakhlin, and Shamir]{golowich2018size}
Noah Golowich, Alexander Rakhlin, and Ohad Shamir.
\newblock Size-independent sample complexity of neural networks.
\newblock In \emph{Conference On Learning Theory}, pages 297--299. PMLR, 2018.

\bibitem[Gunasekar et~al.(2017)Gunasekar, Woodworth, Bhojanapalli, Neyshabur, and Srebro]{Gunasekar2017}
Suriya Gunasekar, Blake~E Woodworth, Srinadh Bhojanapalli, Behnam Neyshabur, and Nati Srebro.
\newblock Implicit regularization in matrix factorization.
\newblock In I.~Guyon, U.~Von Luxburg, S.~Bengio, H.~Wallach, R.~Fergus, S.~Vishwanathan, and R.~Garnett, editors, \emph{Advances in Neural Information Processing Systems}, volume~30. Curran Associates, Inc., 2017.

\bibitem[Gupta and Nagar(1999)]{gupta1999matrix}
Arjun~K. Gupta and Daya~K. Nagar.
\newblock \emph{Matrix Variate Distributions}.
\newblock Monographs and Surveys in Pure and Applied Mathematics. Chapman \& Hall/CRC, Boca Raton, FL, 1999.
\newblock ISBN 9781584880462.

\bibitem[Hanin(2023)]{hanin2023random}
Boris Hanin.
\newblock Random neural networks in the infinite width limit as gaussian processes.
\newblock \emph{The Annals of Applied Probability}, 33\penalty0 (6A):\penalty0 4798--4819, 2023.

\bibitem[Hardt et~al.(2016)Hardt, Recht, and Singer]{hardt2016train}
Moritz Hardt, Ben Recht, and Yoram Singer.
\newblock Train faster, generalize better: Stability of stochastic gradient descent.
\newblock In \emph{International conference on machine learning}, pages 1225--1234. PMLR, 2016.

\bibitem[Hastings(1970)]{hastings1970}
W.~K. Hastings.
\newblock Monte carlo sampling methods using markov chains and their applications.
\newblock \emph{Biometrika}, 57\penalty0 (1):\penalty0 97--109, 1970.
\newblock ISSN 00063444, 14643510.
\newblock URL \url{http://www.jstor.org/stable/2334940}.

\bibitem[He et~al.(2015)He, Zhang, Ren, and Sun]{he2015delving}
Kaiming He, Xiangyu Zhang, Shaoqing Ren, and Jian Sun.
\newblock Delving deep into rectifiers: Surpassing human-level performance on imagenet classification.
\newblock In \emph{Proceedings of the IEEE international conference on computer vision}, pages 1026--1034, 2015.

\bibitem[Jacot et~al.(2018)Jacot, Gabriel, and Hongler]{jacot2018neural}
Arthur Jacot, Franck Gabriel, and Cl{\'e}ment Hongler.
\newblock Neural tangent kernel: Convergence and generalization in neural networks.
\newblock \emph{Advances in neural information processing systems}, 31, 2018.

\bibitem[Jin et~al.(2017)Jin, Ge, Netrapalli, Kakade, and Jordan]{jin2017escape}
Chi Jin, Rong Ge, Praneeth Netrapalli, Sham~M Kakade, and Michael~I Jordan.
\newblock How to escape saddle points efficiently.
\newblock In \emph{International conference on machine learning}, pages 1724--1732. PMLR, 2017.

\bibitem[Kang and Ramanan(2014)]{kang2014characterization}
Weining Kang and Kavita Ramanan.
\newblock {Characterization of stationary distributions of reflected diffusions}.
\newblock \emph{The Annals of Applied Probability}, 24\penalty0 (4):\penalty0 1329 -- 1374, 2014.
\newblock \doi{10.1214/13-AAP947}.
\newblock URL \url{https://doi.org/10.1214/13-AAP947}.

\bibitem[Kang and Ramanan(2017)]{kang2017submartingale}
Weining Kang and Kavita Ramanan.
\newblock {On the submartingale problem for reflected diffusions in domains with piecewise smooth boundaries}.
\newblock \emph{The Annals of Probability}, 45\penalty0 (1):\penalty0 404 -- 468, 2017.
\newblock \doi{10.1214/16-AOP1153}.
\newblock URL \url{https://doi.org/10.1214/16-AOP1153}.

\bibitem[Kavis et~al.(2022)Kavis, Levy, and Cevher]{kavis2022high}
Ali Kavis, Kfir~Yehuda Levy, and Volkan Cevher.
\newblock High probability bounds for a class of nonconvex algorithms with adagrad stepsize.
\newblock \emph{arXiv preprint arXiv:2204.02833}, 2022.

\bibitem[Kingma and Ba(2015)]{kingma2015adam}
Diederik~P. Kingma and Jimmy Ba.
\newblock Adam: {A} method for stochastic optimization.
\newblock In Yoshua Bengio and Yann LeCun, editors, \emph{3rd International Conference on Learning Representations, {ICLR} 2015, San Diego, CA, USA, May 7-9, 2015, Conference Track Proceedings}, 2015.
\newblock URL \url{http://arxiv.org/abs/1412.6980}.

\bibitem[Lee et~al.(2018)Lee, Bahri, Novak, Schoenholz, Pennington, and Sohl-Dickstein]{lee2018deep}
Jaehoon Lee, Yasaman Bahri, Roman Novak, Samuel~S. Schoenholz, Jeffrey Pennington, and Jascha Sohl-Dickstein.
\newblock Deep neural networks as gaussian processes.
\newblock In \emph{International Conference on Learning Representations (ICLR)}, 2018.
\newblock URL \url{https://arxiv.org/abs/1711.00165}.

\bibitem[Lei et~al.(2022)Lei, Jin, and Ying]{lei2022stability}
Yunwen Lei, Rong Jin, and Yiming Ying.
\newblock Stability and generalization analysis of gradient methods for shallow neural networks.
\newblock \emph{Advances in Neural Information Processing Systems}, 35:\penalty0 38557--38570, 2022.

\bibitem[Levy et~al.(2021)Levy, Kavis, and Cevher]{levy2021storm}
Kfir~Yehuda Levy, Ali Kavis, and Volkan Cevher.
\newblock {STORM}+: Fully adaptive {SGD} with recursive momentum for nonconvex optimization.
\newblock In A.~Beygelzimer, Y.~Dauphin, P.~Liang, and J.~Wortman Vaughan, editors, \emph{Advances in Neural Information Processing Systems}, 2021.
\newblock URL \url{https://openreview.net/forum?id=ytke6qKpxtr}.

\bibitem[Li et~al.(2020)Li, Luo, and Qiao]{Li2020On}
Jian Li, Xuanyuan Luo, and Mingda Qiao.
\newblock On generalization error bounds of noisy gradient methods for non-convex learning.
\newblock In \emph{International Conference on Learning Representations}, 2020.
\newblock URL \url{https://openreview.net/forum?id=SkxxtgHKPS}.

\bibitem[Lotfi et~al.(2022)Lotfi, Finzi, Kapoor, Potapczynski, Goldblum, and Wilson]{lotfi2022pac}
Sanae Lotfi, Marc Finzi, Sanyam Kapoor, Andres Potapczynski, Micah Goldblum, and Andrew~G Wilson.
\newblock Pac-bayes compression bounds so tight that they can explain generalization.
\newblock \emph{Advances in Neural Information Processing Systems}, 35:\penalty0 31459--31473, 2022.

\bibitem[Lyu and Li(2020)]{Lyu2020Gradient}
Kaifeng Lyu and Jian Li.
\newblock Gradient descent maximizes the margin of homogeneous neural networks.
\newblock In \emph{International Conference on Learning Representations}, 2020.
\newblock URL \url{https://openreview.net/forum?id=SJeLIgBKPS}.

\bibitem[Mandt et~al.(2017)Mandt, Hoffman, and Blei]{mandt2017stochastic}
Stephan Mandt, Matthew~D Hoffman, and David~M Blei.
\newblock Stochastic gradient descent as approximate bayesian inference.
\newblock \emph{Journal of Machine Learning Research}, 18\penalty0 (134):\penalty0 1--35, 2017.

\bibitem[Matthews et~al.(2018)Matthews, Hron, Rowland, Turner, and Ghahramani]{matthews2018gaussian}
Alexander G de~G Matthews, Jiri Hron, Mark Rowland, Richard~E Turner, and Zoubin Ghahramani.
\newblock Gaussian process behaviour in wide deep neural networks.
\newblock In \emph{International Conference on Learning Representations}, 2018.

\bibitem[Maurer(2004)]{maurer2004note}
Andreas Maurer.
\newblock A note on the pac bayesian theorem.
\newblock \emph{arXiv preprint cs/0411099}, 2004.

\bibitem[McAllester(1998)]{mcallester1998some}
David~A McAllester.
\newblock Some pac-bayesian theorems.
\newblock In \emph{Proceedings of the eleventh annual conference on Computational learning theory}, pages 230--234, 1998.

\bibitem[Mei et~al.(2018)Mei, Montanari, and Nguyen]{mei2018mean}
Song Mei, Andrea Montanari, and Phan-Minh Nguyen.
\newblock A mean field view of the landscape of two-layer neural networks.
\newblock \emph{Proceedings of the National Academy of Sciences}, 115\penalty0 (33):\penalty0 E7665--E7671, 2018.

\bibitem[Merhav(2011)]{Merhav2011}
Neri Merhav.
\newblock Data processing theorems and the second law of thermodynamics.
\newblock \emph{IEEE Transactions on Information Theory}, 57\penalty0 (8):\penalty0 4926--4939, 2011.
\newblock \doi{10.1109/TIT.2011.2159052}.

\bibitem[Metropolis et~al.(1953)Metropolis, Rosenbluth, Rosenbluth, Teller, and Teller]{metropolis1953}
Nicholas Metropolis, Arianna~W. Rosenbluth, Marshall~N. Rosenbluth, Augusta~H. Teller, and Edward Teller.
\newblock Equation of state calculations by fast computing machines.
\newblock Technical report, Los Alamos Scientific Lab., Los Alamos, NM (United States); Univ. of Chicago, IL (United States), 03 1953.
\newblock URL \url{https://www.osti.gov/biblio/4390578}.

\bibitem[Mnih et~al.(2015)Mnih, Kavukcuoglu, Silver, Rusu, Veness, Bellemare, Graves, Riedmiller, Fidjeland, Ostrovski, et~al.]{mnih2015human}
Volodymyr Mnih, Koray Kavukcuoglu, David Silver, Andrei~A Rusu, Joel Veness, Marc~G Bellemare, Alex Graves, Martin Riedmiller, Andreas~K Fidjeland, Georg Ostrovski, et~al.
\newblock Human-level control through deep reinforcement learning.
\newblock \emph{nature}, 518\penalty0 (7540):\penalty0 529--533, 2015.

\bibitem[Mou et~al.(2018)Mou, Wang, Zhai, and Zheng]{mou2018generalization}
Wenlong Mou, Liwei Wang, Xiyu Zhai, and Kai Zheng.
\newblock Generalization bounds of sgld for non-convex learning: Two theoretical viewpoints.
\newblock In \emph{Conference on Learning Theory}, pages 605--638. PMLR, 2018.

\bibitem[Neal(1996)]{Neal1996}
Radford~M. Neal.
\newblock \emph{Priors for Infinite Networks}, pages 29--53.
\newblock Springer New York, New York, NY, 1996.
\newblock ISBN 978-1-4612-0745-0.
\newblock \doi{10.1007/978-1-4612-0745-0_2}.
\newblock URL \url{https://doi.org/10.1007/978-1-4612-0745-0_2}.

\bibitem[Negrea et~al.(2019)Negrea, Haghifam, Dziugaite, Khisti, and Roy]{negrea2019information}
Jeffrey Negrea, Mahdi Haghifam, Gintare~Karolina Dziugaite, Ashish Khisti, and Daniel~M Roy.
\newblock Information-theoretic generalization bounds for sgld via data-dependent estimates.
\newblock \emph{Advances in Neural Information Processing Systems}, 32, 2019.

\bibitem[Nesterov(1983)]{Nesterov1983AMF}
Yurii Nesterov.
\newblock A method for solving the convex programming problem with convergence rate $o(1/k^2)$.
\newblock \emph{Proceedings of the USSR Academy of Sciences}, 269:\penalty0 543--547, 1983.
\newblock URL \url{https://api.semanticscholar.org/CorpusID:145918791}.

\bibitem[Netzer et~al.(2011)Netzer, Wang, Coates, Bissacco, Wu, and Ng]{netzer2011reading}
Yuval Netzer, Tao Wang, Adam Coates, Alessandro Bissacco, Bo~Wu, and Andrew~Y. Ng.
\newblock Reading digits in natural images with unsupervised feature learning.
\newblock In \emph{NIPS Workshop on Deep Learning and Unsupervised Feature Learning 2011}, 2011.
\newblock URL \url{http://ufldl.stanford.edu/housenumbers/nips2011_housenumbers.pdf}.

\bibitem[{\O}ksendal(2003)]{Oksendal2003}
Bernt {\O}ksendal.
\newblock \emph{Stochastic Differential Equations}, pages 65--84.
\newblock Springer Berlin Heidelberg, Berlin, Heidelberg, 2003.
\newblock ISBN 978-3-642-14394-6.
\newblock \doi{10.1007/978-3-642-14394-6_5}.
\newblock URL \url{https://doi.org/10.1007/978-3-642-14394-6_5}.

\bibitem[Pensia et~al.(2018)Pensia, Jog, and Loh]{pensia2018generalization}
Ankit Pensia, Varun Jog, and Po-Ling Loh.
\newblock Generalization error bounds for noisy, iterative algorithms.
\newblock In \emph{2018 IEEE International Symposium on Information Theory (ISIT)}, pages 546--550. IEEE, 2018.

\bibitem[Petersen and Pedersen(2012)]{petersen2012cookbook}
K.~B. Petersen and M.~S. Pedersen.
\newblock The matrix cookbook, nov 2012.
\newblock URL \url{http://www2.compute.dtu.dk/pubdb/pubs/3274-full.html}.
\newblock Version 20121115.

\bibitem[Pilipenko(2014)]{sder-pilipenko-2014}
Andrey Pilipenko.
\newblock An introduction to stochastic differential equations with reflection, 09 2014.

\bibitem[Raginsky et~al.(2016)Raginsky, Rakhlin, Tsao, Wu, and Xu]{raginsky2016information}
Maxim Raginsky, Alexander Rakhlin, Matthew Tsao, Yihong Wu, and Aolin Xu.
\newblock Information-theoretic analysis of stability and bias of learning algorithms.
\newblock In \emph{2016 IEEE Information Theory Workshop (ITW)}, pages 26--30, 2016.
\newblock \doi{10.1109/ITW.2016.7606789}.

\bibitem[Raginsky et~al.(2017)Raginsky, Rakhlin, and Telgarsky]{raginsky2017non}
Maxim Raginsky, Alexander Rakhlin, and Matus Telgarsky.
\newblock Non-convex learning via stochastic gradient langevin dynamics: a nonasymptotic analysis.
\newblock In \emph{Conference on Learning Theory}, pages 1674--1703. PMLR, 2017.

\bibitem[Richards and Rabbat(2021)]{richards2021learning}
Dominic Richards and Mike Rabbat.
\newblock Learning with gradient descent and weakly convex losses.
\newblock In \emph{International Conference on Artificial Intelligence and Statistics}, pages 1990--1998. PMLR, 2021.

\bibitem[Russo and Zou(2020)]{russo2020how}
Daniel Russo and James Zou.
\newblock How much does your data exploration overfit? controlling bias via information usage.
\newblock \emph{IEEE Transactions on Information Theory}, 66\penalty0 (1):\penalty0 302--323, 2020.
\newblock \doi{10.1109/TIT.2019.2945779}.

\bibitem[Schulman et~al.(2017)Schulman, Wolski, Dhariwal, Radford, and Klimov]{schulman2017proximal}
John Schulman, Filip Wolski, Prafulla Dhariwal, Alec Radford, and Oleg Klimov.
\newblock Proximal policy optimization algorithms.
\newblock \emph{arXiv preprint arXiv:1707.06347}, 2017.

\bibitem[Schuss(2013)]{Schuss2013}
Zeev Schuss.
\newblock \emph{Euler's Scheme and Wiener's Measure}, pages 35--88.
\newblock Springer New York, New York, NY, 2013.
\newblock ISBN 978-1-4614-7687-0.
\newblock \doi{10.1007/978-1-4614-7687-0_2}.
\newblock URL \url{https://doi.org/10.1007/978-1-4614-7687-0_2}.

\bibitem[Soudry et~al.(2018)Soudry, Hoffer, Nacson, Gunasekar, and Srebro]{soudry2018implicit}
Daniel Soudry, Elad Hoffer, Mor~Shpigel Nacson, Suriya Gunasekar, and Nathan Srebro.
\newblock The implicit bias of gradient descent on separable data.
\newblock \emph{Journal of Machine Learning Research}, 19\penalty0 (70):\penalty0 1--57, 2018.

\bibitem[Van~Erven and Harremos(2014)]{van2014renyi}
Tim Van~Erven and Peter Harremos.
\newblock R{\'e}nyi divergence and kullback-leibler divergence.
\newblock \emph{IEEE Transactions on Information Theory}, 60\penalty0 (7):\penalty0 3797--3820, 2014.

\bibitem[Vardi(2023)]{vardi2023implicit}
Gal Vardi.
\newblock On the implicit bias in deep-learning algorithms.
\newblock \emph{Communications of the ACM}, 66\penalty0 (6):\penalty0 86--93, 2023.

\bibitem[Wang et~al.(2025)Wang, Lei, Wang, Ying, and Zhou]{wang2025generalization}
Puyu Wang, Yunwen Lei, Di~Wang, Yiming Ying, and Ding-Xuan Zhou.
\newblock Generalization guarantees of gradient descent for shallow neural networks.
\newblock \emph{Neural Computation}, 37\penalty0 (2):\penalty0 344--402, 2025.

\bibitem[Wenger et~al.(2025)Wenger, Coker, Marusic, and Cunningham]{wenger2025variational}
Jonathan Wenger, Beau Coker, Juraj Marusic, and John~P Cunningham.
\newblock Variational deep learning via implicit regularization.
\newblock \emph{arXiv preprint arXiv:2505.20235}, 2025.

\bibitem[Wenzel et~al.(2020)Wenzel, Roth, Veeling, Swiatkowski, Tran, Mandt, Snoek, Salimans, Jenatton, and Nowozin]{wenzel2020good}
Florian Wenzel, Kevin Roth, Bastiaan Veeling, Jakub Swiatkowski, Linh Tran, Stephan Mandt, Jasper Snoek, Tim Salimans, Rodolphe Jenatton, and Sebastian Nowozin.
\newblock How good is the bayes posterior in deep neural networks really?
\newblock In \emph{International Conference on Machine Learning}, pages 10248--10259. PMLR, 2020.

\bibitem[Xu and Raginsky(2017)]{xu2017information}
Aolin Xu and Maxim Raginsky.
\newblock Information-theoretic analysis of generalization capability of learning algorithms.
\newblock \emph{Advances in neural information processing systems}, 30, 2017.

\bibitem[Zhang et~al.(2017)Zhang, Bengio, Hardt, Recht, and Vinyals]{Zhang16}
Chiyuan Zhang, Samy Bengio, Moritz Hardt, Benjamin Recht, and Oriol Vinyals.
\newblock Understanding deep learning requires rethinking generalization.
\newblock In \emph{International Conference on Learning Representations}, 2017.

\end{thebibliography}


\newpage

\appendix

\paragraph{Appendix structure:}
\begin{itemize}
    \item In \cref{app-sec: preliminaries and auxiliaries} we recap and establish notation and conventions, and present some well-known lemmas.
    \item In \cref{app-sec: proof of new main} we prove \cref{thm: main result sup and mean} and its related claims in \cref{sec:general Markov Processes}.
    \item In \cref{app-sec: tightness and necessity} we discuss the tightness and necessity of the divergence conditions found in \cref{app-sec: proof of new main}.
    \item In \cref{apps-sec:CLD extentions} we prove a generalized version of \cref{cor: generalization for ngf box and regularization}.
    \item The bounds found in this paper only bound the generalization gap, and not the absolute error of a model. 
    In \cref{app-sec: quadratic objective} and \cref{app-sec: experiments} we study the applicability of our bound in realistic settings.
    Specifically, whether the regime in which the bound on the generalization gap is non-vacuous allows for meaningful learning, \ie coincides with a regime in which the absolute training error is also small.
    In \cref{app-sec: quadratic objective} we study linear regression trained with CLD, for which we can analytically characterize the training loss, and in \cref{app-sec: experiments} we experiment with NNs trained with SGLD (discretized version of CLD) on standard training sets.
    \item As \cref{sec: examples} deals only with models trained with some form of regularization, it is natural to ask whether the regularization alone is sufficient for the use of uniform convergence to arrive the desired generalization bounds. 
    In \cref{app-sec: no_uc} we show that the regularization used in \cref{sec: examples} is not sufficient for such bounds, and that the models can remain highly expressive.
    \item Finally, for completeness, in \cref{app-sec: background} we recall some definitions and properties related to SDEs used throughout the paper. 
\end{itemize}

\section{Preliminary and Auxiliary Results} \label{app-sec: preliminaries and auxiliaries}

\subsection{Preliminaries} \label{app-sec: preliminaries}

We start by restating and introducing notation.

\paragraph{Notation.} 
We use bold lowercase letters (\eg $\mathbf{x} \in \bbR^\PARAMDIM$) to denote vectors, bold capital letters to denote matrices (\eg $\mathbf{A} \in \bbR^{\PARAMDIM \times \PARAMDIM}$), and regular capital letters to denote random elements (\eg $S, X, Y$).
We may deviate from these conventions when it does not create confusion.
Unless stated otherwise, all vectors are assumed to be column vectors.
Specifically, we use $\bfe_i \in \bbR^\PARAMDIM$, $i=1, \dots, \PARAMDIM$, to denote the standard basis vector with $1$ in the $i^{th}$ entry, and 0 elsewhere.
For a subset $\DOM \subseteq \bbR^\PARAMDIM$, we denote by $\DOMCLOS$, $\partial \DOM$, and $\DOMINT$, the closure, boundary, and interior of $\DOM$, respectively.
In addition, we denote the volume of $\SUBDOM \subset \DOM$, when it is defined, by $\abs{\SUBDOM}$.
With some abuse of notation, when $\SUBDOM$ is finite we denote its cardinality by $\abs{\SUBDOM}$.
We use $\norm{\cdot}$ for the standard Euclidean norm on $\bbR^\PARAMDIM$.
Then, the open Euclidean ball centered at $\bfx \in \bbR^\PARAMDIM$ with radius $r>0$ is $B_r \rb{\bfx} = \cb{\bfy \in \bbR^\PARAMDIM \,\mid\, \norm{\bfy - \bfx} < r}$.
In addition, we use $\ind{\cdot}{}$ for the indicator function, and specifically for $A \subset \bbR^\PARAMDIM$ and $\bfx \in \bbR^\PARAMDIM$, $\ind{x}{A} = \ind{\bfx \in A}{}$.
We denote the set of all probability measures over $\DOM$ by $\Delta \rb{\DOM}$.
For some $\mu \in \Delta \rb{\DOM}$ with density $p$, with some abuse of notation we denote $p \in \Delta \rb{\DOM}$, and $p\rb{\SUBDOM} = \mu \rb{\SUBDOM}$ for measurable $\SUBDOM \subseteq \DOM$. 
In addition, we use $\bbE_{X\sim p}$ or $\bbE_{p}$ to denote the expectation w.r.t $p$, and omit the subscript when it can be inferred.
For two distributions $\mu, \nu$ with densities $p, q$ we denote by $\KL{\mu}{\nu} = \KL{p}{q}$ their KL-divergence (relative entropy).
Furthermore, we use $H \rb{\delta} = -\delta \ln \rb{\delta} - \rb{1 - \delta} \ln \rb{1 - \delta}$, $\delta \in \bb{0, 1}$, for the binary entropy function (in nats).
We denote the divergence of a vector field by $\nabla \cdot $, and the gradient and Laplacian of a scalar function by $\nabla$ and $\Delta = \nabla \cdot \nabla$, respectively.
Given a domain $E \subset \bbR^k$ and $k \in \bbZ_+ \cup \cb{\infty}$, we denote by $\calC^k \rb{E}$ the set of real valued functions that are continuous over $\bar{E}$, and $k$-times continuously differentiable with continuous partial derivatives in $E$.
In particular, $\calC = \calC^0$ is the set of continuous functions.

\paragraph{Conventions.}
Unless stated otherwise, we use $\DOM \subset \bbR^\PARAMDIM$ to denote a non-empty, connected, and open domain.
In addition, we follow the following naming conventions for probability distributions.
\begin{itemize}
    \item For a discrete/continuous-time Markov process, we use $p_n$ or $p_t$ for its marginal distribution at time $n \in \bbN$ or $t \in \bbR_+$.
    \item We denote stationary distributions of Markov processes by $\stationarydist$.
    \item In the context of PAC-Bayesian theory, we denote prior distributions by $\priordist$, and data dependent posteriors by $\posteriordist = \posteriordist_S$.
    \item In case some stationary distribution is also data-dependent, we use $\posteriorstationary$.
    \item We also use $p,q$ for generic distributions, or modify the pervious notation.
\end{itemize}

\subsection{General Lemmas: Data processing inequality and generalized second laws of thermodynamics} \label{app-sec: auxiliary general}
For completeness, we start by proving some well known results in probability and the theory of Markov processes.

\begin{lemma}[Data processing inequality] \label{app-lem: data processing}
Let $p\left(x,y\right)$ and $q\left(x,y\right)$
be the densities of two joint distributions over a product measure space ${\cal X}\times{\cal Y}$.
Denote by $p_{X}\left(x\right),q_{X}\left(x\right)$ the marginal
densities, \eg
\[
p_{X}\left(x\right)=\intop_{{\cal Y}}p\left(x,y\right)dy\,,
\]
and by $p\left(y\mid x\right),q\left(y\mid x\right)$ the conditional
densities, so $p\left(x,y\right)=p\left(y\mid x\right)p_{X}\left(x\right)$,
and similarly for $q$. Then 
\[
\KL{p_{X}}{q_{X}} \le \KL{p}{q} \,.
\]
\end{lemma}

\begin{proof}
By definition of the KL divergence 
\begin{align*}
\KL{p}{q} & =\intop_{{\cal X}\times{\cal Y}}p\left(x,y\right)\ln\left(\frac{p\left(x,y\right)}{q\left(x,y\right)}\right)dxdy\\
 & =\intop_{{\cal X}\times{\cal Y}}p\left(x,y\right)\ln\left(\frac{p\left(y\mid x\right)p_{X}\left(x\right)}{q\left(y\mid x\right)q_{X}\left(x\right)}\right)dxdy\\
 & =\intop_{{\cal X}\times{\cal Y}}p\left(x,y\right)\ln\left(\frac{p_{X}\left(x\right)}{q_{X}\left(x\right)}\right)dxdy+\intop_{{\cal X}\times{\cal Y}}p\left(x,y\right)\ln\left(\frac{p\left(y\mid x\right)}{q\left(y\mid x\right)}\right)dxdy\\
 & =\intop_{{\cal X}\times{\cal Y}}p\left(y\mid x\right)p_{X}\left(x\right)\ln\left(\frac{p_{X}\left(x\right)}{q_{X}\left(x\right)}\right)dxdy+\intop_{{\cal X}\times{\cal Y}}p_{X}\left(x\right)p\left(y\mid x\right)\ln\left(\frac{p\left(y\mid x\right)}{q\left(y\mid x\right)}\right)dxdy\\
\left[\text{Fubini}\right] & =\intop_{{\cal X}}p_{X}\left(x\right)\ln\left(\frac{p_{X}\left(x\right)}{q_{X}\left(x\right)}\right)dx+\mathbb{E}_{X\sim p_{X}}\intop_{{\cal Y}}p\left(y\mid X\right)\ln\left(\frac{p\left(y\mid X\right)}{q\left(y\mid X\right)}\right)dy\\
 & =\KL{p_{X}}{ q_{X}}+\mathbb{E}_{X\sim p_{X}}\KL{p\left(\cdot\mid X\right)}{ q\left(\cdot\mid X\right)}\,.
\end{align*}
The KL divergence is non-negative and therefore the expectation in
the last line is non-negative as well, and we conclude that 
\[
\KL{p}{ q}\ge \KL{p_{X}}{ q_{X}}\,.
\]
\end{proof}

Let $X_{n} = \cb{X_n}_{n=0}^\infty$ be a discrete-time Markov chain on $\DOM \subset \bbR^\PARAMDIM$, with transition kernel $P\left(y\mid x\right)$ such that for all $n\in\bbN_0$, 
\[
p_{n+1}\left(y\right)=\intop_\DOM P\left(y\mid x\right)p_{n}\left(x\right)dx\,.
\]
In addition, assume that the there exists an invariant distribution $\stationarydist$ such that 
\[
\stationarydist\left(y\right)=\intop_\DOM P\left(y\mid x\right)\stationarydist\left(x\right)dx\,.
\]

We proceed to present a generalized form of the second law of thermodynamics, regarding the monotonicity of the (relative) entropy of Markov processes with possibly non-uniform stationary distributions \citep{Cover1994,cover2001elements}.

\begin{lemma}[Generalized second law of thermodynamics] \label{app-lem: second law discrete time} 
For all $n\ge0$, 
\[
\KL{p_{n+1}}{\stationarydist}\le \KL{p_{n}}{\stationarydist}\,.
\]
\end{lemma}

\begin{proof} First, note that we can assume that
$
\KL{p_{n}}{\stationarydist}<\infty\,,
$
since otherwise the claim holds trivially.
Let $q\left(x,y\right)=p_{n}\left(x\right)P\left(y\mid x\right)$ be the
joint densities of $\left(X_{n},X_{n+1}\right)$ where $X_{n}\sim p_{n}$,
and let $r\left(x,y\right)=\stationarydist\left(x\right)P\left(y\mid x\right)$ be
the joint distribution under $X_{n}\sim\stationarydist$. By definition of $p_{n+1}$,
\[
q_{Y}\left(y\right)=p_{n+1}\left(y\right)\,,
\]
and by definition of the stationary distribution, 
\[
r^{Y}\left(y\right)=\stationarydist\left(y\right)\,.
\]
Therefore according to \cref{app-lem: data processing},
\[
\KL{p_{n+1}}{\stationarydist}\le \KL{q}{ r}\,.
\]
In addition, 
\begin{align*}
\KL{q}{ r} & =\intop_{\DOM \times \DOM} q\left(x,y\right)\ln\left(\frac{q\left(x,y\right)}{r\left(x,y\right)}\right)dxdy\\
 & =\intop_{\DOM \times \DOM} q\left(x,y\right)\ln\left(\frac{p_{n}\left(x\right)P\left(y\mid x\right)}{\stationarydist\left(x\right)P\left(y\mid x\right)}\right)dxdy\\
 & =\intop_{\DOM \times \DOM} q\left(x,y\right)\ln\left(\frac{p_{n}\left(x\right)}{\stationarydist\left(x\right)}\right)dxdy\\
 & =\intop_{\DOM \times \DOM} p_{n}\left(x\right)P\left(y\mid x\right)\ln\left(\frac{p_{n}\left(x\right)}{\stationarydist\left(x\right)}\right)dxdy\\
\left[\text{Fubini}\right] & =\intop_{\DOM} p_{n}\left(x\right)\ln\left(\frac{p_{n}\left(x\right)}{\stationarydist\left(x\right)}\right)dx\\
 & =\KL{p_{n}}{\stationarydist}\,,
\end{align*}
and overall 
\[
\KL{p_{n+1}}{\stationarydist}\le \KL{p_{n}}{\stationarydist}\,.
\]
\end{proof}
A similar result can be obtained form $\Dinf{\cdot}{\cdot}$.
\begin{lemma}[The Pointwise Second Law]\label{app-lem:pointwise second law}
For all $n>0:$
\[
\Dinf{p_{n+1}}{\stationarydist}\;\le\;\Dinf{p_{n}}{\stationarydist}.
\]
\end{lemma}
\begin{proof}
Let $p, q$ be some probability measures such that $\frac{\diff p}{\diff q}$ exists.
By definition, 
\begin{align*}
\Dinf{p}{q} = \esssup_{q} \ln \frac{\diff p}{\diff q} = \inf \cb{ c \in \bbR \setst q \rb{ \cb{ x \setst \ln \frac{\diff p}{\diff q} > c}} = 0 } \,.
\end{align*}

Let $C \in \bbR$ and suppose that for all measurable $A \subset \calX$, $p\rb{A} \le e^C q\rb{A}$.
Assume by way of contradiction that $\Dinf{p}{q} > C$, that is, that there exists $c > C$ such that 
\begin{align*}
    q \rb{ \cb{ x \setst \ln \frac{\diff p}{\diff q} > c}} > 0 \,.
\end{align*}
Denote 
\begin{align*}
    A = \cb{ x \setst \ln \frac{\diff p}{\diff q} > c} \,,
\end{align*}
then 
\begin{align*}
    p \rb{A} = \intop_A \frac{\diff p}{\diff q} \diff q > e^c q \rb{A} > e^C q \rb{A} \,,
\end{align*}
in contradiction to the assumption.
Therefore, for all $C$ such that $p \rb{A} \le e^C q \rb{A}$ for all measurable $A$, $C \ge \Dinf{p}{q}$.
We can now show the claim.

Let $P (\diff y \mid x)$ be the processes' transition kernel (in measure form).
We can assume $\Dinf{p_{n}}{\stationarydist}<\infty$, since otherwise the claim holds trivially. 
Let $A$ be measurable, then by definition, 
\begin{align*}
    p_{n+1} \rb{A} &= \intop P \rb{A \mid x} dp_n \rb{x} = \intop P \rb{A \mid x} \frac{\diff p_n}{\diff p_\infty} \rb{x} \diff p_\infty \rb{x} \\
    & \le e^{\Dinf{p_n}{p_\infty}} \intop P\rb{A \mid x} \diff p_\infty \rb{x} = e^{\Dinf{p_n}{p_\infty}} p_\infty \rb{A} \,,
\end{align*}
so $\Dinf{p_{n+1}}{p_\infty} \le \Dinf{p_n}{p_\infty}$.
\end{proof}

We can now state the relevant results for continuous-time processes. 

\begin{corollary} \label{app-cor: second law continuous time}
Let $X_{t}$ be a Markov process with marginals $p_{t}$ and stationary
distribution $\stationarydist$. 
Then, for all $t>0:$
\[
\KL{p_{t}}{\stationarydist}\le \KL{p_{0}}{\stationarydist}\; \mathrm{or} \; \Dinf{p_{t}}{\stationarydist}\le\Dinf{p_{0}}{\stationarydist}
\]
\end{corollary}

\begin{proof}
Let $0 < t$ and let $\Delta t>0$ such that $t\in\Delta t\cdot\mathbb{N}$.
Define $Y_{n}=X_{n\Delta t}$, then $Y_{n}$ is a discrete time Markov chain with marginals $p_{n \cdot \Delta t}$  and stationary distribution $\stationarydist$, so \cref{app-lem: second law discrete time} and \cref{app-lem:pointwise second law} imply the results.
\end{proof}

\newpage
\section{Proof of \texorpdfstring{\cref{thm: main result sup and mean}}{} and its Related Claims in \texorpdfstring{\cref{sec:general Markov Processes}}{}} \label{app-sec: proof of new main}

In this section, we present the proof of \cref{thm: main result sup and mean}, the claims leading to it, and some of its generalizations.

\subsection{Derivation of \texorpdfstring{\cref{cor: second law + chain rule for d inf and kl}}{}}

\begin{recall}[\cref{claim: sum of dists is difference of expectations or extremas}]
If $p,q,\mu,\nu$ are probability measures, and $p$ is Gibbs w.r.t $q$ with potential $\generalpotential < \infty$, then
\begin{enumerate}[ leftmargin=15pt]
    \item $\KLrel{\mu}{p}{q} + \KLrel{\nu}{q}{p} = \bbE_{\nu} \generalpotential - \bbE_{\mu} \generalpotential$, 
    \item $\Dinfrel{\mu}{p}{q}+\Dinfrel{\nu}{q}{p} = \esssup_\nu \generalpotential - \essinf_{\mu} \generalpotential$.
\end{enumerate}
In particular, $\KL{p}{q}+\KL{q}{p} = \bbE_{q} \generalpotential - \bbE_{p} \generalpotential$, and $\Dinf{p}{q}+\Dinf{q}{p} = \esssup_q \generalpotential - \essinf_p \generalpotential$.
\end{recall}

\begin{proof}
    By definition $\frac{\diff p}{\diff q} = Z^{-1} e^{-\generalpotential}$
    where $Z < \infty$ is the appropriate partition function.
    Then we have 
    \begin{align*}
        \KLrel{\mu}{p}{q} + \KLrel{\nu}{q}{p} &= \int \diff \mu \ln \frac{\diff p}{\diff q} + \int \diff \nu \ln \frac{\diff q}{\diff p} \\
        & = \int \rb{- \generalpotential - \ln Z} \diff \mu + \int \rb{\generalpotential + \ln Z} \diff \nu = \bbE_{\nu} \generalpotential - \bbE_{\mu} \generalpotential \,.
    \end{align*}
    Also, 
    \begin{align*}
        \Dinfrel{\mu}{p}{q} &+ \Dinfrel{\nu}{q}{p} = \ln \rb{\esssup_\mu \frac{\diff p}{\diff q}} + \ln \rb{\esssup_\nu \frac{\diff q}{\diff p}} \\
        & = \esssup_\mu \rb{-\generalpotential - \ln Z} + \esssup_\nu \rb{\generalpotential + \ln Z} = \esssup_\nu \generalpotential - \essinf_\mu \generalpotential \,,
    \end{align*}
    where in the last equality we used the fact that $\esssup \rb{- \generalpotential} = - \essinf \generalpotential$, and that $Z$ is a constant.
\end{proof}

Using the Chain Rule and \cref{claim: second law}, we derive the bounds of \eqref{eq:klchain} and \eqref{eq:dinfchain}, as re-stated and established in the following lemma.

\begin{lemma} \label{app-lem: triangle like inequalities for kl and d infty with second law}
    If $p_t$ is the marginal distribution of a Markov process with initial distribution $p_0$ at time $t$, $\stationarydist$ is a stationary distribution, and $\nu$ is a probability measure, then
    \begin{align*}
        \KL{p_t}{\nu} \le \KL{p_0}{\nu} + \KLrel{p_0}{\nu}{p_\infty} + \KLrel{p_t}{p_\infty}{\nu} \,,
    \end{align*}
    and similarly,
    \begin{align*}
        \Dinf{p_t}{\nu} \leq \Dinf{p_0}{\nu} + \Dinfrel{p_0}{\nu}{p_\infty} + \Dinfrel{p_t}{p_\infty}{\nu} \,.
    \end{align*}
\end{lemma}

\begin{proof}
    This is a simple application of the chain rule,
    \begin{align*}
        \KL{p_t}{\nu} &= \int \diff p_t \ln \frac{\diff p_t}{\diff \nu} = \int \diff p_t \ln \frac{\diff p_t}{\diff p_\infty} \frac{\diff p_\infty}{\diff \nu} = \KL{p_t}{p_\infty} + \KLrel{p_t}{p_\infty}{\nu} \\
        &\le \KL{p_0}{p_\infty} + \KLrel{p_t}{p_\infty}{\nu} = \KL{p_0}{\nu} + \KLrel{p_0}{\nu}{p_\infty} + \KLrel{p_t}{p_\infty}{\nu} \,,
    \end{align*}
    where in the first inequality we used \cref{claim: second law}.
    Similarly,
    \begin{align*}
        \Dinf{p_t}{\nu} &= \esssup_{p_t} \ln \frac{\diff p_t}{\diff \nu} = \esssup_{p_t} \ln \frac{\diff p_t}{\diff p_\infty} \frac{\diff p_\infty}{\diff \nu}
        \leq \Dinf{p_t}{p_\infty} + \Dinfrel{p_t}{p_\infty}{\nu} \\
        & \le \Dinf{p_0}{p_\infty} + \Dinfrel{p_t}{p_\infty}{\nu} = \Dinf{p_0}{\nu} + \Dinfrel{p_0}{\nu}{p_\infty} + \Dinfrel{p_t}{p_\infty}{\nu} \,.
    \end{align*}
\end{proof}

\cref{cor: second law + chain rule for d inf and kl} now follows from plugging in \Cref{claim: sum of dists is difference of expectations or extremas} into \Cref{app-lem: triangle like inequalities for kl and d infty with second law}.

Given these bounds on the divergences, All that remains in order to prove \Cref{thm: main result sup and mean} is plugging \cref{cor: second law + chain rule for d inf and kl} into a PAC-Bayes bound.

\subsection{In-Expectation PAC-Bayes Bounds}

\begin{theorem}[Theorem 5 from \citet{maurer2004note}] \label{app-thm: gen bound kl tolstikhin}
    For any $\delta\in(0,1)$ and any $N \geq 8$, for any data-independent prior distribution $\priordist$:
    \begin{align*}
    \bbP_{S \sim \datadist^N} \rb{ \forall_{\posteriordist \;}
        \kl{\bbE_{h\sim \posteriordist} \trainfail \rb{h}}{\bbE_{h\sim \posteriordist} \expectedfail \rb{h}} \le \frac{\KL{\posteriordist}{\priordist} + \ln \frac{2\sqrt{N}}{\delta}}{N} } \geq 1-\delta \,,
    \end{align*}
    where $\kl{a}{b}= a \ln\tfrac{ a }{b}+(1- a )\ln\tfrac{1- a }{1-b}$ for $0\leq  a ,b \leq 1$ is the KL divergence for a Bernoulli random variable, and $\posteriordist$ denotes a posterior distribution.
\end{theorem}

\subsection{Single-Sample PAC-Bayes Bounds}

\cref{app-thm: gen bound kl tolstikhin} can be viewed as a bound in expectation over the draw from the posterior, which corresponds to the traditional PAC-Bayes view of considering the expected error of a randomized predictor.
But it is actually possible to get guarantees for a single draw from this predictor, which is more appropriate when we view the randomness as part of the training algorithm, that then outputs a single deterministic predictor (chosen at random).
High probability guarantees for a single draw from the posterior were shown by \citet{alquier2024user} based on \citet{catoni2007pac} and also discussed by \citet{dziugaite2025size}.
Here we present a tight version based on a simple modification to Maurer's proof \citep{maurer2004note}.

\begin{theorem}\label{app-thm: gen bound kl single}
For any $\delta\in(0,1)$ and $N\geq 8$, for any data independent prior $\rho$, and any learning rule specified by a conditional probability $h|S\sim\hat{\rho}_S$ such that $\priordist \ll \posteriordist_S$ $S$-a.s.,
    \begin{align*}
       \bbP_{S \sim \datadist^N, h \sim \posteriordist_S} \rb{
        \kl{ \trainfail \rb{h}}{\expectedfail \rb{h}} \le \frac{\ln \frac{\diff \posteriordist_S}{\diff \priordist}(h) + \ln \frac{2\sqrt{N}}{\delta}}{N} } \geq 1-\delta \,,
    \end{align*}
    and so, by the definition of $\Dinf{\posteriordist_S}{\priordist}$,
    \begin{align*}
       \bbP_{S \sim \datadist^N, h \sim \posteriordist_S} \rb{
        \kl{ \trainfail \rb{h}}{\expectedfail \rb{h}} \le \frac{\Dinf{\posteriordist_S}{\priordist} + \ln \frac{2\sqrt{N}}{\delta}}{N} } \geq 1-\delta \,.
    \end{align*}
\end{theorem}

\begin{proof}
    Following and modifying the proof of Theorem 5 of \citet{maurer2004note}, we start with the inequality $\bbE_S \bb{e^{N \kl{\trainfail \rb{h}}{\expectedfail \rb{h}}}} \leq 2\sqrt{N}$ \cite[Theorem 1]{maurer2004note}, which holds for any $h$, and so also in expectation over $h$ w.r.t.~$\priordist$:
    \begin{align}
        2\sqrt{N} &\geq \bbE_{h\sim\priordist} \bbE_S \bb{\exp\rb{N\;\kl{\trainfail \rb{h}}{\expectedfail \rb{h}}}} 
        = \bbE_S\bbE_{h\sim\priordist} \bb{ \exp\rb{N\;\kl{\trainfail \rb{h}}{\expectedfail \rb{h}}}} \notag \\
        \intertext{with a change of measure from $\priordist$ to $\posteriordist_S$,}
        &= \bbE_S\bbE_{h\sim\posteriordist_S} \bb{ \exp\rb{N\;\kl{\trainfail \rb{h}}{\expectedfail \rb{h}}} \frac{\diff \priordist}{\diff \posteriordist_S}(h)}\\
        &= \bbE_{S,h\sim\posteriordist_S} \bb{ \exp\rb{N\;\kl{\trainfail \rb{h}}{\expectedfail \rb{h}} - \ln \frac{\diff \posteriordist_S}{\diff \priordist}(h)}}
    \end{align}
    Now applying Markov's inequality, we get:
    \begin{align}
        \bbP_{S,h\sim\posteriordist_S} \rb{ \exp\rb{N\;\kl{\trainfail \rb{h}}{\expectedfail \rb{h}}-\ln \frac{\diff \posteriordist_S}{\diff \priordist}(h)} \le \frac{2\sqrt{N}}{\delta} } \geq 1-\delta\,.
    \end{align}
    Rearranging terms, we get the desired bound.
\end{proof}

\subsection{Arriving at \texorpdfstring{\cref{thm: main result sup and mean}}{}}

\begin{theorem} \label{app-thm: main result sup and mean on kl}
Consider any distribution $\datadist$ over $\calZ$, function $f:\calH \times \calZ \to [0,1]$, and sample size $N\ge 8$, any distribution $\nu$ over $\calH$, and any discrete or continuous time process $\{ h_t \in \calH \}_{t\geq 0}$ (i.e.~$t \in \mathbb{Z}_+$ or $t \in \mathbb{R}_+$) that is time-invariant Markov conditioned on $S$.
Denote $p_0(\cdot;S)$ the initial distribution of the Markov process (that may depend on $S$).
Let $p_\infty(\cdot;S)$ be any stationary distribution of the process conditioned on $S$, and $\Psi_S(h)\geq 0$ a non-negative potential function that can depend arbitrarily on $S$, such that $p_\infty(\cdot;S)$ is Gibbs w.r.t.~$\nu$ with potential $\Psi_S$.
Then:
\begin{enumerate}
    \item With probability $1-\delta$ over $S\sim \datadist^N$, 
\begin{align} \label{app-eq: main thm bound on kl}
    \kl{\bbE\left[\trainfail (h_t)\middle|S\right]}{\bbE\left[\expectedfail (h_t)\middle|S\right]}  &\leq 
    \frac{\KL{p_0(\cdot;S)}{\nu}+\bbE\left[\Psi_S(h_0)\middle|S\right] + \ln \sfrac{2\sqrt{N}}{\delta}}{N} 
\end{align}
and so
\begin{align} 
    \bbE\left[\expectedfail(h_t) - \trainfail(h_t) \middle| S \right] &\leq 
    \sqrt{ 2 \bbE \bb{\trainfail \rb{h_t} \mid S}
    \frac{\KL{p_0(\cdot;S)}{\nu}+\bbE\left[\Psi_S(h_0)\middle|S\right] + \ln \sfrac{2\sqrt{N}}{\delta}}{N} } \notag \\
    & \quad + 2 \frac{\KL{p_0(\cdot;S)}{\nu}+\bbE\left[\Psi_S(h_0)\middle|S\right] + \ln \sfrac{2\sqrt{N}}{\delta}}{N} \label{app-eq: main thm bound on kl implication to difference}
\end{align}
\item With probability $1-\delta$ over $S\sim \datadist^N$ and over $h_t$:
\begin{align} \label{app-eq: main thm bound on kl single}
    \kl{\trainfail (h_t)}{\expectedfail (h_t)}  &\leq 
    \frac{\Dinf{p_0(\cdot;S)}{\nu}+\esssup_{p_0} \Psi_S(h_0) + \ln \sfrac{2\sqrt{N}}{\delta}}{N} 
\end{align}
and so, when $\trainfail (h_t) < \expectedfail(h_t)$
\begin{align} 
    \expectedfail(h_t) - \trainfail(h_t) &\leq 
    \sqrt{ 2 \trainfail \rb{h_t}
    \frac{\Dinf{p_0(\cdot;S)}{\nu}+\esssup_{p_0} \Psi_S(h_0) + \ln \sfrac{2\sqrt{N}}{\delta}}{N} } \notag \\
    & \quad + 2 \frac{\Dinf{p_0(\cdot;S)}{\nu}+\esssup_{p_0} \Psi_S(h_0) + \ln \sfrac{2\sqrt{N}}{\delta}}{N} \label{app-eq: main thm bound on kl implication to difference single}
\end{align}
\end{enumerate}
\end{theorem}

\begin{lemma}
    Let $a, b \in \bb{0, 1}$. 
    Then
    \begin{align}\label{app-eq: kl simplifying inequality}
        b \le a + \sqrt{2 a \kl{a}{b}} + 2 \kl{a}{b} \,.
    \end{align}
\end{lemma}
\begin{proof}
    The KL divergence is non-negative, so it suffices to consider the case that $b \ge a$.
    Defining $\varphi : \bb{0, 1-a} \to \bbR$ as
    \begin{align*}
        \varphi \rb{u} = \frac{u^2}{2 \rb{a + u}} \,,
    \end{align*}
    it can be readily checked by differentiation that for all $u \in \bb{0, 1 - a}$,
    \begin{align*}
        \kl{a}{a + u} \ge \varphi \rb{u} \,.
    \end{align*}
    In particular, for $u = b - a \in \bb{0, 1 - a}$, 
    \begin{align} \label{app-eq: kl quadratic lower bound}
        \kl{a}{b} \ge \frac{\rb{b - a}^2}{2b} \,.
    \end{align}
    Next, we consider the following inequality
    \begin{align} \label{app-eq: quadratic inequality}
        2u^2 + \sqrt{2a} u + a - b \ge 0 \,, \; u \ge 0 \,.
    \end{align}
    Solving for $u$, it turns out that the inequality holds when 
    \begin{align} \label{app-eq: quadtratic inequality positive solution}
        u \ge \frac{\sqrt{8b - 6a} - \sqrt{2a}}{4} \,.
    \end{align}
    In addition, under the assumption that $b \ge a$, 
    \begin{align} \label{app-eq: quadratic inequality solution edge upper bound}
        \frac{\sqrt{8b - 6a} - \sqrt{2a}}{4} \le \sqrt{\frac{\rb{b - a}^2}{2b}} \,.
    \end{align}
    Combining \eqref{app-eq: kl quadratic lower bound}, \eqref{app-eq: quadtratic inequality positive solution}, and \eqref{app-eq: quadratic inequality solution edge upper bound}, $u = \sqrt{\kl{a}{b}}$ solves \eqref{app-eq: quadratic inequality} implying \eqref{app-eq: kl simplifying inequality}.
\end{proof}

\begin{proof}
    The inequalities \eqref{app-eq: main thm bound on kl} and \eqref{app-eq: main thm bound on kl single} follow by plugging \cref{cor: second law + chain rule for d inf and kl} into \cref{app-thm: gen bound kl tolstikhin,app-thm: gen bound kl single}.
    For inequalities \eqref{app-eq: main thm bound on kl implication to difference} and \eqref{app-eq: main thm bound on kl implication to difference single}, we use \eqref{app-eq: kl simplifying inequality}.
    For \eqref{app-eq: main thm bound on kl implication to difference}, we use $a = \expectedfail(h_t)$ and $b = \trainfail(h_t)$, which yields: 
        \begin{align*}
        \expectedfail(h_t) &\le \trainfail(h_t) + \sqrt{2 \trainfail(h_t) \kl{\trainfail(h_t)}{\expectedfail(h_t)}} + 2\kl{\trainfail(h_t)}{\expectedfail(h_t)} \\
        &\le \trainfail(h_t) + \sqrt{2 \trainfail(h_t) \frac{\KL{p_0(\cdot;S)}{\nu}+\bbE_{p_0} \Psi_S(h_0) + \ln \sfrac{2\sqrt{N}}{\delta}}{N}} \\
        &\quad + 2 \frac{\KL{p_0(\cdot;S)}{\nu}+\bbE_{p_0} \Psi_S(h_0) + \ln \sfrac{2\sqrt{N}}{\delta}}{N},
    \end{align*}
    and similarly for \eqref{app-eq: main thm bound on kl implication to difference single}.
\end{proof}

\begin{remark}
    Notice that when $h_t$ has a small training error $\bbE \bb{ \trainfail \rb{h_t} \mid S} \approx 0$, the effective generalization gap decays as $O \rb{1/N}$ instead of as $O \rb{1 / \sqrt{N}}$.
\end{remark}

\begin{remark}
    In order to get the version in \Cref{thm: main result sup and mean} we use the upper bound of Pinsker's inequality, \ie that for all $a, b \in \rb{0, 1}$
    \begin{align*}
        \abs{a - b} \le \sqrt{\frac{1}{2} \kl{a}{b}} \,,
    \end{align*}
    and simplify $\ln \frac{2\sqrt{N}}{\delta} \le \ln \frac{N}{\delta}$ as $N \ge 8$.
\end{remark}

Finally, we prove the equivalence statement made in \cref{foot:converse}:

\begin{claim} \label{app-claim: footnote converse of symmeterized KL}
    $\KL{p}{q}+\KL{q}{p}\leq \beta$ iff there exists a potential $\Psi$ such that $p$ is Gibbs w.r.t. $q$ with potential $\Psi$ and $\bbE_{q} \generalpotential - \bbE_{p} \generalpotential\leq \beta$, and similarly $\Dinf{p}{q}+\Dinf{q}{p}\leq \beta$ iff there exists a potential $0\leq \Psi\leq\beta$ such that $p$ is Gibbs w.r.t. $q$ with potential $\Psi$.
\end{claim}
\begin{proof}
    The first direction follows directly from \cref{claim: sum of dists is difference of expectations or extremas}, so we only need to prove the converse.
    Assume that either $\KL{p}{q}+\KL{q}{p}\leq \beta$, or $\Dinf{p}{q}+\Dinf{q}{p}\leq \beta$.
    In these cases, both $\diff p / \diff q$ and $\diff q / \diff p$ exist, and for any measurable event $B$, $p \rb{B} = 0 \iff q \rb{B} = 0$, or equivalently, $p \rb{B} > 0 \iff q \rb{B} > 0$. 
    Therefore, $\mathrm{supp} \rb{p} = \mathrm{supp} \rb{q}$, and $\diff p / \diff q > 0$ on $\mathrm{supp} \rb{p}$. 
    Denote $\generalpotential = - \ln \diff p / \diff q$, then $p$ is Gibbs w.r.t. $q$ with potential $\generalpotential$.
    The same derivation as in the proof of \cref{claim: sum of dists is difference of expectations or extremas} results in the bounds $\bbE_{q} \generalpotential - \bbE_{p} \generalpotential\leq \beta$ and $\esssup_q \generalpotential - \essinf_p \generalpotential \leq\beta$.
    In particular, if the latter holds then $\generalpotential$ can be shifted such that essentially $0 \le \generalpotential \le \beta$.
\end{proof}

\newpage

\section{Tightness and Necessity of the Divergence Conditions} \label{app-sec: tightness and necessity}

If we are only interested in ensuring generalization at time $t\rightarrow\infty$, and when we converge to the stationary distribution $p_\infty$, then it is enough to bound the divergence $\D{p_\infty}{\nu}$. 
If we are interested in bounding $\D{p_t}{\nu}$ (and consequently, the generalization gap) at all times $t$, then we need also to limit $p_0$'s dependence on $S$, since $p_0$ (as well as $p_t$ for small $t$) can be completely different from a stationary $p_\infty$, and just bounding $\D{p_\infty}{\nu}$ does not say anything about it. Bounding $\D{p_0}{\mu}$, for some data-independent distribution $\mu$, ensures generalization at $p_0$.
This leaves the following questions regarding the proof of \cref{thm: main result sup and mean}:
\begin{itemize}[ leftmargin=15pt]
\item Why do we need to bound the divergences $\D{p_\infty}{\nu}$ and $\D{p_0}{\nu}$ from the same distribution $\nu$? That is, we do we need to require $\mu=\nu$?  Bounding the divergences of $p_0$ and $p_\infty$ to two different divergences $\mu\neq\nu$ is sufficient to get generalization at the beginning (i.e.~initialization) and end (i.e.~after mixing)--is it sufficient for generalization in the middle (i.e.~at any $t$)?
\item Why do we need to also bound the reverse divergence $\D{\nu}{p_\infty}$?  I.e., why do we need to require $p_\infty$ is Gibbs w.r.t.~$\nu$ with a bounded potential, instead of just controlling the divergence $\D{p_\infty}{\nu}$, which is a weaker requirement and sufficient for generalization after mixing?
\end{itemize}

As we now show, both are necesairy, and without requiring both, i.e.~if we drop either one of these, we cannot ensure generalization at intermediate times $t \ge 0$.

\paragraph{Construction.} 
Consider a supervised learning problem with $\calZ = \calX \times \calY$, $\calX = [0,1]$, $\calY=\{0,1\}$, $\calH =$ all measurable functions from $\calX$ to $\calY$, and the zero-one loss $f(h,(x,y)) = \indnosub{h(x)\neq y}$, with $\datadist$ being the uniform distribution over $\calX$, and $y$ being $\mathrm{Bernoulli} (\tfrac{1}{2})$ independent of $x$. 
For all $h$, $\expectedfail (h) = 0.5$. 
Let $p_0$ be the constant zero function with probability $\tfrac{1}{2}$ and the constant one function with probability $\tfrac{1}{2}$. 
Consider the following deterministic $S$-dependent transition function over $h$: if $h_t$ is the constant zero function, then $h_{t+1}=h_S$ which memorizes $S$, \ie $h_S(x)=y$ for $(x,y)\in S$, and $h_S(x)=1$ otherwise. 
If $h_t$ is not the constant zero function, then $h_{t+1}$ is the constant ones function. 
We have that $p_\infty$ is deterministic at the constant one function, and $\KL{p_\infty}{p_0}=\ln{2}$, and in fact $p_t=p_\infty$ for $t>1$. 
But with probability half, $h_1=h_S$, for which for any sample size $N>0$, $\trainfail(h_S)=0$ while $\expectedfail(h_S)=\tfrac{1}{2}$.

\paragraph{How does this show it is not enough to bound $\D{p_0}{\nu}$ and $\D{p_\infty}{\nu}$, but that we also need the reverse $\D{\nu}{p_\infty}$?} 
Since $p_0$ is data independent, we can take $\nu=p_0$, in which case $\KL{p_0}{\nu}=\Dinf{p_0}{\nu}=0$ and $\KL{p_\infty}{\nu}=\Dinf{p_\infty}{\nu}=\ln{2}$, but even as $N\to\infty$, the gap for $h_1$ does not diminish. Indeed, $\D{\nu}{p_\infty}=\infty$, and so $p_\infty$ is not Gibbs w.r.t. $\nu$ and Theorem \ref{thm: main result sup and mean} does not apply.

\paragraph{How does this show it is not enough to bound $\D{p_\infty}{\nu}+\D{\nu}{p_\infty}$ and $\D{p_0}{\mu}$ for $\mu \neq \nu$?} 
Since in this example $p_\infty$ is also data independent, we can take $\nu=p_\infty$ and $\mu=p_0$, in which case $\D{p_0}{\mu}=0$ and $\D{p_\infty}{\nu}+\D{\nu}{p_\infty}=0$. We are indeed ensured a small gap for $h_0$ and $h_\infty$, but not for $h_1$.

\newpage

\section{Generalized Version of \texorpdfstring{\cref{cor: generalization for ngf box and regularization}}{}}\label{apps-sec:CLD extentions}

We start by characterizing the stationary distributions of SDERs in a box with different noise scales $\sigma^2$.
The stationary distributions for Gaussian initialization can be found similarly.
Then, we extend \cref{cor: generalization for ngf box and regularization} to scenarios where $p_0 \neq \nu$, as an immediate consequence of \cref{thm: main result sup and mean}.

\subsection{Stationary distributions of CLD} \label{app-sec: stationary distirbutions of cld}

We first derive the stationary distribution of SDERs of the form
\begin{align} \label{app-eq: general reflected langevin auxiliary}
    d\bfx_t = - \nabla L \rb{\bfx_t} dt + \sqrt{2 \beta^{-1} \sigma^2 \rb{\bfx_t}} d\bfw_t + d\reflectionprocess \,,
\end{align}
with normal reflection in a box domain (for a full definition see \eqref{app-eq: box dom 1}-\eqref{app-eq: box reflecting field} in \cref{app-sec: SDER in a box}), where $L \ge 0$ is some $\calC^2$ loss function, $\beta > 0$ is an inverse temperature parameter, and $\sigma^2$ is a diffusion coefficient.
First, we present a well known characterization of the stationary distribution of \eqref{app-eq: general reflected langevin auxiliary}.

\begin{lemma} \label{app-lem: langevin in a box stationary}
    If $L, \sigma^2 \in \calC^2$, $\sigma^2 \rb{\cdot} > 0$ is uniformly bounded away from $0$ in $\DOMCLOS$, 
    \begin{align*}
        Z = \intop_{\DOMCLOS} \frac{1}{\sigma^{2}\rb{\bfx}} \exp\rb{-\beta\intop\frac{\nabla L\rb{\bfx}}{\sigma^{2}\rb{\bfx}}d\bfx} < \infty \,,
    \end{align*}
    the integrals exist, and the field $\nabla L / \sigma^2$ is conservative (curl-free),
    then 
    \begin{align} \label{app-eq: stationary langevin auxiliary}
        \stationarydist \rb{\bfx} = \frac{1}{Z} \frac{1}{\sigma^{2}\rb{\bfx}} \exp\rb{-\beta\intop\frac{\nabla L\rb{\bfx}}{\sigma^{2}\rb{\bfx}}d\bfx} \,
    \end{align}
     is a stationary distribution of \eqref{app-eq: general reflected langevin auxiliary}.
\end{lemma}
For completeness, the proof is presented in \cref{app-sec: langevin dynamics}, following additional results and definitions in \cref{app-sec: background}.
We can now calculate explicit stationary distributions for some choices of $\sigma^2$.
Specifically, we focus on cases where $\sigma^2 \rb{\bfx} = g \rb{L \rb{\bfx}}$ for some scalar function $g$, as it guarantees the curl-free condition, and is convenient to integrate.

\begin{example}[Uniform noise scale] \label{app-ex: uniform noise scale stationary}
    Assuming that $\sigma^2 \rb{\bfx} \equiv 1$, the stationary distribution becomes the well-known Gibbs distribution
    \begin{align}
        \stationarydist \rb{\bfx} = \frac{1}{Z} e^{-\beta L \rb{\bfx}} \,,
    \end{align}
    so
    \begin{align}
        \potential{uniform} \rb{\bfx} = \beta L \rb{\bfx} \,.
    \end{align}
\end{example}

\begin{example}[Linear noise scale] \label{app-ex: linear noise scale stationary}
    Let $\alpha > 0$, and suppose that $\sigma^2 \rb{\bfx} = \rb{L \rb{\bfx} + \alpha}$. 
    Then
    \begin{align*}
        \frac{\nabla L \rb{\bfx}}{\sigma^2 \rb{\bfx}} = \nabla \ln \rb{L \rb{\bfx} + \alpha}
    \end{align*}
    so the stationary distribution is
    \begin{align} \label{app-eq: linear noise scale stationary}
        \stationarydist \rb{\bfx} \propto \frac{1}{L \rb{\bfx} + \alpha} \exp \rb{- \beta \ln \rb{L \rb{\bfx} + \alpha}} = \frac{1}{L \rb{\bfx} + \alpha} \rb{L \rb{\bfx} + \alpha}^{-\beta} = \rb{L \rb{\bfx} + \alpha}^{-\beta - 1} \,,
    \end{align}
    which is integrable in a bounded domain.
    Recall that we want to represent $\stationarydist$ using a potential $\generalpotential$ with $\inf \generalpotential \ge 0$.
    In this case, we can start from $\Tilde{\generalpotential} \rb{\bfx} = \rb{\beta + 1} \ln \rb{L \rb{\bfx} + \alpha}$.
    Since $L \ge 0$ it clearly holds that $\tilde{\generalpotential} \ge \rb{\beta + 1} \ln \rb{\alpha}$, so we can use the shifted version
    \begin{align} \label{app-eq: linear noise scale sder potential}
    \potential{linear} \rb{\bfx} = \rb{\beta + 1} \rb{\ln \rb{L \rb{\bfx} + \alpha} - \ln \rb{\alpha}} = \rb{\beta + 1} \ln \rb{\frac{L \rb{\bfx}}{\alpha} + 1} \,.
    \end{align}
\end{example}

\begin{example}[Polynomial noise scale] \label{app-ex: polynomial noise scale stationary}
    Let $\alpha > 0$, and $k > 1$.
    Suppose that $\sigma^2 \rb{\bfx} = \rb{L \rb{\bfx} + \alpha}^k$. 
    Then
    \begin{align*}
        \frac{\nabla L \rb{\bfx}}{\sigma^2 \rb{\bfx}} = \nabla L \rb{\bfx} \rb{L \rb{\bfx} + \alpha}^{-k} = \frac{1}{1 - k} \nabla \rb{L \rb{\bfx} + \alpha}^{1 - k}
    \end{align*}
    so 
    \begin{align*}
        \stationarydist \rb{\bfx} \propto \rb{L \rb{\bfx} + \alpha}^{-k} \exp \rb{\frac{\beta}{k - 1} \rb{L \rb{\bfx} + \alpha}^{1 - k}} \,.
    \end{align*}
    As before, the potential is monotonically increasing with $L \rb{\bfx}$, so we can make a shift
    \begin{align*}
        \potential{poly} = k \ln \rb{\frac{L\rb{\bfx}}{\alpha}  + 1} + {\frac{\beta}{k - 1} \rb{\alpha^{1-k}-\rb{L \rb{\bfx} + \alpha}^{1 - k}} }  \,.
    \end{align*}
\end{example}

\begin{example}[Exponential noise scale] \label{app-ex: exponential noise scale stationary}
    Let $\alpha > 0$ and suppose that $\sigma^2 \rb{\bfx} = e^{\alpha L \rb{\bfx}}$.
    Then
    \begin{align*}
        \frac{\nabla L \rb{\bfx}}{\sigma^2 \rb{\bfx}} = - \frac{1}{\alpha} \nabla \rb{e^{-\alpha L \rb{\bfx}}} 
    \end{align*}
    so 
    \begin{align*}
        \stationarydist \rb{\bfx} \propto e^{-\alpha L \rb{\bfx}} \exp \rb{\frac{\beta}{\alpha} e^{-\alpha L \rb{\bfx}}} = \exp \rb{\frac{\beta}{\alpha} e^{-\alpha L \rb{\bfx}} - \alpha L \rb{\bfx}} \,.
    \end{align*}
    Denote $\psi \rb{\tau} = \alpha \tau - \frac{\beta}{\alpha} e^{-\alpha \tau}$, then $\psi^\prime \rb{\tau} = \alpha + \beta e^{-\alpha \tau} \ge 0$.
    Therefore, $\min_{\tau \ge 0} \psi \rb{\tau} = \psi \rb{0} = -\frac{\beta}{\alpha}$, and we can take
    \begin{align} \label{app-eq: exponential noise scale sder potential}
        \potential{exp} \rb{\bfx} = \alpha L \rb{\bfx} - \frac{\beta}{\alpha} e^{-\alpha L \rb{\bfx}} + \frac{\beta}{\alpha} = \alpha L \rb{\bfx} + \frac{\beta}{\alpha} \rb{1 - e^{-\alpha L \rb{\bfx}}}
    \end{align}
\end{example}

\subsection{Generalization bounds} \label{app-sec: cld generalization subsec}

\paragraph{Bounded domain with uniform initialization.}
Assume that training follows a CLD in a bounded domain as described in \eqref{app-eq: general reflected langevin auxiliary} with uniform initialization $p_0 = \mathrm{Uniform} \rb{\PARAMDOM_0}$, where $\PARAMDOM_0 \subseteq \PARAMDOM$.
For simplicity we take $\sigma^2 \equiv 1$.
In that case \cref{thm: main result sup and mean} implies the following.

\begin{lemma} \label{app-lem: generalization for ngf box}
    Assume that the parameters evolve according to \eqref{app-eq: general reflected langevin auxiliary} with $\sigma^2 \equiv 1$ and uniform initialization $p_0 = \mathrm{Uniform} \rb{\PARAMDOM_0}$, where $\PARAMDOM_0 \subseteq \PARAMDOM$.
    Then for any time $t\ge 0$, and $\delta \in \rb{0, 1}$,
    \begin{enumerate}[ leftmargin=15pt]
        \item w.p. $1-\delta$ over $S\sim \datadist^N$, 
        \begin{align} \label{app-eq: sder bound in expectation}
        \bbE_{\param_t \sim p_t} \bb{\expectedfail(\param_t) - \trainfail(\param_t) \mid S} \le \sqrt{\frac{\beta \bbE_{p_0} \bb{ \trainloss (\param) \mid S} + \ln \abs{\PARAMDOM} / \abs{\PARAMDOM_0} + \ln \rb{N/\delta}}{2N}} \,.
        \end{align}
        \item w.p. $1-\delta$ over $S\sim \datadist^N$ and $\param_t \sim p_t$
        \begin{align} \label{app-eq: sder bound in probability}
        \expectedfail(\param_t) - \trainfail(\param_t) \le \sqrt{\frac{\beta \esssup_{p_0} \trainloss (\param) + \ln \abs{\PARAMDOM} / \abs{\PARAMDOM_0} + \ln \rb{N/\delta}}{2N}} \,.
        \end{align}
    \end{enumerate}
\end{lemma}
\begin{proof}
    This is a direct corollary of \cref{thm: main result sup and mean} with $\KL{p_0}{\nu} = \ln \abs{\PARAMDOM} / \abs{\PARAMDOM_0}$.
\end{proof}

\paragraph{$\ell^2$ regularization with Gaussian initialization.}
Let $\blambda \in \bbR^\PARAMDIM_{>0}$ be regularization terms, and consider the unconstrained SDE
\begin{align} \label{app-eq: sde with regularization}
    d \param_t = -\nabla L \rb{\param_t} dt - \beta^{-1} \diag \rb{\blambda} \param_t dt + \sqrt{2 \beta^{-1} \sigma^2 \rb{\param_t}} d \bfw_t \,.
\end{align}
Notice that $- \beta^{-1} \diag \rb{\blambda} \param_t dt$ corresponds to an additive regularization of the form $\frac{1}{2\beta} \param_t^\top \diag\rb{\blambda} \param_t$, so each parameter can have a different regularization coefficient.
We shall denote by $\phi_{\blambda}$ a multivariate Gaussian distribution with mean $\bfzero$ and covariance matrix $\diag \rb{\blambda^{-1}}$, where $\blambda^{-1} = \rb{ \lambda_1^{-1}, \dots, \lambda_{\PARAMDIM}^{-1}}$.
For simplicity, we present the results with $\sigma^2 \equiv 1$.

\begin{lemma} \label{app-lem: generalization for cld with regularization}
    Let $\blambda_0, \blambda_1 > 0$, and let $\param_t$ evolve according to \eqref{app-eq: sde with regularization} with $\sigma^2 \equiv 1$ and $\blambda = \blambda_1$, and start from a Gaussian initialization $p_0 = \phi_{\blambda_0}$.
    Then for any time $t\ge 0$, and $\delta \in \rb{0, 1}$,
    \begin{enumerate}[ leftmargin=15pt]
        \item w.p. $1-\delta$ over $S\sim \datadist^N$, 
        \begin{align} \label{app-eq: cld with regularization bound in expectation}
        \bbE_{\param_t \sim p_t} \bb{\expectedfail(\param_t) - \trainfail(\param_t) \mid S} \le \sqrt{\frac{\beta \bbE_{p_0} \bb{ \trainloss (\param) \mid S} + \KL{\phi_{\blambda_0}}{\phi_{\blambda_1}} + \ln \rb{N/\delta}}{2N}} \,.
        \end{align}
        \item w.p. $1-\delta$ over $S\sim \datadist^N$ and $\param_t \sim p_t$
        \begin{align} \label{app-eq: cld with regularization bound in probability}
        \expectedfail(\param_t) - \trainfail(\param_t) \le \sqrt{\frac{\beta \esssup_{p_0} \trainloss (\param) + \KL{\phi_{\blambda_0}}{\phi_{\blambda_1}} + \ln \rb{N/\delta}}{2N}} \,,
        \end{align}
    \end{enumerate}
    where $\KL{\phi_{\blambda_0}}{\phi_{\blambda_1}} = \frac{1}{2} \sum_{i=1}^\PARAMDIM \rb{\ln \rb{\frac{\lambda_{1,i}}{\lambda_{0,i}}} - 1 + \frac{\lambda_{0,i}}{\lambda_{1, i}}}$.\footnote{For $\blambda_0 = \lambda_0 \bfI, \blambda_1 = \lambda_1 \bfI$, $\lambda_0, \lambda_1 > 0$, this simplifies to $\KL{\phi_{\lambda_0}}{\phi_{\lambda_1}} = \frac{d}{2} \rb{\ln \frac{\lambda_0}{\lambda_1} - 1 + \frac{\lambda_1}{\lambda_0}}$}
\end{lemma}
\begin{proof}
    This is a direct corollary of \cref{thm: main result sup and mean} with the explicit expression for the KL divergence between two Gaussians.
\end{proof}

\begin{remark}[Dependence on the parameters' dimension]
    While the bound in \cref{app-lem: generalization for cld with regularization} depends on the dimension of the parameters $\PARAMDIM$, this can be mitigated in practice.
    For example, by matching the regularization coefficient and initialization variance, the KL-divergence term vanishes and we lose the dependence on dimension.
    Furthermore, we can control each parameter separately by using parameter specific initialization variances and regularization coefficients.
    Then, the KL-divergence can have different dependencies, if any, on the dimension $\PARAMDIM$.
\end{remark}
\newpage
\section{Linear Regression with CLD} \label{app-sec: quadratic objective}

\cref{thm: main result sup and mean,cor: generalization for ngf box and regularization} only bound the \emph{gap} between the population and training errors, yet this does not necessarily bound the population error itself.
One way to do this is by separately bounding the training error and showing that in the regime in which the generalization gap is small, the training error can be small as well.
In \cref{app-sec: experiments} we show empirically that deep NNs can reach low training error when trained with SGLD in the regime in which \cref{cor: generalization for ngf box and regularization} is not vacuous.
Here, we look at the particular case of the asymptotic behavior of ridge regression with CLD training with Gaussian i.i.d. data, for which we can analytically study the training and population \emph{losses}.

\paragraph{Setup.}
Let $\paramteach \in \bbR^\PARAMDIM$, $y = \bfx^\top \paramteach + \varepsilon$ with $\norm{\paramteach}=1$ and $\varepsilon \sim \calN \rb{0, \sigma^2}$ independent of $\bfx$. We assume that $\bfx$ has i.i.d. entries with $\bbE \bfx = \bfzero$ and covariance $\bbE \bb{\bfx \bfx^\top} = \bfI$.
Let $\bfX \in \bbR^{N \times \PARAMDIM}$ be the data (design) matrix, $\bfy \in \bbR^N$ the training targets, $\bvareps \in \bbR^N$ the pointwise perturbations, and $\param \in \bbR^\PARAMDIM$ the parameters in a linear regression problem.
In what follows, we focus on the overdetermined case $N > \PARAMDIM$, where $\bfX$ has full column rank with probability 1, so the empirical covariance $\bfA = \frac{1}{N} \bfX^\top \bfX \succ 0$ a.s.
In addition, we denote $\paramls = \frac{1}{N} \bfA^{-1} \bfX^\top \bfy$, and $\Tilde{\param} = \param - \paramls$.
The training objective is then the minimization of the regularized empirical loss
\begin{align*}
    \trainloss \rb{\param} + \frac{\lambda}{2\beta} \norm{\param}^2 = \frac{1}{2N} \norm{\bfX \param - \bfy}^2 + \frac{\lambda}{2\beta} \norm{\param}^2  = \frac{1}{2} \Tilde{\param}^\top \bfA \Tilde{\param} + C_S + \frac{\lambda}{2\beta} \norm{\param}^2 \,,
\end{align*}
where $C_S = \trainloss \rb{\paramls} = \frac{1}{2N} \norm{\bfy}^2 - \frac{1}{2} \paramls \bfA \paramls = \frac{1}{2N} \norm{\bfy}^2 - \frac{1}{2N} \bfy^\top \bfX \rb{\bfX^\top \bfX}^{-1} \bfX^\top \bfy$, is the empirical irreducible error.

\paragraph{CLD training.}
Assume that training is performed by CLD with inverse temperature $\beta > 0$, which, because $\trainloss$ is quadratic, takes the form
\begin{align} \label{app-eq: linear regression OU SDE}
    \diff \param_t = - \bfA \rb{\param_t - \paramls} \diff t - \lambda \beta^{-1} \param_t \diff t + \sqrt{\frac{2}{\beta}} \diff \bfw_t \,.
\end{align}
Since $\bfA \succ 0$ and $\lambda > 0$, the Gibbs distribution
\begin{align*}
    \posteriorstationary \rb{\param} & \propto \exp \rb{- \frac{1}{2} \rb{\rb{\param - \paramls}^\top \beta \bfA \rb{\param - \paramls} + \lambda \param^\top \param }}
\end{align*}
is the unique stationary distribution, and furthermore, it is the asymptotic distribution of \eqref{app-eq: linear regression OU SDE}.
We can simplify this to a Gaussian.
Denote $\alpha = \lambda / \beta$ and
\begin{align*}
    \bSigma = \frac{1}{\beta} \rb{\bfA + \alpha \bfI}^{-1} \; \mathrm{and} \; \Bar{\param} = \beta \bSigma \bfA \paramls = \frac{1}{N} \rb{\bfA + \alpha \bfI}^{-1} \bfX^\top \bfy \,, 
\end{align*}
then 
\begin{align*}
    \rb{\param - \Bar{\param}}^\top \bSigma^{-1} \rb{\param - \Bar{\param}} & = \beta \param^\top \rb{\bfA + \alpha \bfI} \param - 2 \param^\top \bSigma^{-1} \Bar{\param} + \Bar{\param}^\top \bSigma^{-1} \Bar{\param} \\
    & = \beta \param^\top \rb{ \bfA + \alpha \bfI} \param - 2 \beta \param^\top \bSigma^{-1} \bSigma \bfA \paramls + \beta^2 \paramls^\top \bfA \bSigma \bSigma^{-1} \bSigma \bfA \paramls \\
    & = \beta \param^\top \rb{\bfA + \alpha \bfI} \param - 2 \beta \param^\top \bfA \paramls + \beta^2 \paramls^\top \bfA \bSigma \bfA \paramls \,.
\end{align*}
Since the last term is constant w.r.t. $\param$, we deduce that
\begin{align*}
    \posteriorstationary \rb{\param} \propto \exp \rb{-\frac{1}{2} \rb{\param - \Bar{\param}}^\top \bSigma^{-1} \rb{\param - \Bar{\param}}} \,,
\end{align*}
\ie the stationary distribution is a Gaussian $\calN \rb{\Bar{\param}, \bSigma}$.
We can now calculate the expected training and population losses.

\paragraph{Goal.}
In the rest of this section, our final aim is to calculate the expected training and population losses in the setup described above, in the case when the data is sampled i.i.d. from standard Gaussian distribution, $\sigma$ is a fixed constant, $\lambda \propto d$ (to match standard initialization), \footnote{Since this is a linear model $d=\mathrm{layer\,width}$, and as we assume the regularization matches the standard initialization.
This initialization is considered in many works as a Bayesian prior in various settings \citep{lee2018deep,wenger2025variational}.} $N,\beta$ and $d$ are large, but $\beta \ll N$, so our generalization bound is small (since $\mathbb{E}_{p_0}L$ is a fixed constant in this case). 
We will find (in \cref{app-rem: taylor simplification of training loss} and \cref{app-rem: taylor simplification population loss}) that if also $d \ll \beta$ then the training and expected population loss are not significantly degraded. This is not a major constraint, since we need $d\ll N$ to get good population loss anyway, even without noise (i.e. $\beta\rightarrow\infty$). This shows that in this regime $d \ll \beta \ll N$, the randomness required by our generalization bound (the KL bounds in \cref{cor: generalization for ngf box and regularization}) does not significantly harm the training loss or the expected population loss.

\begin{claim} \label{app-claim: linear regression expected train loss conditioned on X}
    With some abuse of notation, denote $\trainloss \rb{\param_\infty} = \bbE_{\param \sim \posteriorstationary} \trainloss \rb{\param}$.
    Then
    \begin{align*}
    \bbE \bb{ \trainloss \rb{\param_\infty} \mid \bfX } & = \frac{1}{2\beta} \tr \rb{\bfA \rb{\bfA + \alpha \bfI}^{-1}} + \frac{\alpha^2}{2} \param^{\star \top} \rb{\bfA + \alpha \bfI}^{-2} \bfA \paramteach  \\
    & \quad + \frac{\sigma^2 \alpha^2}{2N} \mathrm{Tr} \rb{ \rb{\bfA + \alpha \bfI}^{-2} } + \frac{\sigma^2}{2} \rb{1 - \frac{\PARAMDIM}{N}} \,.
\end{align*}
\end{claim}
\begin{proof}
From \citet{petersen2012cookbook} (equation 318)
\begin{align*}
    \trainloss \rb{\param_\infty} &= \frac{1}{2} \bbE \rb{\param - \paramls}^\top \bfA \rb{\param - \paramls} + C_S \\
    & = \frac{1}{2} \mathrm{Tr} \rb{\bfA \bSigma} + \frac{1}{2} \rb{\Bar{\param} - \paramls}^\top \bfA \rb{\Bar{\param} - \paramls} + C_S \,.
\end{align*}
For the second term, notice that 
\begin{align*}
    \Bar{\param} - \paramls &= \rb{\beta \bSigma \bfA - \bfI} \paramls \\
    & = \rb{\beta \bSigma \bfA + \lambda \bSigma - \lambda \bSigma - \bfI} \paramls \\
    & = \rb{\bSigma \underset{=\bSigma^{-1}}{\underbrace{\beta \rb{\bfA + \alpha \bfI}}} - \lambda \bSigma - \bfI} \paramls \\
    & = -\lambda \bSigma \paramls = - \alpha \rb{\bfA + \alpha \bfI}^{-1} \paramls \,. \\
\end{align*}
$\bfA$ and $\bSigma$ are simultaneously diagonalizable. 
To see this, let $\bfA = \bfQ \bLambda \bfQ^\top$ be a spectral decomposition of $\bfA$, then $\bfA + \alpha \bfI = \bfQ \rb{\bLambda + \alpha \bfI} \bfQ^\top$, so $\bSigma = \beta^{-1} \bfQ \rb{\bLambda + \alpha \bfI}^{-1} \bfQ^\top$.
This means that $\bfA$, $\bSigma$, and their inverses all multiplicatively commute.
Therefore,
\begin{align*}
    \trainloss \rb{\param_\infty} &= \frac{1}{2} \mathrm{Tr} \rb{\bfA \bSigma} + \frac{\alpha^2}{2} \paramls^\top \rb{\bfA + \alpha \bfI}^{-1} \bfA \rb{\bfA + \alpha \bfI}^{-1} \paramls + C_S \\
    & = \frac{1}{2} \mathrm{Tr} \rb{\bfA \bSigma} + \frac{\alpha^2}{2N^2} \bfy^\top \bfX \bfA^{-1} \rb{\bfA + \alpha \bfI}^{-1} \bfA \rb{\bfA + \alpha \bfI}^{-1} \bfA^{-1} \bfX^\top \bfy + C_S \\
    & = \frac{1}{2\beta} \mathrm{Tr} \rb{\bfA \rb{\bfA + \alpha \bfI}^{-1}} + \frac{\alpha^2}{2N^2} \bfy^\top \bfX \rb{\bfA + \alpha \bfI}^{-2} \bfA^{-1} \bfX^\top \bfy + C_S \,,
\end{align*}
Conditioned on $\bfX$, standard results about the residuals in linear regression imply that, 
\begin{align*}
    \bbE \bb{C_S \mid \bfX} = \frac{\sigma^2}{2} \rb{1 - \frac{\PARAMDIM}{N}} \,.
\end{align*}
In addition, for any symmetric matrix $\bfM$ we have 
\begin{align*}
    \bbE_\bvareps \bb{ \bfy^\top \bfM \bfy } & = \bbE_{\bvareps} \bb{ \rb{\bfX \paramteach + \bvareps}^\top \bfM \rb{\bfX \paramteach + \bvareps} } \\
    & = \rb{\bfX \paramteach}^\top \bfM \bfX \paramteach + \bbE_{\bvareps} \bb{ \bvareps^\top \bfM \bvareps } \\
    & = \rb{\bfX \paramteach}^\top \bfM \bfX \paramteach + \sigma^2 \mathrm{Tr} \rb{\bfM} \,.
\end{align*}
In particular, 
\begin{align*}
    \bbE & \bb{\bfy^\top \bfX \rb{\bfA + \alpha \bfI}^{-2} \bfA^{-1} \bfX^\top \bfy \mid \bfX} \\
    & = \param^{\star \top} \bfX^\top \bfX \rb{\bfA + \alpha \bfI}^{-2} \bfA^{-1} \bfX^\top \bfX \paramteach + \sigma^2 \mathrm{Tr} \rb{ \bfX \rb{\bfA + \alpha \bfI}^{-2} \bfA^{-1} \bfX^\top} \\
    & = \param^{\star \top} N \bfA \rb{\bfA + \alpha \bfI}^{-2} \bfA^{-1} N \bfA \paramteach + \sigma^2 \mathrm{Tr} \rb{ \bfX^\top \bfX \rb{\bfA + \alpha \bfI}^{-2} \bfA^{-1} } \\
    & = N^2 \param^{\star \top} \rb{\bfA + \alpha \bfI}^{-2} \bfA \paramteach + N \sigma^2 \mathrm{Tr} \rb{ \rb{\bfA + \alpha \bfI}^{-2} } \,,
\end{align*}
where we used the definition of $\bfA$, the joint diagonalizability of $\bfA$ and $\bSigma$, and the cyclicality of the trace.
In total, the expected training loss, conditioned on the data is
\begin{align*}
    \bbE_{\bvareps} \trainloss \rb{\param_\infty} & = \frac{1}{2\beta} \tr \rb{\bfA \rb{\bfA + \alpha \bfI}^{-1}} + \frac{\alpha^2}{2} \param^{\star \top} \rb{\bfA + \alpha \bfI}^{-2} \bfA \paramteach  \\
    & \quad + \frac{\sigma^2 \alpha^2}{2N} \mathrm{Tr} \rb{ \rb{\bfA + \alpha \bfI}^{-2} } + \frac{\sigma^2}{2} \rb{1 - \frac{\PARAMDIM}{N}} \,.
\end{align*}
\end{proof}

\begin{remark} \label{app-rem: taylor simplification of training loss}
    We intuitively derive the asymptotic behavior of \cref{app-claim: linear regression expected train loss conditioned on X}.
    Let $\lambda$ be constant, and let $\beta$ grow (so $\alpha$ shrinks).
    We can decompose $\rb{\bfA + \alpha \bfI}^{-1}$ as 
    \begin{align*}
        \rb{\bfA + \alpha \bfI}^{-1} & = \bfA^{-1} - \alpha \bfA^{-2} + \alpha^2 \bfA^{-2} \rb{\bfA + \alpha \bfI}^{-1} \,.
    \end{align*}
    This can be readily verified as
    \begin{align*}
        \bfA^{-1} & - \alpha \bfA^{-2} + \alpha^2 \bfA^{-2} \rb{\bfA + \alpha \bfI}^{-1} \\
        & = \bfA^{-2} \rb{\bfA + \alpha \bfI}^{-1} \rb{\bfA \rb{\bfA + \alpha \bfI} - \alpha \rb{\bfA + \alpha \bfI} + \alpha^2 \bfI} \\
        & = \bfA^{-2} \rb{\bfA + \alpha \bfI}^{-1} \rb{ \bfA^2 + \alpha \bfA - \alpha \bfA - \alpha^2 \bfI + \alpha^2 \bfI } \\
        & = \bfA^{-2} \rb{\bfA + \alpha \bfI}^{-1} \bfA^2 \\
        & = \rb{\bfA + \alpha \bfI}^{-1} \,,
    \end{align*}
    where we used the multiplicative commutativity, as before.
    Notice that since $\bfA \succ 0$, $\bfA + \alpha \bfI \succ \bfA$, so $\rb{\bfA + \alpha \bfI}^{-k} \prec \bfA^{-k}$ for any $k \in \bbN$.
    Denote
    \begin{align*}
        R_2 \rb{\alpha} = \alpha^2 \bfA^{-2} \rb{\bfA + \alpha \bfI}^{-1} \,,
    \end{align*}
    then $\norm{R_2 \rb{\alpha}}_2 \le \frac{\alpha^2}{\lambda_{\min} \rb{\bfA}^3}$, where $\lambda_{\min} \rb{\bfA}$ is the minimal eigenvalue of $\bfA$.
    As the elements of $\bfX$ are i.i.d. with mean $0$ and variance $1$, the limiting distribution of the spectrum of $\bfA$ as $N, \PARAMDIM \to \infty$ with $\PARAMDIM / N \to \gamma \in \rb{0, 1}$ is the Marchenko–Pastur distribution, which is supported on $\bb{\rb{1 - \sqrt{\gamma}}^2, \rb{1 + \sqrt{\gamma}}^2}$.
    In particular, as $N, \PARAMDIM \to \infty$, $\lambda_{\min} \rb{\bfA} \ge \rb{1 - \sqrt{\PARAMDIM / N}}^2$, so for $\varepsilon > 0$,
    \begin{align*}
        \norm{R_2 \rb{\alpha}}_2 \le \frac{\alpha^2}{\rb{1 - \sqrt{\PARAMDIM / N} - \varepsilon}^6} \,
    \end{align*}
    with high probability.
    Therefore, in the following we shall treat the remainder as $R_2 \rb{\alpha} = O \rb{\alpha^2}$, even when taking the expectation over $\bfX$.
    
    Since $\alpha = \lambda / \beta$ and $\lambda \propto d$, then for $d\leq \beta$, we have $\alpha / \beta = O \rb{\alpha^2}$, and we conclude that 
    \begin{align*}
        \bbE \bb{\trainloss \rb{\param_\infty} \mid \bfX} & =  \frac{\PARAMDIM}{2} \rb{\frac{1}{\beta} + \sigma^2\left(\frac{1}{d}- \frac{1}{N}\right)} + O \rb{\alpha^2} \,.
    \end{align*}
    Therefore, the added noise does not significantly hurt the training loss when $\frac{1}{\beta} \lessapprox  \sigma^2\left(\frac{1}{d}- \frac{1}{N}\right)$, or equivalently, $\beta \gtrapprox \frac{Nd}{(N-d)\sigma^2}$.
    In particular, this holds when $d\ll\beta\ll N$, which is a regime where our generalization bound \cref{cor: generalization for ngf box and regularization} also becomes small (since $\beta \ll N$). This shows that the randomness required by \cref{cor: generalization for ngf box and regularization} can allow for successful optimization of the training loss.
\end{remark}

Moving on to the population loss, we define $\poploss$ in the usual way
\begin{align*}
    \poploss \rb{\param_t} = \frac{1}{2} \bbE_{\bfx, \varepsilon} \rb{\bfx^\top \param_t - y}^2 = \frac{1}{2} \bbE \rb{\bfx^\top \param_t - \bfx^\top \paramteach - \varepsilon}^2 \,.
\end{align*}
Due to the independence between $\bfx$ and $\varepsilon$, 
\begin{align*}
    \poploss \rb{\param} = \frac{1}{2} \bbE \rb{\bfx^\top \rb{\param - \paramteach}}^2 + \frac{\sigma^2}{2} = \frac{1}{2} \norm{\param - \paramteach}^2 + \frac{\sigma^2}{2} \,.
\end{align*}

\begin{claim} \label{app-claim: linear regression expected pop loss conditioned on X}
    With some abuse of notation, denote $\poploss \rb{\param_\infty} = \bbE_{\param \sim \posteriorstationary} \poploss \rb{\param}$.
    Then
    \begin{align*}
        \bbE \bb{\poploss \rb{\param_\infty} \mid \bfX} & = \frac{1}{2\beta} \mathrm{Tr} \rb{ \rb{\bfA + \alpha \bfI}^{-1}} + \frac{1}{2} \param^{\star \top} \bfA^2 \rb{\bfA + \alpha \bfI}^{-2} \paramteach \\
        & + \frac{\sigma^2}{2N} \tr \rb{\bfA \rb{\bfA + \alpha \bfI}^{-2} } - \param^{\star \top} \bfA \rb{\bfA + \alpha \bfI}^{-1} \paramteach + \frac{1}{2} \norm{\paramteach}^2 + \frac{\sigma^2}{2} \,.
    \end{align*}
\end{claim}
\begin{proof}
Taking the expectation w.r.t $\param \sim \calN \rb{\Bar{\param}, \bSigma}$ we get from \citet{petersen2012cookbook}
\begin{align*}
    \poploss \rb{\param_\infty} & = \frac{1}{2} \mathrm{Tr} \rb{\bSigma} + \frac{1}{2} \norm{\Bar{\param} - \paramteach}^2 + \frac{\sigma^2}{2} \\
    & = \frac{1}{2\beta} \mathrm{Tr} \rb{ \rb{\bfA + \alpha \bfI}^{-1}} + \frac{1}{2} \Bar{\param}^\top \Bar{\param} - \Bar{\param}^\top \paramteach + \frac{1}{2} \norm{\paramteach}^2 + \frac{\sigma^2}{2} \,.
\end{align*}
We can simplify some of the terms when taking the expectation conditioned on $\bfX$.
\begin{align*}
    \bbE_{\bvareps} \bb{ \Bar{\param}^\top \Bar{\param} } &= \frac{1}{N^2} \bbE \bb{ \bfy^\top \bfX \rb{\bfA + \alpha \bfI}^{-1} \rb{\bfA + \alpha \bfI}^{-1} \bfX^\top \bfy } \\
    & = \frac{1}{N^2} \bbE \bb{ \rb{\bfX \paramteach + \bvareps}^\top \bfX \rb{\bfA + \alpha \bfI}^{-2} \bfX^\top \rb{\bfX \paramteach + \bvareps} } \\
    & = \frac{1}{N^2} \param^{\star \top} \bfX^\top \bfX \rb{\bfA + \alpha \bfI}^{-2} \bfX^\top \bfX \paramteach + \frac{1}{N^2} \bbE_{\bvareps} \bb{ \bvareps^\top \bfX \rb{\bfA + \alpha \bfI}^{-2} \bfX^\top \bvareps } \\
    & = \param^{\star \top} \bfA^2 \rb{\bfA + \alpha \bfI}^{-2} \paramteach + \frac{\sigma^2}{N^2} \tr \rb{\bfX \rb{\bfA + \alpha \bfI}^{-2} \bfX^\top} \\
    & = \param^{\star \top} \bfA^2 \rb{\bfA + \alpha \bfI}^{-2} \paramteach + \frac{\sigma^2}{N} \tr \rb{\bfA \rb{\bfA + \alpha \bfI}^{-2} } \,.
\end{align*}
In addition,
\begin{align*}
    \bbE_{\bvareps} \bb{ \Bar{\param}^\top \paramteach } & = \frac{1}{N} \bbE_{\bvareps} \bb{ \rb{\bfX \paramteach + \bvareps}^\top \bfX \rb{\bfA + \alpha \bfI}^{-1} \paramteach } \\
    & = \frac{1}{N} \param^{\star \top} \bfX^\top \bfX \rb{\bfA + \alpha \bfI}^{-1} \paramteach + \frac{1}{N} \bbE_{\bvareps} \bb{ \bvareps^\top \bfX \rb{\bfA + \alpha \bfI}^{-1} \paramteach } \\
    & = \param^{\star \top} \bfA \rb{\bfA + \alpha \bfI}^{-1} \paramteach \,.
\end{align*}
Combining these we get the desired result.
\end{proof}

\begin{remark} \label{app-rem: taylor simplification population loss}
    As we have done for the training loss in \cref{app-rem: taylor simplification of training loss}, we can estimate the expected population loss in some asymptotic regimes.    
    Let $\lambda$ be constant, and let $\beta$ grow (so $\alpha$ shrinks).
    As in \cref{app-rem: taylor simplification of training loss}, we use the approximation $\rb{\bfA + \alpha \bfI}^{-1} = \bfA^{-1} - \alpha \bfA^{-2} + O \rb{\alpha^2 \bfI}$, which also implies $\rb{\bfA + \alpha \bfI}^{-2} = \bfA^{-2} - 2 \alpha \bfA^{-3} + O \rb{\alpha^2 \bfI}$, and treat the remainders as $O \rb{\alpha^2}$ even when taking the expectation w.r.t. $\bfX$.
    Then,
    \begin{align*}
        \bbE \bb{\poploss \rb{\param_\infty} \mid \bfX} &= \frac{1}{2\beta} \tr \rb{\bfA^{-1} - \alpha \bfA^{-2} + O \rb{\alpha^2 \bfI}} + \frac{1}{2} \param^{\star \top} \bfA^2 \rb{\bfA^{-2} - 2\alpha \bfA^{-3} + O \rb{\alpha^2 \bfI}} \paramteach \\
        & \quad + \frac{\sigma^2}{2N} \tr \rb{\bfA \rb{\bfA^{-2} - 2\alpha \bfA^{-3} + O \rb{\alpha^2 \bfI}}} \\
        & \quad - \param^{\star \top} \bfA \rb{\bfA^{-1} - \alpha \bfA^{-2} + O \rb{\alpha^2 \bfI}} \paramteach + \frac{1}{2} \norm{\paramteach}^2 + \frac{\sigma^2}{2} \\
        & = \frac{1}{2} \rb{\frac{1}{\beta} + \frac{\sigma^2}{N}} \tr \rb{\bfA^{-1}} + \frac{\sigma^2}{2} \\
        & \quad - \frac{\alpha}{2\beta} \tr \rb{\bfA^{-2} + O \rb{\alpha \bfI}} - \alpha \param^{\star \top} \rb{\bfA^{-1} + O \rb{\alpha \bfI}} \paramteach \\
        & \quad - \frac{\sigma^2 \alpha}{N} \tr \rb{\bfA^{-2} + O \rb{\alpha \bfI}} + \alpha \param^{\star \top} \rb{\bfA^{-1} + O \rb{\alpha \bfI}} \paramteach \,.
    \end{align*}
    Simplifying, we arrive at
    \begin{align*}
        \bbE \bb{\poploss \rb{\param_\infty} \mid \bfX} & = \frac{1}{2} \rb{\frac{1}{\beta} + 
        \frac{\sigma^2}{N}} \tr \rb{\bfA^{-1}} + \frac{\sigma^2}{2} - \alpha \rb{\frac{1}{2\beta} + \frac{\sigma^2}{N}} \tr \rb{\bfA^{-2}} + O \rb{\alpha^2} \,.
    \end{align*}
    \textbf{Assuming} that $\bfx$ are i.i.d. $\calN \rb{\bfzero, \bfI}$, $N \cdot \bfA \sim \calW_{\PARAMDIM} \rb{N, \bfI}$, \ie has a Wishart distribution.
    According to Theorem 3.3.16 of \citep{gupta1999matrix}, if $N > \PARAMDIM + 3$ then 
    \begin{align*}
        \bbE \bfA^{-1} & = \frac{N}{N - \PARAMDIM - 1} \bfI \,, \\
        \bbE \bfA^{-2} & = N^2 \cdot \frac{\tr \rb{\bfI} \bfI}{ \rb{N - \PARAMDIM} \rb{N - \PARAMDIM - 1} \rb{N - \PARAMDIM - 3}} + N^2 \cdot \frac{\bfI}{\rb{N - \PARAMDIM} \rb{N - \PARAMDIM - 3}} \\
        & = \frac{N^2 \PARAMDIM + N^2 \rb{N - \PARAMDIM - 1}}{\rb{N - \PARAMDIM} \rb{N - \PARAMDIM - 1} \rb{N - \PARAMDIM - 3}} \bfI \,.
    \end{align*}
    Then, the expectation over $\bfX$ and if $\frac{\sigma^2}{N} \lessapprox \alpha$ (which is true for $\lambda\propto d$ and $\beta\ll N$ like we assume here),
    \begin{align*}
        \bbE \poploss \rb{\param_\infty} & = \frac{1}{2} \rb{\frac{1}{\beta} + \sigma^2 \rb{\frac{1}{N} + \frac{N - \PARAMDIM - 1}{N \PARAMDIM}}} \cdot \frac{N \PARAMDIM}{N - \PARAMDIM - 1} + O \rb{\alpha^2} \\
        & = \frac{1}{2} \rb{\frac{1}{\beta} + \sigma^2 \cdot \frac{N - 1}{N \PARAMDIM}} \cdot \frac{N \PARAMDIM}{N - \PARAMDIM - 1} + O \rb{\alpha^2} \,.
    \end{align*}
    This result is similar to the one in \cref{app-rem: taylor simplification of training loss} --- for the expected population loss not to be significantly hurt by the added noise, it must hold that $\beta \gtrapprox \frac{N \PARAMDIM}{\rb{N - 1} \sigma^2}$. In particular, this holds when $d\ll\beta\ll N$, which is a regime where our generalization bound \cref{cor: generalization for ngf box and regularization} also becomes small (since $\beta \ll N$).
    This shows that the randomness required by \cref{cor: generalization for ngf box and regularization} does not harm the expected population loss.
\end{remark}

\newpage
\section{Numerical Experiments}\label{app-sec: experiments} 

\subsection{Experimental results}
The following are results of training with SGLD (a discretized version of the CLD in 
\eqref{eq: train loss sde regularized}) on a few benchmark datasets.
Notice we use the regularized version where regularization coefficient is $\lambda \cdot \beta^{-1}$ and the $\lambda$ hyperparameter is dictated by the initialization from the normal distribution $p_0 = \calN \rb{\bfzero, \lambda^{-1} \bfI_{\PARAMDIM}}$.
We used a common initialization of $\calN \rb{\bfzero, \frac{1}{d_{\mathrm{in}}}}$, \ie $\lambda=d_{\mathrm{in}}$.

We use several different values of $\beta$ relative to $N$ (the number of training samples).
For simplicity, we focused on binary classification cases.
In all datasets with more than 2 classes, we constructed a binary classification task by partitioning the original label set into $2$ disjoint sets of the same size.

The results demonstrate that learning with SGLD is possible with various values of $\beta$.
In fact, in several instances, the injected noise appears to improve the \emph{generalization gap}, e.g, in SVHN \citep{netzer2011reading}, in all the tested $\beta$ values between $0.4 \cdot N$ and $2 \cdot N$ the average test error remained almost the same while the training error decreased as $\beta$ increased (\ie the generalization gap increased).
Notably, we also observe that for sufficiently large levels of noise, the generalization bounds are non-vacuous.

\begin{table}[!h]
\small
\centering
\caption{\label{app-tab:MNIST}
\textbf{MNIST} (binary classification)
}
\begin{tabular}{ccccc}
\toprule
$\beta$ & $\trainfail$ & $\expectedfail$ & $ \expectedfail - \trainfail $ & Bound \eqref{eq: cld bound in expectation} w.p 0.99 \\
\midrule
$0.01 \cdot N$ & $0.2279 \space (\pm 0.0021)$ & $0.1972 \space (\pm 0.0243)$ & -0.0307 & 0.06124 \\
$0.03 \cdot N$ & $0.1161 \space (\pm 0.0028)$ & $0.1074 \space (\pm 0.0035)$ & -0.0087 & 0.10498 \\
$0.1 \cdot N$ & $0.0618 \space (\pm 0.001)$ & $0.062 \space (\pm 0.0041)$ & 0.0002 & 0.19096 \\
$0.15 \cdot N$ & $0.0497 \space (\pm 0.0014)$ & $0.0494 \space (\pm 0.0031)$ & -0.0003 & 0.23376 \\
$0.4 \cdot N$ & $0.0281 \space (\pm 0.0002)$ & $0.0358 \space (\pm 0.0029)$ & 0.0077 & 0.38147 \\
$0.7 \cdot N$ & $0.0202 \space (\pm 0.0006)$ & $0.0284 \space (\pm 0.0024)$ & 0.0082 & 0.50456 \\
$N$ & $0.0162 \space (\pm 0.0006)$ & $0.0278 \space (\pm 0.0023)$ & 0.0116 & 0.60302 \\
$2 \cdot N$ & $0.0092 \space (\pm 0.0004)$ & $0.0262 \space (\pm 0.0016)$ & 0.017 & 0.85273 \\
$\infty$ & $0.0001 \space (\pm 0)$ & $0.0229 \space (\pm 0.0004)$ & 0.0228 & $>1$ \\
\bottomrule
\end{tabular}
\end{table}

\begin{table}[!h]
\small
\centering
\caption{ \label{app-tab:FMNIST}
\textbf{fashionMNIST} (binary classification)
}
\begin{tabular}{ccccc}
\toprule
$\beta$ & $\trainfail$ & $\expectedfail$ & $ \expectedfail - \trainfail $ & Bound \eqref{eq: cld bound in expectation} w.p 0.99 \\
\midrule
$0.01 \cdot N$ & $0.1215 \space ( \pm 0.0027)$& $0.1251 \space (\pm 0.0087)$ & 0.0036 & 0.06833 \\
$0.03 \cdot N$ & $0.0999 \space ( \pm 0.001)$& $0.1087 \space (\pm 0.0167)$ & 0.0088 & 0.11738 \\
$0.1 \cdot N$ & $0.0821 \space ( \pm 0.0012)$& $0.086 \space (\pm 0.001)$ & 0.0039 & 0.21368 \\
$0.15 \cdot N$ & $0.0765 \space ( \pm 0.0009)$ & $0.0803 \space (\pm 0.0015)$ & 0.0038 & 0.26159 \\
$0.4 \cdot N$ & $0.0635 \space ( \pm 0.0005)$ & $0.0722 \space (\pm 0.002)$& 0.0087 & 0.42695 \\
$0.7 \cdot N$ & $0.0567 \space ( \pm 0.0006)$ & $0.0691 \space (\pm 0.0019)$ & 0.0124 & 0.56473 \\
$N$ & $0.0525 \space (\pm 0.0005)$ & $0.0675 \space (\pm 0.0013)$ & 0.015 & 0.67495 \\
$2 \cdot N$ & $0.043 \space (\pm 0.0007)$ & $0.0672 \space (\pm 0.0023)$ & 0.0242 & 0.95446 \\
$\infty$ & $0.0248 \space (\pm 0.001)$ & $0.0675 \space (\pm 0.0033)$ & 0.0427 & $>1$ \\
\bottomrule
\end{tabular}
\end{table}

\begin{table}[!h]
\small
\centering
\caption{\label{app-tab:SVHN}
\textbf{SVHN} (binary classification)
}
\begin{tabular}{ccccc}
\toprule
$\beta$ & $\trainfail$ & $\expectedfail$ & $ \expectedfail - \trainfail $ & Bound \eqref{eq: cld bound in expectation} w.p 0.99 \\
\midrule
$0.01 \cdot N$ & $0.0746 \space (\pm 0.0012)$ & $0.1033 \space (\pm 0.0032)$ & 0.0287 & 0.05898 \\ 
$0.03 \cdot N$ & $0.0441 \space (\pm 0.0004)$ & $0.067 \space (\pm 0.0026)$ & 0.0229 & 0.10203  \\
$0.1 \cdot N$ & $0.0282 \space (\pm 0.0008)$ & $0.0476 \space (\pm 0.007)$ & 0.0194 & 0.1862  \\
$0.15 \cdot N$ & $0.0251 \space (\pm 0.0005)$ & $0.0445 \space (\pm 0.002)$ & 0.0194 & 0.22803  \\
$0.4 \cdot N$ & $0.0182 \space (\pm 0.0005)$ & $0.0374 \space (\pm 0.0017)$ & 0.0192 & 0.37235  \\
$0.7 \cdot N$ & $0.0146 \space (\pm 0.0004)$ & $0.0363 \space (\pm 0.002)$ & 0.0217 & 0.49256  \\
$N$ & $0.0124 \space (\pm 0.0002)$ & $0.0342 \space (\pm 0.0014)$ & 0.0218 & 0.58872  \\
$2\cdot N$ & $0.0085 \space (\pm 0)$ & $0.0371 \space (\pm 0.001)$ & 0.0286 & 0.83256 \\

\bottomrule
\end{tabular}
\end{table}

\begin{figure}[!ht]
\centering

\begin{minipage}[b]{0.25\textwidth}
\includegraphics[width=\textwidth]{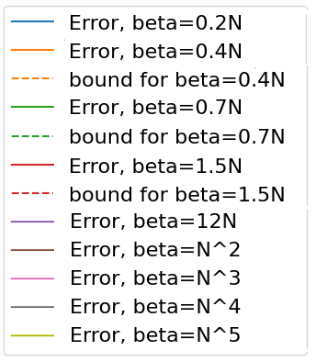}
\end{minipage}
\hfill

\begin{minipage}[b]{1\textwidth}
\includegraphics[width=\textwidth]{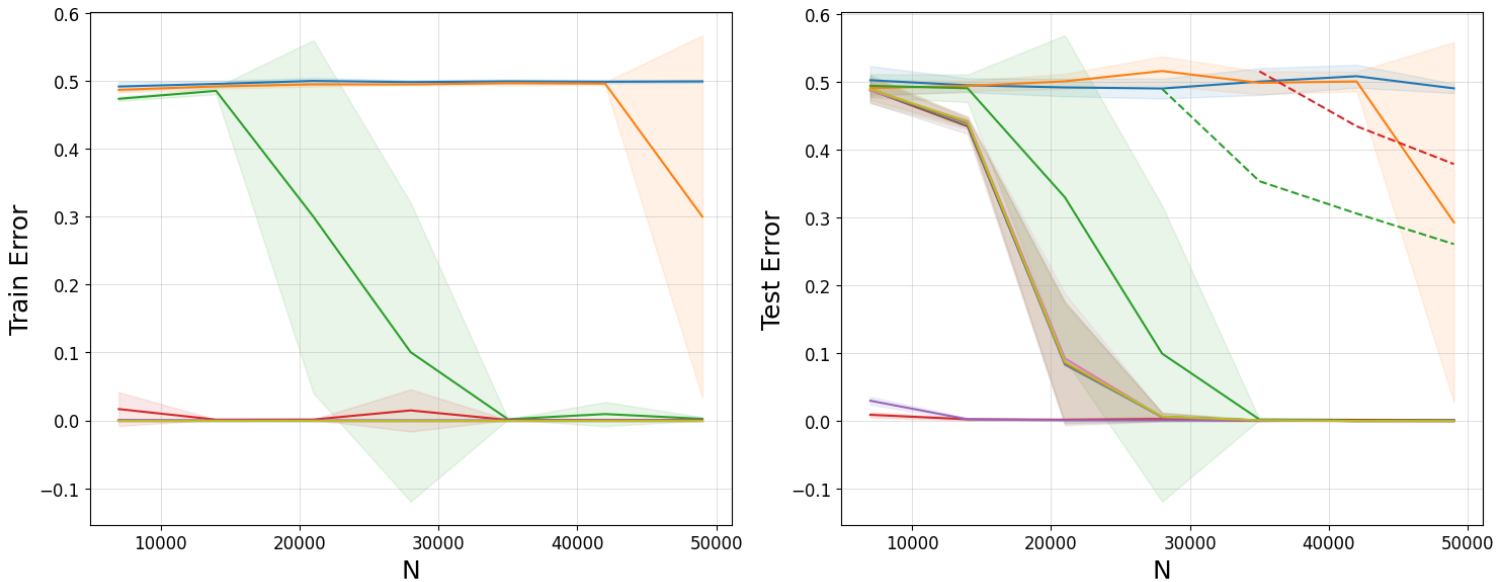}
\end{minipage}

\caption{\label{app-fig:parity}
\textbf{Parity Results.}
Left: Training error. Right: test error and generalization bound.
}
\end{figure}

\subsection{Training details}

\paragraph{MNIST and fashionMNIST.} 
We trained a fully connected network with 4 hidden layers of sizes $[256, 256, 256, 128]$ and ReLU activation, $lr=0.01$, for 60 epochs.

\paragraph{SVHN.} 
The network was trained with a convolutional neural network with 5 convolutional layers, $lr=0.01$, for 80 epochs. The complete architecture:

\begin{itemize}
  \item Two convolutional layers (3×3 kernel, padding 1) with 32 channels, followed by ReLU activations and a 2×2 max pooling.
  \item Two convolutional layers (3×3 kernel, padding 1) with 64 channels, followed by ReLU activations and a 2×2 max pooling.
  \item A 3×3 convolution with 128 channels, ReLU, and 2×2 max pooling.
  \item 2 A linear layer $2048 \rightarrow 512$, followed by ReLU and another $512 \rightarrow 1$ linear layer
\end{itemize}

\paragraph{Parity.} 
In this experiment, we consider a synthetic binary classification task where each input is a binary vector of length 70 and the target label is defined as the parity of 3 randomly selected input dimensions. 
We train a neural network using SGLD with varying values of the inverse temperature parameter $\beta$ and different sample sizes.

The network was trained with a fully connected network with 4 hidden layers of sizes $[512, 1028, 2064, 512]$ and ReLU activation, $lr=0.05$, for 100 epochs.

The results show that injecting noise can improve the generalization gap: specifically, the case of $\beta\geq N^2$ leads to overfitting, while smaller values of $\beta$ (e.g., $1.5 \cdot N$ to $12 \cdot N$) yield better generalization.
Moreover, as well as in the benchmark datasets, in this setting, our generalization bound is non-vacuous in several cases.

\subsection{Comparison with the bound of \texorpdfstring{\citet{mou2018generalization}}{}} 
The bound proposed by \citet{mou2018generalization} has demonstrated non-vacuous results. 
To further assess the effectiveness of our bound and evaluate its relative tightness, we conducted a series of numerical experiments on the MNIST binary classification task (see Tables \ref{Tab:Mou comparison 1}-\ref{Tab:Mou comparison 4}).

It is worth emphasizing that our bound offers a distinct advantage: it can be evaluated directly at initialization, whereas the bound of \citet{mou2018generalization} depends on gradients and therefore cannot be computed before training. When testing their bound we used the continuous version, i.e.

\begin{equation*}
    \bbE_{p_T} \! \bb{\expectedfail \! \rb{\param} \! - \trainfail \! \rb{\param}} \leq s \left(
    \frac{\beta}{2n}\int_{0}^{T} e^{\frac{\lambda}{2}(T-t)}
    \mathbb{E}_{p_{t}}\!\left[\Vert\nabla \trainloss (\param) \Vert^{2}\right] dt 
    + \frac{\log (1/\delta) + \log \log M}{n}
    \right)^{0.5} \,.
\end{equation*}

For simplicity, we omitted the term involving $M$ (which makes the bound more favorable). 
In addition, we set $s = 0.5$ since the zero–one loss (denoted here by $f(w)$, unlike \citep{mou2018generalization}) is bounded within the interval $[0,1]$. 
We observed that the relative tightness of the two bounds varies across different values of $\beta$ 
and at different points in time. 
Consequently, in some instances, the bound of \citet{mou2018generalization} is tighter,
while in others our bound performs better, and we could not draw any further conclusions. 

\begin{table}[!h]
\small
\centering
\caption{\label{Tab:Mou comparison 1}
\textbf{20 training epochs}
}
\begin{tabular}{cccccc}
\toprule
$\beta$ & Train Error & Test Error & Generalization Gap & \citet{mou2018generalization} & Our bound \\
\midrule
$0.03N=1800$  & $0.1224$ & $0.137$  & $0.0146$  & $0.0539$  & $0.1144$  \\
$0.15N=9000$  & $0.0515$ & $0.0747$ & $0.0232$  & $0.1279$  & $0.2548$ \\
$0.4N=24000$  & $0.0335$ & $0.058$  & $0.0245$  & $0.2845$  & $0.4157$ \\
$0.7N=42000$  & $0.0278$ & $0.0498$ & $0.0220$  & $0.4930$  & $0.5499$ \\
$N=60000$     & $0.0249$ & $0.0428$ & $0.0179$  & $0.7032$  & $0.6572$ \\
$2N=120000$   & $0.0209$ & $0.0356$ & $0.0147$  & $1.4044$  & $0.9294$ \\
\hline
\end{tabular}
\end{table}

\begin{table}[!h]
\small
\centering
\caption{\label{Tab:Mou comparison 2}
\textbf{50 training epochs}
}
\begin{tabular}{cccccc}
\toprule
$\beta$ & Train Error & Test Error & Generalization Gap & \citet{mou2018generalization} & Our bound \\
\midrule
$0.03N=1800$  & $0.1156$ & $0.1697$ & $0.0541$ & $0.0637$  & $0.1144$  \\
$0.15N=9000$  & $0.0491$ & $0.0615$ & $0.0124$ & $0.1324$  & $0.2548$ \\
$0.4N=24000$  & $0.0295$ & $0.0348$ & $0.0053$ & $0.2992$  & $0.4157$ \\
$0.7N=42000$  & $0.0217$ & $0.0283$ & $0.0066$ & $0.4903$  & $0.5499$ \\
$N=60000$     & $0.0173$ & $0.0277$ & $0.0104$ & $0.6827$  & $0.6572$ \\
$2N=120000$   & $0.0108$ & $0.0265$ & $0.0157$ & $1.3153$  & $0.9294$ \\
\hline
\end{tabular}
\end{table}

\begin{table}[!h]
\small
\centering
\caption{\label{Tab:Mou comparison 3}
\textbf{250 training epochs}
}
\begin{tabular}{cccccc}
\toprule
$\beta$ & Train Error & Test Error & Generalization Gap & \citet{mou2018generalization} & Our bound \\
\midrule
$0.03N=1800$  & $0.122$  & $0.1049$ & $-0.0171$ & $0.1273$  & $0.1144$  \\
$0.15N=9000$  & $0.0502$ & $0.0476$ & $-0.0026$ & $0.1503$  & $0.2548$ \\
$0.4N=24000$  & $0.0284$ & $0.0296$ & $0.0011$  & $0.2853$  & $0.4157$ \\
$0.7N=42000$  & $0.0178$ & $0.0247$ & $0.0069$  & $0.4595$  & $0.5499$ \\
$N=60000$     & $0.0127$ & $0.0240$ & $0.0113$  & $0.6478$  & $0.6572$ \\
$2N=120000$   & $0.0050$ & $0.0234$ & $0.0184$  & $1.2158$  & $0.9294$ \\
\hline
\end{tabular}
\end{table}

\begin{table}[!h]
\small
\centering
\caption{\label{Tab:Mou comparison 4}
\textbf{400 training epochs}
}
\begin{tabular}{cccccc}
\toprule
$\beta$ & Train Error & Test Error & Generalization Gap & \citet{mou2018generalization} & Our bound \\
\midrule
$0.03N=1800$  & $0.1224$ & $0.1105$ & $-0.0119$ & $0.1900$  & $0.1144$  \\
$0.15N=9000$  & $0.0499$ & $0.0556$ & $0.0057$  & $0.1774$  & $0.2548$ \\
$0.4N=24000$  & $0.0261$ & $0.0357$ & $0.0096$  & $0.3005$  & $0.4157$ \\
$0.7N=42000$  & $0.0161$ & $0.0271$ & $0.0110$  & $0.4548$  & $0.5499$ \\
$N=60000$     & $0.0112$ & $0.0255$ & $0.0143$  & $0.6247$  & $0.6572$ \\
$2N=120000$   & $0.0038$ & $0.0249$ & $0.0211$  & $1.1455$  & $0.9294$ \\
\hline
\end{tabular}
\end{table}

\newpage
\section{Mild Overparametrization Prevents Uniform Convergence}  \label{app-sec: no_uc}

In this section, we consider fully-connected ReLU networks, where the weights are bounded, such that for each layer $j$ the absolute values of all weights are bounded by $\frac{1}{\sqrt{d_{j-1}}}$, where $d_{j-1}$ is the width of layer $j-1$. Moreover, we assume that the input $\bfx$ is such that each coordinate $x_i$ is bounded in $[-1,1]$. We show that $m$ training examples do not suffice for learning constant depth networks with $O(m)$ parameters. Thus, even a mild overparameterization prevents uniform convergence in our setting.

Our result follows by bounding the fat-shattering dimension, defined as follows:
\begin{definition}
	Let $\mathcal{F}$ be a class of real-valued functions from an input domain $\mathcal{X}$. We say that $\mathcal{F}$ shatters $m$ points $\{\bfx_i\}_{i=1}^m \subseteq \mathcal{X}$ with margin $\epsilon>0$ if there are $r_1,\ldots,r_m \in \mathbb{R}$ such that for all $y_1,\ldots,y_m \in \{0,1\}$ there exists $f \in \mathcal{F}$ such that 
	\[
		\forall i \in [m],\;\; f(\bfx_i) \leq r_i - \epsilon \;\text{ if }\; y_i=0 \;\text{ and }\; f(\bfx_i) \geq r_i + \epsilon \; \text{ if }\; y_i=1~.
	\]
	The fat-shattering dimension of $\mathcal{F}$ with margin $\epsilon$ is the maximum cardinality $m$ of a set of points in $\mathcal{X}$ for which the above holds.
\end{definition}

The fat-shattering dimension of $\mathcal{F}$ with margin $\epsilon$ lower bounds the number of samples needed to learn $\mathcal{F}$ within accuracy $\epsilon$ in the distribution-free setting (see, e.g., \cite[Part III]{anthony2009neural}).
Hence, to lower bound the sample complexity by some $m$ it suffices to show that we can shatter a set of $m$ points with a constant margin.

\begin{theorem} \label{thm:shattering}
    We can shatter $m$ points $\{\bfx_i\}_{i=1}^m$ where $\norm{\bfx_i}_{\infty} \leq 1$, with margin $1$, using ReLU networks of constant depth and $O(m)$ parameters, such that for each layer $j$ the absolute values of all weights are bounded by $\frac{1}{\sqrt{d_{j-1}}}$, where $d_{j-1}$ is the width of layer $j-1$.
\end{theorem}

\begin{proof}
    Consider input dimension $d_0=1$. For $1 \leq i \leq m$, consider the points $x_i = \frac{i}{m}$, and let $\{y_i\}_{i=1}^m \subseteq \{0,1\}$. Consider the following one-hidden-layer ReLU network $N$, which satisfies $N(x_i)=\frac{y_i}{m}$ for all $i$. First, the network $N$ includes a neuron with weight $0$ and bias $\frac{y_1}{m}$, i.e., $[0 \cdot x + \frac{y_1}{m}]_+$. Now, for each $i$ such that $y_i = 0$ and $y_{i+1}=1$ we add two neurons: $[x - y_i]_+ - [x - y_{i+1}]_+$, and for $i$ such that $y_i = 1$ and $y_{i+1}=0$ we add $-[x - y_i]_+ + [x - y_{i+1}]_+$. It is easy to verify that this construction has width at most $2m - 1$ and allows us to shatter $m$ points with margin $\frac{1}{2m}$. However, the output weights of the neurons are $\pm 1$, and thus it does not satisfy the theorem's requirement. Consider the network $N'(x) = N(x) \cdot \frac{1}{\sqrt{2m-1}}$ obtained from $N$ by modifying the output weights. The network $N'$ satisfies the theorem's requirement on the weight magnitudes, and allows for shattering with margin $\frac{1}{2m\sqrt{2m-1}}$. We will now show how to increase this margin to $1$ using a constant number of additional layers.

    Let $\tilde{N}$ be a network obtained from $N'$ as follows. First, we add a ReLU activation to the output neuron of $N'$. Since for every $x_i$ we have $N'(x_i) \geq 0$, it does not affect these outputs. Next, we add $L=8$ additional layers (layers $3,\ldots,3+L-1$) of width $\sqrt{m}$ and without bias terms, where the incoming weights to layer $3$ are all $1$ and the weights in layers $4,\ldots,3+L-1$ are $\frac{1}{m^{1/4}}$. Finally, we add an output neuron (layer $3+L$) with incoming weights $\frac{1}{m^{1/4}}$. The network $\tilde{N}$ satisfies the theorem's requirements on the weight magnitudes, and it has depth $3+L=11$ and $O(m)$ parameters. Now, suppose that all neurons in a layer $3 \leq j \leq 3+L-1$ have values (i.e., activations) $z \geq 0$, then the values of all neurons in layer $j+1$ are $z \cdot \frac{1}{m^{1/4}} \cdot \sqrt{m} = z \cdot m^{1/4}$. Hence, if the value of the neuron in layer $2$ is $\frac{1}{2m\sqrt{2m-1}}$, then the output of the network $\tilde{N}$ is $\frac{1}{2m\sqrt{2m-1}} \cdot (m^{1/4})^{L} = \frac{m^{L/4}}{2m\sqrt{2m-1}} = \frac{m^2}{2m\sqrt{2m-1}} \geq 2$ for large enough $m$. If the value of the neuron in layer $2$ is $0$ then the output of $\tilde{N}$ is also $0$. Hence, this construction allow for shattering $m$ points with margin at least $1$, using $O(m)$ parameters and weights that satisfy the theorem's conditions.
\end{proof}

\newpage
\section{Background on Stochastic Differential Equations with Reflection} \label{app-sec: background}

We supply an introduction to the theory of stochastic differential equations with reflection (SDERs), then proceed to characterize the stationary distribution of a family of SDERs in a box.
The background of standard (non-reflective) SDEs is similar and more common, and is therefore not included here.
See for example \citep{Oksendal2003} for more.

\subsection{SDEs with reflection} \label{app-sec: sder}
One of the main analytical tools of this work is the characterization of stationary distributions of SDER in bounded domains (see \citealp{sder-pilipenko-2014,Schuss2013}, for an introduction). 

The purpose of this section is to present more rigorously the setting of the paper, and supply the relevant definitions and results required to arrive at \cref{app-lem: langevin in a box stationary}.
As \cref{app-lem: langevin in a box stationary} is considered a well-known result, this section is mainly intended for completeness. Specifically, in the following we present some relevant definitions and results by \citet{kang2014characterization,kang2017submartingale}, and specifically, ones that relate solutions to SDERs (Definition 2.4 in \citep{kang2017submartingale}), to solutions to sub-martingale problems (Definition 2.9 in \citep{kang2017submartingale}), and that characterize the stationary distributions of such solutions. 
For simplicity, we sometimes do not state the results in full generality.

\paragraph{Setting.}
Let $\DOM \subset \bbR^\PARAMDIM$ be a domain (non-empty, connected, and open).
Let the drift term $\drift : \bbR^\PARAMDIM \to \bbR^\PARAMDIM$ and dispersion coefficient $\dispcoef : \bbR^\PARAMDIM \to \bbR^{\PARAMDIM \times \PARAMDIM}$ be measurable and locally bounded.
We also denote the diffusion coefficient by $\diffcoef \rb{\cdot} = \dispcoef \rb{\cdot} \dispcoef \rb{\cdot}^\top = \rb{\diffij \rb{\cdot}}_{i,j=1}^\PARAMDIM$, and denote its columns by $\diffcol_i \rb{\cdot}$.
We say that the diffusion coefficient is uniformly elliptic if there exists $\sigma > 0$ such that
\begin{align} \label{app-eq: unifom ellipticity}
    \forall \bfv \in \bbR^\PARAMDIM \,, \; \forall \bfx \in \DOMCLOS \quad \bfv^\top \diffcoef \rb{\bfx} \bfv > \sigma \norm{\bfv} \,.
\end{align}
Let $\BOUNDFIELD$ be a set valued mapping of allowed reflection directions defined on $\DOMCLOS$ such that $\BOUNDFIELD \rb{\bfx} = \cb{\vect{0}}$ for $\bfx \in \DOM$, and $\BOUNDFIELD \rb{\bfx}$ is a non-empty, closed and convex cone in $\bbR^\PARAMDIM$ such that $\cb{\vect{0}} \subseteq \BOUNDFIELD \rb{\bfx}$ for $\bfx \in \partial \DOM$, and furthermore assume that the set $\cb{\rb{\bfx, \bfv} \, : \, \bfx \in \DOMCLOS, \bfv \in \BOUNDFIELD\rb{\bfx}}$ is closed in $\bbR^{2 \PARAMDIM}$.
In addition, for $\bfx \in \partial \DOM$ let $\NORMFIELD \rb{\bfx}$ be the set of inwards normals to $\DOM$ at $\bfx$, 
\begin{align*}
    \NORMFIELD \rb{\bfx} = \bigcup_{r > 0} \, \NORMFIELD_r \rb{\bfx} \,,
\end{align*}
\vspace{-4mm}
\begin{align*}
    \NORMFIELD_r \rb{\bfx} = \cb{\NORMAL \in \bbR^\PARAMDIM \setst \norm{\NORMAL} = 1,\, B_r \rb{\bfx - r\NORMAL} \cap \DOM = \emptyset} \,.
\end{align*}
Then, denote the set of boundary points with inward pointing cones
\begin{align*}
    \calU \triangleq \cb{\bfx \in \partial \DOM \, \mid \, \exists \NORMAL \in \NORMFIELD \rb{\bfx} \, : \, \forall \BOUNDVEC \in \BOUNDFIELD \rb{\bfx} \; \innerprod{\NORMAL}{\BOUNDVEC} > 0} \,,
\end{align*}
and let $\calV \triangleq \partial \DOM \setminus \calU$. 
For example, if $\DOM$ is a convex polyhedron and $\BOUNDFIELD \rb{\bfx}$ is the cone defined by the positive span of $\NORMFIELD \rb{\bfx}$ we get that $\calV = \emptyset$.

Throughout this section and the rest of the paper, the \emph{stochastic differential equation with reflection} (SDER) in $\rb{\DOM, \BOUNDFIELD}$
\begin{align} \label{app-eq: sder general}
    d\bfx_t = \drift \rb{\bfx_t} dt + \dispcoef \rb{\bfx_t} d\bfw_t + d\reflectionprocess \,,
\end{align}
where $\bfw_t$ is a Wiener process, and $\reflectionprocess$ is a reflection process with respect to some filtration, is understood as in Definition 2.4 of \citep{kang2017submartingale}, and the \emph{submartingale problem} associated with $\rb{\DOM, \BOUNDFIELD}$, $\calV$, $\drift$ and $\dispcoef$, refers to Definition 2.9 of \citep{kang2017submartingale}.
In addition, we use the following definition.

\begin{definition}[Piecewise $\calC^2$ with continuous reflection; Definition 2.11 in \citep{kang2017submartingale}] \label{app-def: piecesie c2}
    The pair $\rb{\DOM, \BOUNDFIELD}$ is said to be piecewise $\calC^2$ with continuous reflection if it satisfies the following properties:
    \begin{enumerate}
        \item $\DOM$ is a non-empty domain in $\bbR^\PARAMDIM$ with representation 
        \begin{align*}
            \DOM = \bigcap_{i\in \calI} \DOM^i \,,
        \end{align*}
        where $\calI$ is a finite set and for each $i \in \calI$, $\DOM^i$ is a non-empty domain with $\calC^2$ boundary in the sense that for each $\bfx \in \partial \DOM$, there exist a neighborhood $\neighb{\bfx}$ of $\bfx$, and functions $\varphi^i_{\bfx} \in \calC^2 \rb{\bbR^\PARAMDIM}$, $i \in \calI \rb{\bfx} = \cb{i \in \calI \setst \bfx \in \partial \DOM^i}$, such that
        \begin{align*}
            \neighb{\bfx} \cap \DOM^i = \cb{\bfz \in \neighb{\bfx} \setst \varphi^i_{\bfx} \rb{\bfz} > 0}\,, \; \neighb{\bfx} \cap \partial \DOM^i = \cb{\bfz \in \neighb{\bfx} \setst \varphi^i_{\bfx} \rb{\bfz} = 0} \,,
        \end{align*}
        and $\nabla \varphi^i_{\bfx} \neq \bfzero$ on $\neighb{\bfx}$.
        For each $\bfx \in \partial \DOM^i$ and $i \in \calI \rb{\bfx}$, let 
        \begin{align*}
            \NORMAL^i \rb{\bfx} = \frac{\nabla \varphi^i_{\bfx}}{\norm{\nabla \varphi^i_{\bfx}}}
        \end{align*}
        denote the unit inward normal vector to $\partial \DOM^i$ at $\bfx$.
        \item The (set-valued) direction ``vector field'' $\BOUNDFIELD : \DOMCLOS \to \bbR^\PARAMDIM$ is given by 
        \begin{align} \label{app-eq: piecewise continuous reflection boundary field}
            \BOUNDFIELD \rb{\bfx} = \begin{cases}
                \cb{\bfzero} & \bfx \in \DOM \,, \\
                \cb{\sum_{i \in \calI \rb{\bfx}} \alpha_i \BOUNDVEC^i \rb{\bfx} \setst \alpha_i \ge 0\,, \; i \in \calI \rb{\bfx}} & \bfx \in \partial \DOM \,,
            \end{cases}
        \end{align}
        where for each $i \in \calI$, $\BOUNDVEC^i \rb{\cdot}$ is a continuous unit vector field defined on $\partial \DOM^i$ that satisfies for all $\bfx \in \partial \DOM^i$
        \begin{align*}
            \innerprod{\NORMAL^i \rb{\bfx}}{\BOUNDVEC^i \rb{\bfx}} > 0 \,.
        \end{align*}
        If $\BOUNDFIELD^i \rb{\cdot}$ is constant for every $i \in \calI$, the the pair $\rb{\DOM, \BOUNDFIELD}$ is said to be piecewise $\calC^2$ with constant reflection. 
        If, in addition, $\NORMAL^i \rb{\cdot}$ is constant for every $i \in \calI$, then the pair $\rb{\DOM, \BOUNDFIELD}$ is said to be polyhedral with piecewise constant reflection.
    \end{enumerate}
\end{definition}

In addition, let $\calS$ denote the smooth parts of $\partial \DOM$.

\begin{remark} \label{app-rem: examples of piecewise c2}
    It is clear from the definition that if $\DOM$ is polyhedral, \ie if all $\DOM^i$'s are half-spaces, and $\BOUNDFIELD$ consists of inward normal reflections, then $\rb{\DOM, \BOUNDFIELD}$ is polyhedral with piecewise constant reflection.
\end{remark}

\begin{theorem}[Theorem 3 in \citep{kang2014characterization}, simplified] \label{app-thm: submartingale stationary}
    Suppose that the pair $\rb{\DOM, \BOUNDFIELD}$ is piecewise $\calC^2$ with continuous reflection, for all $i\in \calI$ and $\bfx \in \partial \DOM^i$, $\innerprod{\NORMAL^i \rb{\bfx}}{\BOUNDVEC^i \rb{\bfx}} = 1$, $\calV = \emptyset$, $\drift \rb{\cdot} \in \calC^1 \rb{\DOMCLOS}$ and $\diffcoef \in \calC^2 \rb{\DOMCLOS}$ (elementwise), and the submartingale problem associated with $\rb{\DOM, \BOUNDFIELD}$ and $\calV$ is well posed.
    Furthermore, suppose there exists a nonnegative function $p \in \calC^2 \rb{\DOMCLOS}$ with $Z_p = \int_{\DOMCLOS} p\rb{\bfx} d\bfx < \infty$ that solves the PDE defined by the following three relations:
    \begin{enumerate}
        \item For $\bfx \in \DOM$:
        \begin{align} \label{app-eq: general stationarity 1 FPE}
            0 = \frac{1}{2} \sum_{i,j = 1}^\PARAMDIM \frac{\partial^2}{\partial x_i \partial x_j} \rb{\diffij \rb{\bfx} p\rb{\bfx}} - \sum_{i=1}^\PARAMDIM \frac{\partial}{\partial x_i} \rb{\drifti \rb{\bfx} p\rb{\bfx}}\,.
        \end{align}
        \item For each $i \in \calI$ and $\bfx \in \partial \DOM \cap \calS$, 
        \begin{align} \label{app-eq: general stationarity 2 boundary 1}
            0 = -2 p \rb{\bfx} \innerprod{\NORMAL^i \rb{\bfx}}{\drift \rb{\bfx}} + \NORMAL^i \rb{\bfx}^\top \diffcoef \rb{\bfx} \nabla p \rb{\bfx} - \nabla \cdot \rb{p \rb{\bfx} \bfq^i \rb{\bfx}} + p\rb{\bfx} K_i \rb{\bfx} \,,
        \end{align}
        where 
        \begin{align*} 
            \bfq^i \rb{\bfx} \triangleq \NORMAL^i \rb{\bfx}^\top \diffcoef \rb{\bfx} \NORMAL^i \rb{\bfx} \BOUNDVEC^i \rb{\bfx} - \diffcoef \rb{\bfx} \NORMAL^i \rb{\bfx}
        \end{align*}
        and
        \begin{align*}
            K_i \rb{\bfx} \triangleq \innerprod{\NORMAL^i \rb{\bfx}}{\nabla \cdot \diffcoef \rb{\bfx}} = \sum_{k=1}^\PARAMDIM n^i \rb{\bfx}_k \sum_{j=1}^\PARAMDIM \frac{\partial \diffkj}{\partial x_j} \rb{\bfx} \,.
        \end{align*}
        \item For each $i, j \in \calI$, $i\neq j$, and $\bfx \in \partial \DOM^i \cap \partial \DOM^j \cap \partial \DOM$,
        \begin{align} \label{app-eq: general stationarity 2 boundary 2}
            p\rb{\bfx} \rb{\innerprod{\bfq^i \rb{\bfx}}{\NORMAL^j \rb{\bfx}} + \innerprod{\bfq^j \rb{\bfx}}{\NORMAL^i \rb{\bfx}}} = 0\,.
        \end{align}
    \end{enumerate}
    Then the probability measure on $\DOMCLOS$ defined by 
    \begin{align} \label{app-eq: general stationary}
        \stationarydist \rb{A} \triangleq \frac{1}{Z_p} \intop_{A} p \rb{\bfx} d\bfx \,, \quad A \in \calB \rb{\DOMCLOS} \,, 
    \end{align}
    is a stationary distribution for the well-posed submartingale problem.
\end{theorem}

We are now ready to state a characterization of stationary distributions of \eqref{app-eq: sder general}.
Note that for simplicity, we do not maintain full generality.
\begin{corollary} [Stationary distribution of weak solutions to SDERs] \label{app-cor: combine theorems from kang and ramanan}
    Suppose that, 
    $\DOM$ is convex and bounded,
    $\drift \in \calC^1 \rb{\DOMCLOS}$ and $\diffcoef \in \calC^2 \rb{\DOMCLOS}$,
    $\rb{\DOM, \BOUNDFIELD}$ is piecewise $\calC^2$ with continuous reflection,
    $\diffcoef$ is uniformly elliptic (see \eqref{app-eq: unifom ellipticity}),
    and $\calV = \emptyset$.
    Then $p \in \calC^2$ satisfying the conditions in  \cref{app-thm: submartingale stationary} defines a  stationary distribution for \eqref{app-eq: sder general}.
\end{corollary}
\begin{proof}
    Assumptions compactness of the domain, and continuous differentiability of the drift and dispersion coefficient imply that they are Lipschitz, hence Exercise 2.5.1 and Theorem 2.5.4 of \citep{sder-pilipenko-2014} imply that there exists a unique strong solution to the SDER \eqref{app-eq: sder general}.
    Then, piecewise $\calC^2$ with continuous reflection, the uniform ellipticity assumption, Theorems 1 and 3 of \citep{kang2017submartingale}, and \cref{app-thm: submartingale stationary} imply that if there exists $p \in \calC^2$ satisfying \eqref{app-eq: general stationarity 1 FPE}-\eqref{app-eq: general stationarity 2 boundary 2}, then \eqref{app-eq: general stationary} is a stationary distributions of \eqref{app-eq: sder general}.
\end{proof}
In the next subsection we use this to derive explicit expressions for the stationary distribution in the setting of this paper.

\subsection{SDER with isotropic diffusion in a box} \label{app-sec: SDER in a box}

We proceed to assume that the diffusion term is a scalar matrix of the form $\diffcoef \rb{\bfx} = 2 \sigma^2 \rb{\bfx} \bfI_\PARAMDIM$, and that $\DOM$ is a bounded box in $\bbR^\PARAMDIM$, \ie there exist $\cb{m_i < M_i}_{i=1}^\PARAMDIM$ such that 
\begin{align} \label{app-eq: box dom 1}
    \DOM = \prod_{i=1}^\PARAMDIM \rb{m_i, M_i} = \bigcap_{i=1}^\PARAMDIM \rb{\DOM^i_m \cap \DOM^i_M} \,,
\end{align}
where 
\begin{align} \label{app-eq: box dom 2}
    \DOM^i_m \triangleq \cb{\bfx \in \bbR^\PARAMDIM \setst x_i > m_i}\,, \; \Omega^i_M \triangleq \cb{\bfx \in \bbR^\PARAMDIM \setst x_i < M_i} \,,
\end{align}
and that the reflecting field is normal to the boundary, \ie given by \eqref{app-eq: piecewise continuous reflection boundary field} with
\begin{align} \label{app-eq: box reflecting field}
    \BOUNDVEC^i_m \equiv \NORMAL^i_m \equiv \bfe_i \,, \; \text{and} \; \BOUNDVEC^i_M \equiv \NORMAL^i_M \equiv - \bfe_i
\end{align}
for $i = 1, \dots, \PARAMDIM$.
In this setting, we can considerably simplify the conditions in \cref{app-thm: submartingale stationary}, as done in the following corollary.

\begin{lemma}[Stationarity condition for SDER in a box with normal reflection] \label{app-lem: FP with Neumann stationarity}
    Let $\drift \rb{\cdot} \in \calC^1$, and let $\sigma\rb{\cdot} \in \calC^2$ be uniformly bounded away from 0, \ie  there exists $\sigma^2 > 0$ such that for all $\bfx \in \DOMCLOS$, $\sigma^2 \rb{\bfx} > \sigma^2$. 
    If there exists $p \in \calC^2$ such that
    \begin{align} \label{app-eq: FP with Neumann}
        \begin{cases}
            0 = \nabla \cdot \rb{\nabla \rb{ \sigma^2 \rb{\bfx} p\rb{\bfx}} - \drift \rb{\bfx} p\rb{\bfx}} & \bfx \in \DOM \,, \\
            0 = \innerprod{\nabla \rb{ \sigma^2 \rb{\bfx} p\rb{\bfx}} - \drift \rb{\bfx} p\rb{\bfx}}{\NORMAL \rb{\bfx}} & \bfx \in \partial \DOM \,,
        \end{cases}
    \end{align}
    and $\int_{\DOMCLOS} p \rb{\bfx} d \bfx = 1$, then $p$ is a stationary distribution of 
    \begin{align} \label{app-eq: specific SDER}
        d \bfx_t = \drift \rb{\bfx_t} dt + \sqrt{2 \sigma^2 \rb{\bfx_t}} d\bfw_t + d\reflectionprocess
    \end{align}
    in $\DOM$.
\end{lemma}

\begin{remark}
    \eqref{app-eq: FP with Neumann} is exactly the stationarity condition derived from the Fokker-Planck equation with Neumann boundary conditions ensuring conservation of mass.
\end{remark}

\begin{proof}
Under the assumptions we see that the conditions of \cref{app-cor: combine theorems from kang and ramanan} are satisfied, and we can use \eqref{app-eq: general stationarity 1 FPE}-\eqref{app-eq: general stationarity 2 boundary 2} to find stationary distributions of \eqref{app-eq: specific SDER}.
First, notice that \eqref{app-eq: general stationarity 1 FPE} simplifies to 
\begin{align*}
    0 &= \frac{1}{2} \sum_{i,j = 1}^\PARAMDIM \frac{\partial^2}{\partial x_i \partial x_j} \rb{\diffij \rb{\bfx} p\rb{\bfx}} - \sum_{i=1}^\PARAMDIM \frac{\partial}{\partial x_i} \rb{\drifti \rb{\bfx} p\rb{\bfx}} \\
    &= \frac{1}{2} \sum_{i = 1}^\PARAMDIM \frac{\partial^2}{\partial x_i^2} \rb{2 \sigma^2 \rb{\bfx} p\rb{\bfx}} - \sum_{i=1}^\PARAMDIM \frac{\partial}{\partial x_i} \rb{\drifti \rb{\bfx} p\rb{\bfx}} \\
    &= \nabla \cdot \rb{\nabla \rb{ \sigma^2 \rb{\bfx} p\rb{\bfx}} - \drift \rb{\bfx} p\rb{\bfx}} \,.
\end{align*}

Next, we can considerably simplify the boundary conditions.
First, notice that $\calS$ consists of the interior of the domain's faces so for $\bfx \in \partial \DOM \cap \calS$, the set of active boundary regions $\calI \rb{\bfx}$ is a singleton $\calI \rb{\bfx} = \cb{\rb{i, s}}$, for some $i = 1, \dots, \PARAMDIM$ and $s \in \cb{m, M}$.
We focus on the lower boundaries ($m$), as the conditions for the upper boundaries are symmetric.

For $i=1,\dots, \PARAMDIM$ and $\bfx \in \partial \DOM \cap \calS$, $\BOUNDVEC^i_m \rb{\bfx} = \NORMAL^i_m \rb{\bfx} = \bfe_i$ so 
\begin{align*}
    \bfq^i_m \rb{\bfx} &= \NORMAL^i_m \rb{\bfx}^\top \diffcoef \rb{\bfx} \NORMAL^i_m \rb{\bfx} \BOUNDVEC^i_m \rb{\bfx} - \diffcoef \rb{\bfx} \NORMAL^i_m \rb{\bfx} \\
    &= \sigma^2 \rb{\bfx} \rb{\bfe_i^\top \bfI_\PARAMDIM \bfe_i} \bfe_i - \sigma^2 \rb{\bfx} \bfI_\PARAMDIM \bfe_i \\
    &= \bfzero \,,
\end{align*}
so \eqref{app-eq: general stationarity 2 boundary 2} is satisfied.
In addition, 
\begin{align*}
    K^i_m \rb{\bfx} = \nabla \cdot \diffcol_i \rb{\bfx} = \frac{\partial}{\partial x_i} \sigma^2 \rb{\bfx} \,,
\end{align*}
so \eqref{app-eq: general stationarity 2 boundary 1} becomes, for all $i=1, \dots, \PARAMDIM$, 
\begin{align*}
    0 &= -2 p \rb{\bfx} \innerprod{\NORMAL^i_m \rb{\bfx}}{\drift \rb{\bfx}} + \NORMAL^i_m \rb{\bfx}^\top \diffcoef \rb{\bfx} \nabla p \rb{\bfx} - \nabla \cdot \rb{p \rb{\bfx} \bfq^i_m \rb{\bfx}} + p\rb{\bfx} K^i_m \rb{\bfx} \\
    0 &= -2 p\rb{\bfx} \drifti \rb{\bfx} + \diffcol_i^\top \nabla p \rb{\bfx} + p \rb{\bfx} \frac{\partial}{\partial x_i} \sigma^2 \rb{\bfx} \\
    0 &= -2 p\rb{\bfx} \drifti \rb{\bfx} + \sigma^2 \rb{\bfx} \frac{\partial}{\partial x_i} p \rb{\bfx} + p \rb{\bfx} \frac{\partial}{\partial x_i} \sigma^2 \rb{\bfx} \,,
\end{align*}
which is
\begin{align*}
    0 &= - p\rb{\bfx} b_i \rb{\bfx} + \frac{1}{2} \frac{\partial}{\partial x_i} \rb{p\rb{\bfx} \sigma^2 \rb{\bfx}} \\
    &= \innerprod{\nabla \rb{ \sigma^2 \rb{\bfx} p\rb{\bfx}} - \drift \rb{\bfx} p\rb{\bfx}}{\NORMAL \rb{\bfx}} \,.
\end{align*}
\end{proof}

\subsubsection{Reflected Langevin dynamics in a box} \label{app-sec: langevin dynamics}
In this section, we derive some useful properties of the SDER
\begin{align} \label{app-eq: general reflected langevin complementary}
    d\bfx_t = - \nabla L \rb{\bfx_t} + \sqrt{2 \beta^{-1} \sigma^2 \rb{\bfx_t}} d\bfw_t + d\reflectionprocess \,,
\end{align}
in a box domain as defined in \eqref{app-eq: box dom 1}-\eqref{app-eq: box reflecting field}, where $L \ge 0$ is some (loss/potential) function, and $\beta > 0$ is an inverse temperature parameter.
First, we characterize the stationary distribution of this process.

\begin{recall}[\cref{app-lem: langevin in a box stationary}] \label{app-lem: langevin in a box stationary with proof}
    If $L, \sigma^2 \in \calC^2$, $\sigma^2 \rb{\cdot} > 0$ is uniformly bounded away from 0 in $\DOMCLOS$,  
    \begin{align*}
        Z = \intop_{\DOMCLOS} \frac{1}{\sigma^{2}\rb{\bfx}} \exp\rb{-\beta\intop\frac{\nabla L\rb{\bfx}}{\sigma^{2}\rb{\bfx}}d\bfx} < \infty \,,
    \end{align*}
    the integrals exist, and the field $\nabla L / \sigma^2$ is conservative (curl-free),
    then 
    \begin{align} \label{app-eq: stationary langevin complementary}
        \stationarydist \rb{\bfx} = \frac{1}{Z} \frac{1}{\sigma^{2}\rb{\bfx}} \exp\rb{-\beta\intop\frac{\nabla L\rb{\bfx}}{\sigma^{2}\rb{\bfx}}d\bfx} \,
    \end{align}
     is a stationary distribution of \eqref{app-eq: general reflected langevin complementary}.
\end{recall}
\begin{proof}
The drift term in this setting is $\drift = -\beta \nabla L$.
Therefore, from \cref{app-lem: FP with Neumann stationarity}, we get that any distribution that satisfies
\begin{align*}
    0 = \nabla \rb{\sigma^2 \rb{\bfx} \stationarydist \rb{\bfx}} + \beta \stationarydist \rb{\bfx} \nabla L \rb{\bfx}
\end{align*}
on $\DOMCLOS$, is a stationary distribution.
We can solve this PDE as
\begin{align*}
0 & =\beta\nabla L \rb{\bfx} \stationarydist \rb{\bfx} + \stationarydist \rb{\bfx} \nabla \sigma^{2} \rb{\bfx} + \sigma^{2} \rb{\bfx} \nabla \stationarydist \rb{\bfx}\\
 & =\stationarydist\rb{\bfx}\left(\beta\nabla L\rb{\bfx}+\nabla\sigma^{2}\rb{\bfx}\right)+\sigma^{2}\rb{\bfx}\nabla \stationarydist\rb{\bfx}
\end{align*}
\[
-\sigma^{2}\rb{\bfx}\nabla \stationarydist\rb{\bfx}=\stationarydist\rb{\bfx}\left(\beta\nabla L\rb{\bfx}+\nabla\sigma^{2}\rb{\bfx}\right)
\]
\[
\frac{\nabla \stationarydist}{\stationarydist}=-\frac{\beta\nabla L+\nabla\sigma^{2}}{\sigma^{2}}
\]
\[
\nabla\ln \stationarydist=-\frac{\beta\nabla L}{\sigma^{2}}-\nabla\ln\sigma^{2}
\]
\[
\nabla\ln\left(\stationarydist\cdot\sigma^{2}\right)=-\frac{\beta\nabla L}{\sigma^{2}}\,.
\]
Then, 
\[
\ln\left(\stationarydist\cdot\sigma^{2}\right)=-\beta\intop\frac{\nabla L}{\sigma^{2}}+C
\]
where we used the assumption that the integral on the RHS exists, and is well defined.
Hence  
\[
\stationarydist\rb{\bfx}\propto\frac{1}{\sigma^{2}\rb{\bfx}}\exp\left(-\beta\intop\frac{\nabla L\rb{\bfx}}{\sigma^{2}\rb{\bfx}}d\bfx\right)\,.
\]
\end{proof}

When the integral in \eqref{app-eq: stationary langevin complementary} is solvable, we can find an explicit expression for the stationary distribution, as was done in \cref{app-sec: stationary distirbutions of cld}.

\newpage


\end{document}